%% file: main-arxiv.tex
\documentclass{article}

\usepackage[english]{babel}
\usepackage[preprint]{neurips_2026}

\usepackage[utf8]{inputenc} 
\usepackage[T1]{fontenc}    
\usepackage{url}            
\usepackage{booktabs}       
\usepackage{amsfonts}       
\usepackage{nicefrac}       
\usepackage{microtype}      
\usepackage{float}
\usepackage{algorithm}
\usepackage[noend]{algpseudocode}
\usepackage{tikz}
\usetikzlibrary{calc}

\usepackage{graphicx} 
\usepackage{amsmath,amssymb,amsthm,mathtools}
\usepackage{color,soul}
\usepackage{enumitem} 
\usepackage{natbib}
\usepackage{thm-restate}
\usepackage{tikz}
\usetikzlibrary{positioning,fit}
\usepackage{tikz}
\usetikzlibrary{calc}
\usepackage{xcolor}
\usetikzlibrary{positioning,fit,shapes.geometric}
\usepackage{braket}
\usepackage{enumitem}
\usepackage[framemethod=tikz]{mdframed}
\usepackage{etoolbox}
\usepackage{mathrsfs}
\usepackage{soul}

\usepackage[unicode,hypertexnames=false, colorlinks=true, allcolors=blue]{hyperref}
\usepackage{cleveref}

\algblockdefx[Sampler]{Sampler}{EndSampler}
  [1]{\textbf{i.i.d. sampler for #1:}}{}
\algtext*{EndSampler}

\algblockdefx[Oracle]{Oracle}{EndOracle}
  [1]{\textbf{Local Neighbor Oracle #1:}}{}
\algtext*{EndOracle}

\algdef{SE}[INDENT]{Indent}{EndIndent}{}{}
\algtext*{Indent}
\algtext*{EndIndent}

\input{tikz-macros}

\newcommand{\poly}{\textup{poly}}
\newcommand{\F}{\mathcal{F}}
\newcommand{\C}{\mathcal{C}}
\newcommand{\Pro}{\mathbb{P}}
\newcommand{\Fscr}{\mathscr{F}}
\newcommand{\N}{\mathbb{N}}
\newcommand{\eps}{\epsilon}
 
\newcommand{\Lo}{\mathcal{L}} 
\newcommand{\Lmf}{\mathcal{L}_{\textup{mf}}} 
\newcommand{\Lc}{\mathcal{L}_{\textup{cont}}} 
\newcommand{\Lchat}{\hat{\mathcal{L}}_{\textup{cont}, n}}

\newcommand{\1}{\mathbf{1}}

\newcommand{\Psrw}{P_{\textup{s.r.w.}}}

\newcommand{\mutalg}{D_t^{\textup{improved}}}
\newcommand{\mutsrw}{D_t^{\textup{baseline}}}
\newcommand{\qtalg}{q_t^{\textup{improved}}}
\newcommand{\qtsrw}{q_t^{\textup{baseline}}}

\newcommand{\Pimp}{P_{\textup{improved}}}
\newcommand{\Pbase}{P_{\textup{baseline}}}

\newcommand{\Divt}{\textup{Div}^P_\tau}
\newcommand{\Kct}{K^p_\tau}
\newcommand{\Kp}{K^p}
\newcommand{\Ver}{\textup{Ver}}
\newcommand{\rank}{\textup{rank}}
\newcommand{\tr}{\mathrm{tr}}

\newcommand{\diffsim}{\phi^{\textup{diff}}_\tau}
\newcommand{\diffdist}{\rho^{\textup{diff}}_\tau}
\newcommand{\dist}{\textup{dist}}
\newcommand{\ISineq}{\textup{IS}}

\newcommand{\Femb}{\mathcal{F}_{\textup{emb}}}
\newcommand{\Fprov}{\mathcal{F}_{\textup{prover}}}

\newcommand{\ex}{\mathbb{E}}
\newcommand{\R}{\mathbb{R}}
\newcommand{\Rchat}{\hat{\mathcal{R}}}
\newcommand{\Rc}{\mathcal{R}}

\newcommand{\Succ}{\operatorname{succ}}  
\newcommand{\supp}{\operatorname{supp}}
\newcommand{\diag}{\textup{diag}}
\newcommand{\perm}{p_{\textup{erm fails}}}
\newcommand{\pimin}{\pi_{\textup{min}}}

\newtheorem{assumption}{Assumption}
\newtheorem{theorem}{Theorem}
\newtheorem{lemma}[theorem]{Lemma}

\newtheorem{corollary}[theorem]{Corollary}
\newtheorem{definition}[theorem]{Definition}

\numberwithin{equation}{section}
\numberwithin{lemma}{section}
\numberwithin{theorem}{section}
\numberwithin{definition}{section}
\numberwithin{corollary}{section}
\numberwithin{proposition}{section}

\input{shared.tex}
\newcommand{\ermdefnappen}{\ermrestate}
\newcommand{\ermdefnmaintxt}{\ermoriginal}

\newcommand{\raddefnappen}{\radrestate}
\newcommand{\raddefnmaintxt}{\radoriginal}

\newcommand{\etethmappen}{\etethmrestate}
\newcommand{\etethmmaintext}{\etethmoriginal}

\newcommand{\localsubgappen}{\localsubgrestate}
\newcommand{\localsubgmaintxt}{\localsubgoriginal}

\newcommand{\cliquegraphmaintxt}{\cliquegraphoriginal}
\newcommand{\cliquegraphappen}{\cliquegraphrestate}

\newcommand{\cliquepropmaintxt}{\cliqueproporiginal}
\newcommand{\cliquepropappen}{\cliqueproprestate}

\title{A Theoretical Framework for Self-Play Theorem Proving Algorithms}
\author{Thomas Chen, Zhiyuan Li}

\begin{document}
\maketitle

\input{sections/abstract}
\input{sections/intro}

\input{sections/arxiv/intro-arxiv-diff}
\input{sections/arxiv/related-work-arxiv}
\input{sections/arxiv/preliminaries-arxiv}

\input{sections/framework-q1}
\input{sections/arxiv/exploration-q2-arxiv}

\input{sections/conclusion}

\newpage
\input{sections/acknowledgements}
\bibliographystyle{alpha}
\bibliography{bibliography}

\newpage
\appendix
\onecolumn
\input{sections/arxiv/appendix-arxiv}

\end{document}

%% file: tikz-macros.tex
\definecolor{knowblue}{RGB}{35,92,204}
\definecolor{bdblue}{RGB}{174,226,252}
\definecolor{outgray}{RGB}{229,229,229}
\definecolor{brdark}{RGB}{18,55,170}
\definecolor{brlight}{RGB}{75,181,255}

\tikzset{
  edge/.style={draw=black!42, line width=0.45pt},
  bridgeedge/.style={draw=black!75, line width=0.9pt},
  panel/.style={draw=black!55, rounded corners=7pt, line width=0.65pt},
  know/.style={circle, draw=blue!65!black, fill=knowblue, line width=0.8pt,
               minimum size=4.6mm, inner sep=0pt},
  bdy/.style={circle, draw=blue!60!black, fill=bdblue, line width=0.75pt,
              minimum size=4.6mm, inner sep=0pt},
  outside/.style={circle, draw=black!55, fill=outgray, line width=0.65pt,
              minimum size=4.6mm, inner sep=0pt},
  bracketdark/.style={brdark, dashed, line width=1.0pt, dash pattern=on 3pt off 2pt},
  bracketlight/.style={brlight, dashed, line width=1.0pt, dash pattern=on 3pt off 2pt},
}

\newcommand{\MakeClusterCoords}[4]{
  \begin{scope}[shift={(#2,#3)}]
    \coordinate (#1-q0) at (0,0);
    \coordinate (#1-q1) at ({#4*(-0.40)}, 0.42);
    \coordinate (#1-q2) at ({#4*( 0.40)}, 0.42);
    \coordinate (#1-q3) at ({#4*(-0.40)},-0.42);
    \coordinate (#1-q4) at ({#4*( 0.40)},-0.42);
    \coordinate (#1-r1) at (0,1.00);
    \coordinate (#1-r2) at ({#4*( 0.72)}, 0.72);
    \coordinate (#1-r3) at ({#4*( 1.05)}, 0.00);
    \coordinate (#1-r4) at ({#4*( 0.72)},-0.72);
    \coordinate (#1-r5) at (0,-1.00);
    \coordinate (#1-r6) at ({#4*(-0.72)},-0.72);
    \coordinate (#1-r7) at ({#4*(-1.05)}, 0.00);
    \coordinate (#1-r8) at ({#4*(-0.72)}, 0.72);
    \coordinate (#1-l1) at (0,1.48);
    \coordinate (#1-l2) at ({#4*( 1.15)}, 1.15);
    \coordinate (#1-l4) at ({#4*( 1.15)},-1.15);
    \coordinate (#1-l5) at (0,-1.48);
    \coordinate (#1-l6) at ({#4*(-1.15)},-1.15);
    \coordinate (#1-l7) at ({#4*(-1.48)}, 0.00);
    \coordinate (#1-l8) at ({#4*(-1.15)}, 1.15);
    \coordinate (#1-h)  at ({#4*(1.58)},0.00);
  \end{scope}
}

\newcommand{\DrawClusterEdges}[1]{%
  \foreach \u/\v in {q0/q1,q0/q2,q0/q3,q0/q4,q1/q2,q1/q3,q1/q4,q2/q3,q2/q4,q3/q4}
    \draw[edge] (#1-\u) -- (#1-\v);
  \foreach \u/\v in {r1/r2,r2/r3,r3/r4,r4/r5,r5/r6,r6/r7,r7/r8,r8/r1}
    \draw[edge] (#1-\u) -- (#1-\v);
  \foreach \u/\v in {r1/q1,r1/q2,r2/q2,r3/q2,r3/q4,r4/q4,r5/q3,r5/q4,r6/q3,r7/q1,r7/q3,r8/q1}
    \draw[edge] (#1-\u) -- (#1-\v);
  \foreach \u/\v in {l1/r1,l2/r2,l4/r4,l5/r5,l6/r6,l7/r7,l8/r8}
    \draw[edge] (#1-\u) -- (#1-\v);
  \foreach \u/\v in {h/r3,h/q4}
    \draw[edge] (#1-\u) -- (#1-\v);
}

\newcommand{\DrawNodes}[4]{
  \begingroup
  \def\outside{#4}\ifx\outside\empty\else
    \foreach \n in {#4} {\node[outside]  at (#1-\n) {};}
  \fi
  \def\boundary{#3}\ifx\boundary\empty\else
    \foreach \n in {#3} {\node[bdy]  at (#1-\n) {};}
  \fi
  \def\dark{#2}\ifx\dark\empty\else
    \foreach \n in {#2} {\node[know] at (#1-\n) {};}
  \fi
  \endgroup
}

\newcommand{\Bracket}[5]{
  \draw[#5] (#1,#3+0.12) -- (#1,#3) -- (#2,#3) -- (#2,#3+0.12);
  \node[#5, solid, font=\large] at ({(#1+#2)/2},#3-0.34) {$#4$};
}

%% file: shared.tex
\newcommand{\ermoriginal}{
\begin{restatable}[ERM, $\eps$-minimizer]{definition}{ermdefn}\label{defn:erm-cont} 

The empirical risk minimizer (ERM) is any embedding function of $h_\theta  \in \Femb$ that minimizes $\Lchat$ over $\Femb$.
\begin{align}
h_\theta \leftarrow \arg \min_{h_{\theta'} \in \Femb} \Lchat(h_{\theta'})
\end{align}
For $\eps > 0$, an $\eps$-minimizer of $\Lchat$ is any embedding function of $h_\theta  \in \Femb$ where:
\begin{align}
    \Lchat(h_\theta) \leq \min_{h_{\theta'} \in \Femb} \Lchat(h_{\theta'}) + \eps
\end{align}

We will also call an $\eps$-minimizer of $\Lchat$ an $\eps$-ERM. 
\end{restatable}
}
\newcommand{\ermrestate}{
\ermdefn*
}

\newcommand{\radoriginal}{
\begin{restatable}[Rademacher Complexity]{definition}{raddefn}\label{defn:rad-complexity} 
The Rademacher complexity of function class $\F$ of models $\F \ni h: V \to \R^k$ on $n$ datapoints is:
\begin{align}
    \Rchat_n(\F) := \max_{x_1, \ldots x_n \in V} \ex_{\sigma \in \{ \pm 1\}^n} \sup_{h \in \F, i \in [k]} \frac{1}{n} (\sum_{j \in [n]} \sigma_j h_i(x_j))
\end{align}
\end{restatable}
}
\newcommand{\radrestate}{
\raddefn*
}

\newcommand{\etethmoriginal}{
\begin{restatable}[End-to-End Contrastive Learning Result]{theorem}{etecontlearning}
\label{thm:e2e-contrastive-learning-result}
    Let $G = (V,E)$ be any connected, undirected graph. Let $(V,P)$ be an irreducible, reversible Markov chain with transition matrix $P \in [0,1]^{|V| \times |V|}$ and stationary distribution $\pi\in \Delta(V)$. 
    
    For any $\tau \in \N$, instantiate the contrastive loss $\Lc$ with base distribution $\pi$, augmentation function $P^\tau: V \to \Delta(V)$, which augments $x \in V$ into a distribution of nodes given by $e_x^\top P^\tau$. 
    
    For any $k \in [|V| - 1]$, let $\Femb \subset \F_{V}^{(k)}$ be a class of embedding models $h_\theta$ that map a domain $V$ into dimension $k$ Euclidean space: $\Femb \ni h_\theta : V \to \R^k$. Let $\Lc : \F_{V}^{(k)} \to \R$ be the population contrastive loss, instantiated with base distribution $\pi$ and augmentation function $P^\tau$, as in \Cref{eq:cont-loss-defn}, and $\Lchat : \F_{V}^{(k)} \to \R$ be the empirical contrastive loss over a i.i.d. dataset of size $n$ drawn from $\pi$, as in \Cref{eq:emp-cont-loss-defn}. Assume:

    \begin{enumerate}
        \item For some $\chi > 0$, we have that $\forall h_\theta \in \Femb, \forall x \in V$, $||h_\theta(x)||_\infty \leq \chi$. 
        \item $\Femb$ contains at least one global minimizer of $\Lc$ over $\F_V^{(k)}$ (i.e. over functions $h : V \to \R^k$)
        \item $\Rchat_n(\Femb) \leq \sqrt{\frac{\C(\Femb)}{n}}$,  where $\C(\Femb)$ is an $\Femb$-dependent constant
        \item With $\{ \lambda_i\}_{i \in [|V|]} := \{ \lambda_i(P^{2\tau})\}_{i \in [|V|]}$ being the eigenvalues of $P^{2\tau}$, where $\lambda_1 \geq \lambda_2 \geq \ldots \geq \lambda_{|V|} \geq 0$, assume  $\lambda_k^2 - \lambda_{k + 1}^2 > 0$.
        \item There exists $\beta > 0 $ such that $ \forall y \in V, \, \sum_{x \in N(y)} \pi(x) \leq \beta \, \pi(y)$. (E.g. if $P$ is the simple random walk, then $\beta = d$ the maximum degree of any node in $G$).
    \end{enumerate}
    Set any $\eps, \Delta \in (0,1)$ and set $\perm = \frac{\eps}{\poly(k, \chi)}$. Then, with probability at least $1 - \perm$ over the randomness of a training dataset of size $n = \Theta(\frac{\C(\Femb)\, \poly(k, \chi)}{\eps^2}\log \frac{1}{p_{\textup{erm fails}}})$ drawn i.i.d. from $\pi$, the following guarantee holds for any ERM $h_{\hat{\theta}} \in \Femb$ of $\Lchat$, where $\rank(F_{\hat\theta}) = \rank([\sqrt{\pi(x)} h_{\hat\theta}(x)^\top]_{x \in V}) = k$ (where $F_{\hat\theta}$ is a $|V| \times k$ matrix whose $x$-th row is $\sqrt{\pi(x)} h_{\hat\theta}(x)^\top$).

    \begin{align}
        &\Pro_{x \sim \pi}\Big [\exists y,z \in N(x) \cup \{x\} \, : \, |\diffsim(y,z) -  \langle h_{\hat\theta}(y), h_{\hat\theta}(z) \rangle| > \frac{\lambda_{k + 1}}{\pimin} + 2 \chi \sqrt{k} \cdot \sqrt{\Delta} + \Delta \Big]\\
        &\leq \frac{\eps }{\Delta} \cdot \frac{(\beta + 1)\cdot \Big(1 + 4\sqrt{\frac{2}{\lambda_k^2 - \lambda_{k + 1}^2}} \Big)^2}{2 (\sqrt{2} - 1)\lambda_k}
    \end{align}
\end{restatable}
}
\newcommand{\etethmrestate}{
\etecontlearning*
}

\newcommand{\localsubgoriginal}{
\begin{restatable}[Local Subgraph]{definition}{localsubgraph}\label{defn:T-neighborhood}
    Let $G = (V,E)$ be a connected, undirected graph, with reversible, nearest-neighbor-type Markov chain $(V,P)$. For any $x \in V$, $T \in \N$, let $G_{x, T} = (V_{x, T}, E_{x, T})$ be a subgraph of $G$ supported on $B(x, T)$, with edges $(a,b)$ from $B(x, T) \ni a$ to $V - B(x, T) \ni b$ replaced by a self-loop, $(a,a)$ in $G_{x,T}$. Let Markov chain $(V_{x, T}, P_{x, T})$, be the restriction to $(V,P)$ to $G_{x, T}$, where any transition probability of $(V,P)$ along edges from $B(x, T)$ to $V - B(x, T)$ are redirected back to their source via the aforementioned self-loops. Let $\pi_{x, T}$ be its stationary measure of $(V_{x, T}, P_{x, T})$. For $T \in N$, $\tau \in N$, $y \in B(x, T)$, denote $p^\tau_{G_{x, T}}(y)$ as the distribution of a $\tau$ step random walk on $G_{x, T}$ according to Markov chain $(V_{x, T}, P_{x, T})$. 
\end{restatable}
}
\newcommand{\localsubgrestate}{
\localsubgraph*
}

\newcommand{\cliquegraphoriginal}{
\begin{restatable}{definition}{cliquegraph}
\label{defn:clique-graph} (Clique-Graph) Let parameters $M$ be the clique size, $r$ be the meta-graph degree, and $K$ be the number of cliques. Let \(Q=([K],E(Q))\) be a finite connected, undirected, $r$-regular  graph. For $i \in [K]$, denote $N_Q(i) \subset [K]$ as the neighbors of $i$ in $Q$. Suppose there exists $c \in [r]$ where for every $i\in[K]$ and every two
distinct neighbors $j,k\in N_Q(i)$, $|N_Q(j)\cap N_Q(k)| = c$.
For clique size \(M\ge 2\), define the graph \(G=G(Q,M)\) by $V=[K]\times [M]$, where each $C_i:=\{i\}\times [M]$ is a clique of size $M$. For every meta-edge \(ij\in E(Q)\) and \(a\in[M]\), add the inter-clique edge $(i,a)\sim (j,a).$ 
\end{restatable}
}
\newcommand{\cliquegraphrestate}{
\cliquegraph*
}

\newcommand{\cliqueproporiginal}{
\begin{restatable}{proposition}{eteguaranteealg}
\label{prop:e2e-guarantee-alg2} (End-to-End Guarantee of \Cref{defn:conj-alg-improved})
Let $Q = ([K], E(Q))$ be a finite connected, undirected, \(r\)-regular graph. Let $G = G(Q, M)$ be a Clique-Graph satisfying \Cref{defn:clique-graph} with  meta-graph $Q$. Suppose  $M \geq r + 3$. Let $P_Q$ denote the simple random walk over $Q$. \Cref{defn:conj-alg-improved}, with parameters $\tau = 1, \eta = 0$,  defines a random walk over the meta-graph $Q$ with transition matrix $\Pimp^{(Q)} = (\frac{1}{r + 1} + O(\frac{1}{M})) I + (\frac{r}{r + 1} + O(\frac{1}{M})) P_Q$. The simple random walk over $G$ defines a random walk over the meta-graph $Q$ with transition matrix $\Pbase^{(Q)} = \frac{M - 1}{M - 1 + r} I + \frac{r}{M - 1 + r} P_Q$. Define the spectral gap of transition matrix $P$ as $\gamma(P) := 1 - \max(\lambda_2(P), |\lambda_K(P)|)$. Then, there exists a universal constant $c > 0$ such that

\begin{align}
\gamma( \Pimp^{(Q)}) \geq c\, \frac{M}{r^2} \,\gamma(\Pbase^{(Q)})
\end{align}

In particular, for $M \gg \max(\frac{r^2}{c}, r + 3)$, we have $\gamma( \Pimp^{(Q)}) \gg \gamma(\Pbase^{(Q)})$.

As a consequence, suppose $D_0 \in \Delta([K] \times [M])$ is uniform on a starting clique $ \{ i \} \times [M]$. For all $t \in \N$, let $\mutalg$ be the result of applying the random walk defined in \Cref{defn:conj-alg-improved} for $t$ steps on $D_0$, and  $\mutsrw$ the result of applying a simple random walk over $G$ for $t$ steps on $D_0$. Then, for any $\eps \in (0,1)$, for  $t_{\textup{baseline}} := \frac{1}{2\gamma(\Pbase^{(Q)})} \log \frac{K}{\eps}$, $\textup{Div}_1^{\Psrw}(D_{t_{\textup{baseline}}}^{\textup{baseline}}) \geq 1 - \eps$, with maximum value $1$.  Whereas, for $t_{\textup{improved}} := \frac{1}{2\gamma(\Pimp^{(Q)})} \log \frac{K}{\eps} \leq O(\frac{r^2}{M})\, t_{\textup{baseline}}$, $\textup{Div}_1^{\Psrw}(D_{t_{\textup{improved}}}^{\textup{improved}}) \geq 1 - \eps$, with maximum value $1$.

\end{restatable}
}
\newcommand{\cliqueproprestate}{
\eteguaranteealg*
}

%% file: sections/abstract.tex
\begin{abstract}
Self-play, a type of training algorithm that enables a model to self-improve, has recently shown promising empirical results in the context of formal theorem proving using Large Language Models (LLMs). \cite{dong2025stpselfplayllmtheorem} instantiate self-play with two cooperating agents: a prover, which proves theorems, and a conjecturer, which generates new theorems as a curriculum to the prover. In this paper, we provide a theoretical framework for understanding the self-improvement capabilities of self-play algorithms for theorem proving. First, we formalize the set of theorems as a graph, with nodes as theorems and edges between pairs of theorems with similar semantics. We introduce a set of primitive assumptions that characterize the guarantees of a trained prover and how a conjecturer can access the structure of the graph. Second, we show that if the underlying graph of theorems is well-connected, then a prover-conjecturer system, where the conjecturing algorithm is based on a reversible random walk, is sufficient to grow the set of proved theorems exponentially. Third, motivated by an issue encountered empirically by self-play algorithms, where the conjecturer tends to generate artificially complex and non-fundamental theorems, we propose a diversity measure for a training distribution of theorems generated by a conjecturer and an improved conjecturing algorithm that locally maximizes this diversity measure, by computing the diffusion similarity between neighboring theorems in the theorem graph. Finally, we describe a method to compute the diffusion similarity by using contrastive learning to embed nodes into Euclidean space and then computing the inner-product between embeddings.
\end{abstract}


%% file: sections/intro.tex
\section{Introduction}\label{sec:intro}

Formal theorem proving and mathematical reasoning with Large Language Models (LLMs) have made rapid progress, including systems that reached IMO gold-medal-level performance and systems reported to assist with open mathematical research problems \cite{imogold, bubeck2025earlyscienceaccelerationexperiments}. Among the technical ideas used to train LLM-based theorem provers, self-play via two cooperating models, a prover model and a conjecturer model, is a new and promising training technique. 

Self-play can be viewed as an automated synthetic data generation mechanism, where a system of two models generate and solve gradually-more-difficult instances of a task, improving both models with no manually annotated data required. Self-play has been used to train models to play Go \cite{silver2017masteringchessshogiselfplay}. In the context of theorem proving, self-play algorithms have been studied empirically. STP \cite{dong2025stpselfplayllmtheorem} trains two models: a prover, which maps theorem statements to their proofs, and a conjecturer, which maps theorem statements to slightly more difficult theorem statements adaptively based on the prover. The prover and conjecturer models cooperate to gradually enable the prover to prove harder theorems.



To our knowledge, this is the first theoretical analysis of self-play training algorithms for formal theorem proving. In this work, we introduce a theoretical framework, based on a graph of theorems, where the primary objective is to design a training algorithm that yields a prover with a large knowledge set---the set of theorems it can prove (\Cref{sec:framework-q1}). The framework introduces a set of primitive assumptions, characterizing how a prover can generalize in-domain when trained on a distribution of theorems and how a conjecturer can access the structure of the theorem graph. The space of training algorithms includes prover-conjecturer systems analogous to that in STP \cite{dong2025stpselfplayllmtheorem}, where the conjecturer generates a sequence of distributions over theorems, which the prover is trained on to gradually increase its knowledge. Under a connectivity condition on the underlying theorem graph, we show that the prover-conjecturer system in \Cref{alg:main}, using a simple baseline conjecturing algorithm based on a reversible random walk, guarantees that the size of the prover's knowledge set at iteration $t \in \N$ grows exponentially in $t$ (\Cref{thm:main}).




However, the size of the prover's knowledge set alone is insufficient as an objective since the knowledge set can contain many related, non-fundamental theorems (e.g. increasingly complex arithmetic problems). Empirically, STP \cite{dong2025stpselfplayllmtheorem} use heuristic elegancy filters to filter out inelegant, artificially-complicated generated-conjectures. Minimo \cite{poesia2024learningformalmathematicsintrinsic} also encounters an issue where the conjecturer generates longer and more complicated conjectures rather than deeper conjectures. The question we focus on is how to encourage the conjecturer to generate theorems whose fundamental proof ideas are more \textit{diverse}, which we believe is related to the issues encountered empirically.


Towards quantifying \textit{diversity} of a prover's knowledge set and proposing algorithmic improvements to the conjecturer, the secondary objective of a training algorithm within our framework is: given a hyperparameter $\tau \in \N$,  maximize the size of the $\tau$-step expansion of the prover's knowledge set in the theorem graph, a quantity we identify with the diversity of the prover's knowledge. This prover-centric diversity is of practical use itself because it is heuristically the number of theorems with a $\tau$-step reduction to a theorem the prover knows how to prove, which the prover can prove at inference time given $\tau$ steps of ``reasoning." This prover-centric diversity can also be thought of as quantifying the size of the boundary of the prover's knowledge set. Towards maximizing this prover-centric diversity, we propose a diversity measure for the training distribution of theorems generated by a conjecturer (which the prover is trained on) and show it lower bounds the prover-centric diversity (\Cref{lem:divt-kct-connection}). We propose an improved conjecturing algorithm (\Cref{defn:conj-alg-improved}) that generates training distributions that locally maximize this diversity measure, which requires computing a quantity called the diffusion similarity. Finally, we propose a new way to compute the diffusion similarity via contrastive learning (\Cref{thm:contembd}).

\subsection{Technical Overview}

\paragraph{Theorem Graph.} Let $V \subset \{ 0,1\}^*$ be the set of all theorems, encoded as strings according to a formal language like Lean \cite{lean}. Let $\Fprov$ be the hypothesis class of prover models (e.g. the class of pre-trained LLMs), where every prover $ \Fprov \ni f : V \to \Delta(\{ 0,1\}^*)$ maps a theorem to a distribution over proofs. A verifier $\Ver : V \times \{ 0,1\}^* \to \{ 0,1\}$ verifies if a proof of a theorem is correct. Using $\Ver$, for any prover $f \in \Fprov$ and theorem $x \in V$ define the success rate $\Succ(x,f) \in [0,1]$ as  the probability that $f$ on input $x$ generates a correct proof of $x$. For any $\varepsilon \in (0,1)$, define $E(\varepsilon, \Fprov) \subset V \times V$ as the set of pairs of theorems that are semantically similar, with respect to $\Fprov$:

\begin{align}
    E(\varepsilon, \Fprov) := \{ (x,y) \in V \times V : \forall f \in \Fprov, |\Succ(x,f) - \Succ(y,f)| \leq \varepsilon\}\label{eq:edge-ood}
\end{align}

Fix $\varepsilon$ and let $E \subset E(\varepsilon,\Fprov)$. \Cref{assumption:a3-ood} and \Cref{assumption:a4-nbr} of our framework require that for any theorem $x$,  the set $N(x) := \{ y \in V : (x,y) \in E\}$ is efficiently computable. The neighbor oracle $N : V \to 2^V$, mapping $x \rightarrow N(x)$ is a key algorithmic primitive for the conjecturing algorithms we propose later. The theorem graph is the undirected graph $G = (V,E)$ with nodes $V$ and edges $E$. 

\paragraph{Further Assumptions.} We define a conjecturer as a mapping from a distribution $D \in \Delta(V)$ of theorems (represented as a vector in $[0,1]^{|V|}$)  to a new distribution $D' \in \Delta(V)$ of theorems. One type of conjecturer is a single step of a random walk with transition matrix $P \in [0,1]^{|V| \times |V|}$, where $(D')^\top = D^\top P \in [0,1]^{1 \times |V|}$.  Our first result assumes that $|V| = \infty$ and that the random-walk-based conjecturer $P$ satisfies a connectivity condition, \Cref{assumption:graph}. When $P = \Psrw$ is a simple random walk, where a node transitions uniformly to its neighbors, then \Cref{assumption:graph} amounts to a property of $G$: that there exists $\kappa > 0$ where for any finite subset of nodes $A \subset V$ then $ \kappa \cdot |\partial A| \geq \deg(A)$ where $|\partial A|$ are the number of edges with one endpoint in $A$ and the other in $V - A$, and $\deg(A)$ is the sum of the degrees of nodes in $A$. Finally, our framework includes \Cref{assumption:a1-js} and \Cref{assumption:a2-rt}, which are statistical-learning-like assumptions which describe how a prover trained on a finite sample of i.i.d. (theorem, proof) pairs from some training distribution $D$ can generalize in-domain to $D$. 

\paragraph{A Baseline Conjecturing Algorithm.}

First we show if the theorem graph $G$ satisfies a connectivity assumption, \Cref{assumption:graph}, then a conjecturer, based on a nearest-neighbor-type reversible random walk, enables the prover to prove an exponential in $t$ number of theorems at every iteration $t \in \N$. 


\begin{theorem} (\Cref{thm:main}, Informal)
Let $G = (V,E)$ be the theorem graph, and $(V, P)$ a nearest-neighbor-type random walk, constituting a conjecturing algorithm, that satisfies the connectivity condition of Assumption \ref{assumption:graph} with Isoperimetric constant $\kappa > 0$. Suppose the framework Assumptions \ref{assumption:a1-js}-\ref{assumption:a4-nbr} hold, with universal constant $p \in (0,1)$ defined in \Cref{assumption:a2-rt}. For any training-error threshold $\eps < \frac{1}{2\kappa^2}$, the prover-conjecturer system (\Cref{alg:main}) satisfies that for all $t \in \N$, the number of theorems that the prover can prove with success rate at least $p$  at iteration $t$ is exponential in $t$. 
\end{theorem}

\input{figures/KnowledgeSet}

\paragraph{An Improved Conjecturing Algorithm for the Diversity Issue.} 

Suppose $f \in \Fprov$ is the final prover trained by a self-play algorithm like \Cref{alg:main}. Whereas \Cref{thm:main},  under \Cref{assumption:graph}, provides guarantees of $|\Kp(f)| := |\{ x \in V : \Succ(x,f) \geq p\}| $, the size of the prover's knowledge, \Cref{thm:main} does not guarantee the diversity of the prover's knowledge. For hyperparameter $\tau \in \N$, we study how to optimize a proxy for diversity of the prover's knowledge, $|\Kct(f)| := | \{ x' \in V : \exists x \in V : \Succ(x,f) \geq p \textup{ and } \dist(x,x') \leq \tau\}|$. \Cref{fig:diversity} depicts two provers, $f_1$ and $f_2$, that both can prove the same number of theorems ($|\Kp(f_1)| = |\Kp(f_2)|$), but where $f_1$'s knowledge set is much less diverse, in the sense that $|\Kct(f_1)| \ll |\Kct(f_2)|$. 


Suppose prover $f$ is trained on theorem distribution $D \in \Delta(V)$ in the sense that $\Pro_{x \sim D}[\Succ(x,f) \geq p] \geq 1 - \eps$ (as in \Cref{assumption:a2-rt}). \Cref{lem:divt-kct-connection} shows that in order to maximize $|\Kct(f)|$, it is reasonable to maximize the following property of $f$'s training distribution: $\Divt(D) := \frac{1}{||D^\top P^\tau||^2_{\ell^2(1/\pi)}} = \frac{1}{\sum_{x \in V }(D^\top P^\tau e_x)^2 \frac{1}{\pi(x)} }$, where $P$ is the transition matrix of a reversible, nearest-neighbor Markov chain. With this observation, the goal is to design a conjecturing algorithm that generates training distributions $\{ D_t\}_{t \in \N}$ where $\Divt(D_t)$ is large for each $t \in \N$. We propose \Cref{defn:conj-alg-improved}, an improved conjecturer that defines a push-forward map $C : V \to \Delta(V)$ where $\forall x \in V$, $C(x) \in \Delta(N(x))$. The weighting of $C(x)$ is determined by maximizing $\Divt(C(x))$.

\begin{align}
    \forall x \in V, \, C(x) &\leftarrow \arg\max_{\mu \in \Delta(N(x))} \Divt(\mu) =  \arg\min_{\mu \in \Delta(N(x))} \mu^\top P^\tau \diag(1/\pi) (P^\tau)^\top \mu 
\end{align}

\Cref{defn:conj-alg-improved} needs to compute the diffusion similarity $\diffsim(y, y') := e_y^\top P^\tau \diag(1/\pi) (P^\tau)^\top e_{y'}$ for every pair of neighbors $y,y' \in N(x)$ of $x$. \footnote{To further motivate diffusion similarity, note it is closely related to diffusion distance, used for graph partitioning \cite{diffmapscoarsegrain}.} We show the diffusion similarity can be computed by embedding each node in $G$ into Euclidean space via an embedding model learned with contrastive learning, and then computing the inner product between embeddings. This can be used as a subroutine in a  conjecturing algorithm like \Cref{defn:conj-alg-improved}.

\begin{theorem} (\Cref{thm:contembd}, Informal)
    Let $G = (V,E)$ be an undirected, connected theorem graph. Let $\F_V^{(|V|)}$ be the class of all mappings from $V$ to $\R^{|V|}$, and $\Femb \subset \F_V^{(|V|)}$ a class of embedding models. For any $\tau \in \N$, $\Lc^\tau: \F_V^{(|V|)} \to \R$ is a contrastive loss with time-scale parameter $\tau$ (\Cref{eq:pop-cont-loss}). Assume $\Femb$ contains a global minimizer of $\Lc^\tau(\cdot)$. For any such global minimizer $h_{\theta_*} \in \Femb$, then $\forall u,v \in V,\; \diffsim(u,v) = h_{\theta_*}(u)^\top h_{\theta_*}(v)$.
\end{theorem}

Under additional assumptions, we prove \Cref{thm:e2e-contrastive-learning-result}, which is an extension to the result of \Cref{thm:contembd} that applies for classes of embedding models $\Femb$ where models $\Femb \ni h_\theta : V \to \R^k$ have small embedding dimension $k \ll |V|$. \Cref{thm:e2e-contrastive-learning-result} provides approximation-guarantees for an embedding model that is learned via empirical risk minimization (ERM).

%% file: figures/KnowledgeSet.tex
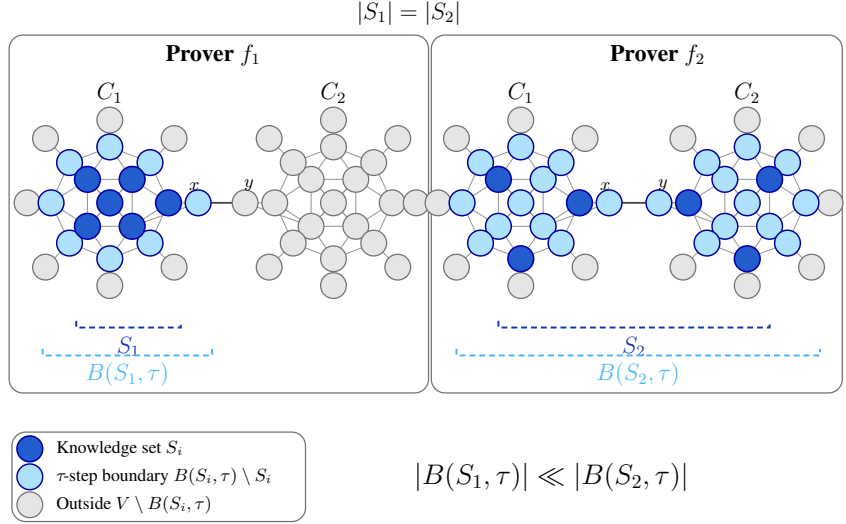
\begin{figure}[t]
\centering
\resizebox{0.8\linewidth}{!}{
\begin{tikzpicture}[font=\rmfamily]
  \MakeClusterCoords{Lone}{-5.35}{1.05}{1}
  \MakeClusterCoords{Ltwo}{-1.35}{1.05}{-1}
  \MakeClusterCoords{Rone}{ 2.00}{1.05}{1}
  \MakeClusterCoords{Rtwo}{ 6.05}{1.05}{-1}

  \node[font=\bfseries\Large] at (0,5.05)
    {Same-size knowledge sets, different $\tau$\!-step boundary};
  \node[font=\large] at (0,4.45) {$|S_1|=|S_2|$};

  \draw[panel] (-7.15,-2.35) rectangle (0.35,4.05);
  \draw[panel] (0.40,-2.35) rectangle (7.75,4.05);
  \node[font=\bfseries\large] at (-3.50,3.70) {Prover $f_1$};
  \node[font=\bfseries\large] at ( 4.45,3.70) {Prover $f_2$};

  \node[font=\large] at (-5.35,3.02) {$C_1$};
  \node[font=\large] at (-1.35,3.02) {$C_2$};
  \node[font=\large] at ( 2.00,3.02) {$C_1$};
  \node[font=\large] at ( 6.05,3.02) {$C_2$};

  \DrawClusterEdges{Lone}
  \DrawClusterEdges{Ltwo}
  \DrawClusterEdges{Rone}
  \DrawClusterEdges{Rtwo}
  \draw[bridgeedge] (Lone-h) -- (Ltwo-h);
  \draw[bridgeedge] (Rone-h) -- (Rtwo-h);

  %
  \DrawNodes{Lone}%
    {q0,q1,q2,q3,q4,r3}
    {r1,r2,r4,r5,r6,r7,r8,h}
    {l1,l2,l4,l5,l6,l7,l8}

  \DrawNodes{Ltwo}%
    {}
    {}
    {q0,q1,q2,q3,q4,r1,r2,r3,r4,r5,r6,r7,r8,l1,l2,l4,l5,l6,l7,l8,h}

  \DrawNodes{Rone}%
    {q1,r3,r5}
    {q0,q2,q3,q4,r1,r2,r4,r6,r7,r8,h}
    {l1,l2,l4,l5,l6,l7,l8}

  \DrawNodes{Rtwo}%
    {q1,r3,r5}
    {q0,q2,q3,q4,r1,r2,r4,r6,r7,r8,h}
    {l1,l2,l4,l5,l6,l7,l8}

  \node[font=\small, above=3pt, xshift=-2pt] at (Lone-h) {$x$};
  \node[font=\small, above=3pt, xshift= 2pt] at (Ltwo-h) {$y$};
  \node[font=\small, above=3pt, xshift=-2pt] at (Rone-h) {$x$};
  \node[font=\small, above=3pt, xshift= 2pt] at (Rtwo-h) {$y$};

  \Bracket{-5.95}{-4.08}{-1.18}{S_1}{bracketdark}
  \Bracket{-6.55}{-3.52}{-1.68}{B(S_1,\tau)}{bracketlight}

  \Bracket{ 1.60}{ 6.45}{-1.18}{S_2}{bracketdark}
  \Bracket{ 0.85}{ 7.35}{-1.68}{B(S_2,\tau)}{bracketlight}

  \draw[panel] (-7.10,-4.65) rectangle (-1.85,-3.05);
  \node[know] at (-6.78,-3.35) {};
  \node[anchor=west, font=\small] at (-6.43,-3.35) {Knowledge set $S_i$};
  \node[bdy] at (-6.78,-3.85) {};
  \node[anchor=west, font=\small] at (-6.43,-3.85) {$\tau$\!-step boundary $B(S_i,\tau)\setminus S_i$};
  \node[outside] at (-6.78,-4.35) {};
  \node[anchor=west, font=\small] at (-6.43,-4.35) {Outside $V\setminus B(S_i,\tau)$};

  \node[font=\Large] at (2.55,-3.85) {$|B(S_1,\tau)|\ll |B(S_2,\tau)|$};
\end{tikzpicture}
}
\caption{Let $\tau = 1$. (Left) Prover $f_1$'s $\tau$-expanded knowledge set is only supported on the left cluster. (Right) Prover $f_2$'s $\tau$-expanded knowledge set is well-supported on both clusters. Their knowledge sets $S_1,S_2$ are equal in size (6), but $S_2$ is more diverse since its $\tau$-step expansion is larger.}
\label{fig:diversity}
\end{figure}

%% file: sections/arxiv/intro-arxiv-diff.tex
In short, our contributions are:

\begin{enumerate}
    \item We provide a theoretical framework for understanding and designing self-play theorem proving algorithms. The framework formalizes the set of theorems as a graph, where nodes are theorems. It introduces a set of primitive assumptions (\Cref{assumption:a1-js} -- \ref{assumption:a4-nbr}) that describe the properties of the graph and the computational primitives available to a self-play algorithm. 
    \item We show in \Cref{thm:main} that if the theorem graph and conjecturer satisfy a well-connectedness property (\Cref{assumption:graph}), the number of theorems the prover-conjecturer system (\Cref{alg:main}) proves grows exponentially with the number of self-play iterations.
    \item Motivated by an issue encountered empirically by self-play algorithms where the conjecturer generates artificially complex and non-fundamental theorems, we propose a diversity measure for a training distribution of theorems generated by a conjecturer (\Cref{defn:diversity-measure}) and an improved conjecturer (\Cref{defn:conj-alg-improved}) that locally maximizes this diversity measure, by computing the diffusion similarity (\Cref{defn:diffusionembd}) between neighboring theorems in the theorem graph. 
    \item In \Cref{thm:contembd} and \Cref{thm:e2e-contrastive-learning-result}, we describe a method to compute the diffusion similarity of two nodes of the theorem graph by using contrastive learning to embed nodes into Euclidean space and then computing the inner-product between the embeddings.
\end{enumerate}

%% file: sections/arxiv/related-work-arxiv.tex
\section{Related Work}

\paragraph{LLM-Based Formal Theorem Proving Systems.} Formal theorem proving with LLMs in Lean \cite{lean} is commonly divided into two paradigms: step-level interaction and whole-proof generation. Step-level methods interact with Lean incrementally, proposing one tactic at a time \cite{Achim2025Aristotle,Hubert2025OlympiadFormalMath,Wu2024InternLM25StepProver,Xin2025BFSProver,Xin2025StepProverScaling}. Whole-proof methods instead generate a complete Lean proof in a single interaction \cite{Chen2025SeedProver,Ji2025LeanabellProverV2,Lin2025GoedelProver,Lin2025GoedelProverV2,Xin2025DeepSeekProverV15,Ren2025DeepSeekProverV2,Shang2025StepFunProverPreview,Wang2025KiminaProverPreview,Zhou2025DeltaProver}.

\paragraph{Self-play for Formal Theorem Proving.} Regarding self-play theorem proving systems, STP instantiates a prover-conjecturer self-play loop from pre-trained LLMs, where the conjecturer is trained to generate synthetic theorems that are barely provable by the current prover \cite{dong2025stpselfplayllmtheorem}. Minimo instantiates a prover-conjecturer self-play loop to learn a prover from scratch, where the conjecturer generates conjectures via constrained decoding and the prover guides an MCTS proof search \cite{poesia2024learningformalmathematicsintrinsic}. In contrast to these empirical works, our work develops a theoretical framework for analyzing when prover-conjecturer self-play can expand the set of provable theorems.

\paragraph{Theoretical Analyses of Self-play Algorithms.} Existing theoretical analyses of self-play largely focus on adversarial settings, such as zero-sum Markov games. The theorem-proving setting we study is different: the prover and conjecturer have asymmetric but cooperative roles. \cite{bai2020nearoptimalreinforcementlearningselfplay} provide a theoretical analysis of self-play algorithms in zero-sum Markov Games, showing convergence to a Nash Equilibrium \cite{bai2020nearoptimalreinforcementlearningselfplay}.  The Wasserstein GAN training algorithm \cite{arjovsky2017wassersteingan} is a theoretically-grounded asymmetric self-play algorithm, but still where the agents are adversarial to each other. Our work also differs from works that study how to  learn a model to verify theorems in a purely statistical setting \cite{balcan2026learningverifiersimplicationschainofthought}. Our setting is non-statistical since the goal of the training algorithm is to allow the final model to generalize beyond its starting distribution, not just to generalize well on its starting distribution given i.i.d. samples. However, we will assume some statistical-learning primitives (\Cref{assumption:a2-rt}) in our framework.

\paragraph{Theorem Graphs.} Prior works formalize the graph or hyper-graph structure of mathematical theorems \cite{geuvers, freedman}. In \cite{geuvers}, each node represents a statement; the statement corresponding to a node can be deduced in one step via a deduction rule from the statements corresponding to its descendants. \cite{freedman} formalizes nodes as statements, and each $(p,q)$ hyper-edge corresponds to applying a deduction rule on $p$ input statements to yield $q$ output conclusions. Axioms of the logical system are root nodes. In relation to these works, our notion of a theorem graph still represents  mathematical statements as nodes. However, the definition of an edge is not defined in terms of deduction rules in a particular logical system. Rather, edges are determined by the class of prover models $\Fprov$, where an edge between two theorems $(x,y)$ indicates that for every prover $f \in \Fprov$, the success rate of $f$ on theorem $x$ is not too different from the success rate of $f$ on theorem $y$.

\paragraph{Graph Exploration Algorithms, Graph Diffusion, Contrastive Learning.} Given a graph, the question of how to design an algorithm to explore the graph efficiently has been studied in \cite{kim2025metastabledynamicschainofthoughtreasoning}. They formalize chain-of-thought reasoning as a random walk over a graph of ideas, where the goal of the reasoner is to hit a particular node of the graph, starting at a designated start node. The underlying graphs they study are metastable graphs \cite{bovier}. The question we study in \Cref{sec:diversity-q2} is conceptually related to \cite{kim2025metastabledynamicschainofthoughtreasoning}, in that we want to find a graph algorithm that can efficiently explore a graph with poor connectivity, such as metastable graphs. However, our proposed graph-exploration algorithms are different from \cite{kim2025metastabledynamicschainofthoughtreasoning}. In particular, our algorithm uses a similarity measure between nodes in the graph, called the diffusion similarity, that is defined in terms of a random walk on the graph. The diffusion similarity measure is closely related to the diffusion distance, a graph metric which is useful for characterizing the clustering-structure of a graph \cite{diffusionmaps}. Finally, we propose a new way to compute the diffusion similarity, using a version of contrastive learning formalized by \cite{haochen2022provableguaranteesselfsuperviseddeep}, who provide a theoretical characterization of the global minimizers of a particular contrastive loss in terms of the minimizers of a related matrix factorization objective.

%% file: sections/arxiv/preliminaries-arxiv.tex
\section{Preliminaries}

\paragraph{Notation.} $0 \in \N$. $[n] := \{ 1, 2, \ldots, n\}$. For a set $V$, $2^V$ is its power set. For $v \in \R^n$, $\diag(v) \in \R^{n \times n}$ is the diagonal matrix with $v$ on its diagonal. For $M \in \R^{n \times n}$, $i,j \in [n]$, denote $M[i,j]$ the $(i,j)$th entry, $M[i]$ the $i$th row, and $M[:, j]$ the $j$th column of $M$. $O^{k \times k}$ is the set of $(k\times k)$ orthogonal matrices. $\1$ is the all-$1$'s vector. $\forall x \in V$, $e_x \in \R^{|V|}$ is the unit vector along dimension indexed by $x \in V$. For an undirected, unweighted graph $G = (V,E)$, for any $x,y \in V$, \, $\dist(x,y)$ is the minimum number of edges of any path connecting $x$ and $y$ in $G$. $\deg(x)$ denotes the degree of node $x \in V$. For $x \in V, r \in \N$, denote $B(x,r) := \{ x' \in V : \dist(x, x') \leq r\}$ and $\partial B(x,r) := B(x,r) - B(x, r - 1), \, B(x, -1) := \emptyset$. For set $S \subset V,$ let $B(S, r) := \bigcup_{v \in S} B(v,r)$. For $\pi,f, f' : V \to (0, \infty)$, let $\langle f, f' \rangle_{\ell^2(1/\pi)} := \sum_{x \in V} f(x) f'(x) \frac{1}{\pi(x)}$ and  $||f||^2_{\ell^2(1/\pi)} := \langle f, f \rangle_{\ell^2(1/\pi)}$. Let $\Delta(V)$ denote the set of probability distributions whose support is contained in the set $V$. For $D \in \Delta(V)$, $\supp(D)$ is its support. For $g : V \to \Delta(V)$, $g \# D \in \Delta(V)$ is the distribution of $y$, drawn according to this process: $x \sim D, y \sim g(x)$. We say an algorithm has independent and identically distributed (i.i.d.) sample access to $D$ if it has an oracle which it can query for i.i.d. samples $x \sim D$. 


\paragraph{Markov Chains on Graphs and Isoperimetric Inequality.} \Cref{appen:more-detailed-prelim} has more-detailed preliminaries. Let $G = (V,E)$ be an undirected, connected graph with a countable set of vertices. Let $(V,P)$ denote an irreducible, reversible Markov chain with transition matrix $P \in [0,1]^{|V| \times |V|}$ and invariant measure $\pi : V \to (0,\infty)$ where reversibility means that $\forall x,y \in V, \pi(x) P[x,y] = \pi(y) P[y, x]$. Let $E(P)$ be the set of unordered tuples $(x,y)$ where $P[x,y] > 0$. Given $e \in E(P)$ we will write $e = (e^-, e^+)$, but the choice of which end-point is $e^-$ and which is $e^+$ does not matter for our analysis. If $E(P) \subset E$, then $(V,P)$ is nearest-neighbor-type random walk (Markov chain). The simple random walk $\Psrw$ is a Markov chain where every node $x$ transitions uniformly to its neighbors $N(x)$. 
For any $A \subset V$, let $\partial A$ be the set of all edges in $E(P)$ with one endpoint in $A$ and the other in $V - A$.  Let $\pi(A) := \sum_{x \in A} \pi(x)$ and $a(\partial A) := \sum_{(e^-, e^+) \in \partial A} \pi(e^-) P[e^-, e^+] = \sum_{e^- \in A, e^+ \notin A} \pi(e^-) P[e^-, e^+]$. 




\begin{restatable}[Isoperimetric Inequality]{definition}{IS}\label{defn:IS}
A Markov chain $(V,P)$ satisfies Strong Isoperimetric Inequality, $\ISineq$, if there exists a constant $\kappa > 0$ such that for every subset of finite cardinality $A \subset V$, $\pi(A) \leq \kappa \cdot a(\partial A)$. 
\end{restatable}

\paragraph{Contrastive Learning.} \Cref{appen:cont-learning-intro} contains more detailed preliminaries. Contrastive learning defines an algorithm for learning an embedding model that maps elements from a domain $V$ (e.g. the set of real-world images) into Euclidean space $\R^k$ for some $k \in \N$. It is most notably used for training neural networks as embedding models \cite{hadsell2006dimensionality, oord2019representationlearningcontrastivepredictive, chen2020simpleframeworkcontrastivelearning}. We will use the setup of \cite{haochen2022provableguaranteesselfsuperviseddeep}, which we describe below.

Let $\F_{V}^{(k)}$ denote the class of all functions mapping $V$ to $\R^{k}$. Let $\Femb \subset \F_{V}^{(k)}$ denote the hypothesis class of models $h_\theta : V \to \R^k$ which we assume contains the ground-truth embedding model of the learning task. Contrastive learning defines two distributions of pairs of elements in $V$: a positive pair distribution $D_+ \in \Delta(V \times V)$  and a negative pair distribution $D_- \in \Delta(V \times V)$. Given i.i.d. samples from both distributions, the trained model is encouraged to embed two elements sampled from a positive pair $(x,y) \sim D_+$ close together in Euclidean space, while embedding two elements sampled from a negative pair $(x,x') \sim D_-$ far apart in Euclidean space. These two distributions of pairs of elements are used to define a contrastive loss, and training the model amounts to minimizing an empirical estimate of the contrastive loss from finite samples.

Let the domain $V$ consist of $N \in \N$ elements, where $N$ is large but finite. We will identify $V$ with $[N]$ here for simplicity. Define $\mu \in \Delta([N])$ as the \textit{base distribution} and $Q \in \R^{N \times N}$ as an augmentation scheme. Overloading notation, we can alternatively view $Q : [N] \to \Delta(N)$ as a mapping from an element $x$ to a distribution $Q(x)$ over augmentations of $x$. For every $x \in [N]$, $Q(x)$, equals the $x$-th row of the matrix-representation of $Q$. For $h_\theta : V \to \R^k$, define the population contrastive loss $\Lc : \F_{V}^{(k)} \to \R$ as

\begin{align}
    \Lc(h_\theta) := \ex_{(x, x^-) \sim D_-} (h_\theta(x)^\top h_\theta(x^-))^2 - 2 \cdot \ex_{(x, x^+)\sim D_+}h_\theta(x)^\top h_\theta(x^+)\label{eq:pop-cont-loss}
\end{align}

Positive pairs $(x, x^+)\sim D_+$ are sampled by first sampling an element from the base distribution, $z \sim \mu$ and then sampling the pair elements as independent augmentations $x \sim Q(z)$, $x^+\sim Q(z)$, returning $(x,x^+)$. Negative pairs $(x,x^-) \sim D_-$ are sampled by first sampling $z, z^- \sim \mu$ independently and then sampling $x \sim Q(z), x^- \sim Q(z^-)$, returning $(x,x^-)$.

For a dataset $\{ \overline{x}_i\}_{i \in [n]} \subset V$, define the empirical contrastive loss $\Lchat : \F_{V}^{(k)} \to \R$ as:

\begin{align}
    &\Lchat(h_{\theta'}) :=\\
    &-\frac{2}{n} \sum_{i \in [n]} \ex_{\substack{x \sim Q(\cdot | \overline{x}_i),\\ x^+ \sim Q(\cdot | \overline{x}_i)}} (h_{\theta'}(x)^\top h_{\theta'}(x^+)) + \frac{1}{n(n - 1)} \sum_{i,j \in [n], i \neq j}\ex_{\substack{x \sim Q(\cdot | \overline{x}_i),\\ x^- \sim Q(\cdot | \overline{x}_j)}} (h_{\theta'}(x)^\top h_{\theta'}(x^-))^2\label{eq:emp-cont-loss-defn}
\end{align}

Claim D.2 of \cite{haochen2022provableguaranteesselfsuperviseddeep} shows that $\ex_{\{ \overline{x}_i\}_{i \in [n]} \sim \mu} \Lchat(h_{\theta}) = \Lc(h_\theta)$, i.e. that $\Lchat$ is unbiased. Define the empirical risk minimizer (ERM) as follows.

\ermdefnmaintxt



\Cref{defn:contrastive-learning} in \Cref{appen:cont-learning-intro} is an algorithm that learns an embedding model using empirical risk minimization, given i.i.d. sample access to base distribution $D \in \Delta(V)$ and augmentation function $Q : V \to \Delta(V)$. To quantify the statistical generalization properties of the ERM, we introduce the following notion of Rademacher complexity.

\raddefnmaintxt

%% file: sections/framework-q1.tex
\section{Framework and Random-Exploration Baseline}\label{sec:framework-q1}

\subsection{Primitive Assumptions}\label{sec:4-1-asmpt}


With $t \in \mathbb{N}$ denoting the iteration index, the primary goal of the training algorithm is to produce a prover $f$ which can prove as many theorems as possible. One way a training algorithm can achieve this is by producing a sequence of provers $\{ f_t\}_{t \in \N}$ where the prover $f_{t + 1}$  can prove more theorems than prover $f_t$ for every $t \in \N$. The following assumptions define a set of core primitives for designing such a training algorithm. 

\Cref{assumption:a1-js} assumes i.i.d. sample access to an initial training distribution $D_0$ of theorems where for every theorem $x$ in the support of $D_0$, the proof of $x$ is known. This means that we have efficient access to a proof of $x$, either because it is recorded explicitly in some database of (theorem, proof) pairs, or because we have access to a prover $f$ where $\Succ(x,f) \geq \Omega(1)$. The latter condition means that given $x$ and $f$, the proof of $x$ can be efficiently sampled in $O(1)$ calls to $f$, with high probability.

\begin{assumption}(Cold-Start Dataset)\label{assumption:a1-js} We are given access to an oracle that provides i.i.d. samples from $D_0 \in \Delta(V)$ where the proof of every theorem $x \in \supp(D_0)$ is known.  
\end{assumption}

An example of $D_0$ would be a distribution of elementary-school-level math problems. 

\Cref{assumption:a2-rt} describes the ability of a model, trained on a distribution $D$ of theorems via i.i.d. samples, to generalize in-domain to $D$. We justify it with statistical learning theory in \Cref{appen:more-detailed-asmpts}. 

\begin{assumption}(Training)\label{assumption:a2-rt} Let $p \in (0,1), p = \Theta(1)$ be a universal constant.  Let $D \in \Delta(V)$ be a distribution of theorems where $\forall x \in \supp(D)$, the proof of $x$ is known. Then, $\forall \, \eps,\delta \in (0,1)$, with $n = \Theta(\frac{\C(\Fprov)\log \frac{1}{\delta}}{\eps^2 (1 - p)^2})$ where $\C(\Fprov)$ is an $\Fprov$-dependent constant, we can train a prover $f$ using $n$ i.i.d. samples from $D$ such that with probability at least $1 - \delta$, we have that $f$ satisfies:

\begin{align}
    \Pro_{x \sim D} [\Succ (x, f) \geq p] \geq 1- \eps
\end{align}

\end{assumption}




The requirement of \Cref{assumption:a2-rt} that the proof of every theorem in $\supp(D)$ is known can also be thought of as requiring  a training distribution of (theorem, proof) pairs, where every theorem in the support of the training distribution must have a correct proof associated with it. This requirement is necessary, as otherwise it would be possible to train a model to have high success rate on arbitrarily hard theorems for which we do not know the proof.

Let $G = (V,E)$ be the theorem graph. \Cref{assumption:a3-ood} says that the edge set $E$ is where the success rate of any prover in $\Fprov$ is not too different between two theorems connected by an edge.

\begin{assumption}(Continuity of Success Rate)\label{assumption:a3-ood} With $p \in (0,1)$ defined as in \Cref{assumption:a2-rt}, assume for all $f \in \Fprov$, for all $(x,y) \in E$, that $|\Succ (x, f) - \Succ(y, f)| \leq \frac{p}{2}$
\end{assumption}

In \Cref{assumption:a4-nbr}, we assume access to a neighbor oracle $N$ as a primitive. 


\begin{assumption}(Neighbor Access)\label{assumption:a4-nbr} For any $x \in V,$ let $N(x) := \{ y \in V : (x,y) \in E\}$. We assume oracle access to a neighbor oracle $N : V \to 2^V$, which on input of any $x \in V$, returns  $N(x)$. 
\end{assumption}

Assumptions \ref{assumption:a2-rt} -- \ref{assumption:a4-nbr} enable a training algorithm to run the following at each iteration. First, expand the current set of known theorems by one step via the neighbor oracle $N$ (\Cref{assumption:a4-nbr}). Second, \Cref{assumption:a3-ood} ensures that unseen theorems have success rate at least $p/2$ under the current prover, enabling the proofs of the unseen theorems to be efficiently sampled. In general, note that if we have access to a prover $f$ whose success rate on theorem $x$ is reasonably high ($\geq p/2$, where $p \geq \Omega(1)$ by \Cref{assumption:a2-rt}), then the proof of $x$ can be efficiently sampled and become known. Finally, a new prover is trained on the expanded set of theorems  via \Cref{assumption:a2-rt}. \Cref{assumption:a3-ood} is necessary to our analysis because standard statistical learning theory arguments rely on exchangeability between training and test distributions and therefore do not by themselves justify this out-of-distribution step.





We do not require the edge set $E$ to have an interpretable definition. We only require $E$ to satisfy \Cref{assumption:a3-ood} and \Cref{assumption:a4-nbr}. One example instantiation of the theorem graph is where \(V\) is the set
of theorems corresponding to discrete logarithm instances in a fixed cyclic group. Fix a cyclic
group \(\mathbb G\) of prime order \(q\) with generator \(g \in \mathbb G\). For each
\(a\in \mathbb Z_q\), let \(h_a=g^a\), and let \(x_a\in V\) denote the
theorem:
\[
    \text{``Given cyclic group $\mathbb{G}$ of prime order $q$, and given $g, g^{-1}, h_a \in \mathbb{G},\text{ there exists } z\in \mathbb Z_q \text{ such that } g^z=h_a$.''}
\]
The proof of \(x_a\) is the unique exponent \(f_*(x_a)=a\). Define the
edge set \(E\) by $(x_a,x_b)\in E \iff b-a\in\{-1,+1\}\pmod q .$ Thus, the theorem graph is a cycle over the discrete logarithm instances. This edge set satisfies \Cref{assumption:a4-nbr}, since given the theorem
\(x_a=(\mathbb G,g, g^{-1},h_a)\), the neighbor oracle can
compute $N(x_a)=\{x_{a+1},x_{a-1}\}$, where \(x_{a+1}\) has target \(h_a g=g^{a+1}\), and \(x_{a-1}\) has
target \(h_a g^{-1}=g^{a-1}\). Finally, \Cref{assumption:a3-ood} amounts to the property that adjacent
discrete-log instances have similar success rates for the prover class: $(x_a,x_{a + 1})\in E
    \implies
    \forall f\in\Fprov,
    |\Succ(x_a,f)-\Succ(x_{a + 1},f)|\leq \frac p2 .$
This property is conceivable for some prover classes since  \(x_a\) and \(x_{a+1}\) are not
substantially different: given \(g\) and
its inverse \(g^{-1}\), an instance with target \(h_{a+1}\) can be
reduced to the adjacent instance with target \(h_a\) by multiplying by
\(g^{-1}\), i.e. $h_{a+1}g^{-1}=g^{a+1}g^{-1}=g^a=h_a.$ Thus, any prover that can prove \(h_a\) can prove \(h_{a+1}\) by first reducing \(h_{a+1}\) to \(h_a\), running the original prover, and then adding \(1\) to the resulting exponent. For model classes \(\Fprov\) that can exploit such reductions, it is conceivable that neighboring instances \(x_a\) and \(x_{a\pm 1}\) have similar success rates.

\subsection{Random-Walk Conjecturing Algorithm}\label{sec:4-2-conj-alg}

\paragraph{Conjecturer.}\label{par:conjecturer-defn} In STP \cite{dong2025stpselfplayllmtheorem}, the conjecturer is a mapping that takes as input a theorem $x$, its proof $f_*(x)$, a lemma $l$ used in $f_*(x)$, and a system prompt; it returns a theorem $c$ related to $x$. We study a simplification where the conjecturer maps a theorem to a distribution of theorems.




\begin{definition}\label{defn:conjecturer} (Conjecturer) For graph $G = (V,E)$, a conjecturer is a mapping $C : V \to \Delta(V)$. Given i.i.d. sample access to input distribution $D \in \Delta(V)$, the conjecturer  returns i.i.d. sample access to $D' = C \# D \in \Delta(V)$. 


A reversible conjecturer is a special type of conjecturer, specified by a set of symmetric, non-negative edge weights $w : E \to \R_+$, where $\forall x,y \in V, w(x,y) = w(y,x) \geq 0$. Denote $P_w \in [0,1]^{|V| \times |V|}$ as the matrix representation of $w$, defined as: $\forall x, y \in V, P_w[x,y] = \frac{w(x,y)}{\sum_{z \in N(x)} w(x,z)}$. Denote $C_w: V \to \Delta(V)$ as the sampler representation of $w$, defined as: $\forall x \in V, C_w(x) = P_w[x]$. Note that the transition matrix $P_w$ defines a nearest-neighbor-type reversible random walk over state-set $V$, with stationary measure $\pi_w : V \to (0,\infty)$, where $\forall x \in V, \pi_w(x) \propto \sum_{z \in N(x)} w(x,z) =: w(x)$. 

These three representations $C_w, P_w, w$ are equivalent and uniquely determined given any of the other. Sometimes we will write the tuple $(C_w, P_w, w)$, where $C_w$ and $ P_w$ are defined in terms of $w$ as above. When the context requires just $C_w$, we will identify $C_w$ with the conjecturer.

\end{definition}



For instance, when $\forall x \in V, z \in N(x), w(x,z) = 1$ and for $z \notin N(x), w(x,z) = 0 $, then $C_w(x) = \textup{Unif}(N(x))$, so that  $(C_w, P_w, w)$ defines a simple random walk over $G$, with $\pi_w(x) \propto \deg(x)$.


\paragraph{Prover-Conjecturer System.} This section provides sufficient conditions for a prover-conjecturer system (\Cref{alg:main}) to learn to prove more theorems across iterations. \Cref{thm:main} assumes the conjecturer $(C_w, P_w, w)$, defining Markov chain $(V, P_w)$, satisfies a connectivity assumption.

\begin{assumption}\label{assumption:graph}
    Let $G = (V, E)$ be a connected graph where the number of vertices is  countably infinite ($|V| = \infty$) and each vertex has finite degree. Assume that the Markov chain $(V,P)$ for $G$ satisfies the strong isoperimetric inequality, $\ISineq$, defined in \Cref{defn:IS}.
\end{assumption}

At iteration $t\in \N$, the prover-conjecturer system (\Cref{alg:main}), initialized with conjecturer $(C_w, P_w, w)$, has i.i.d. sample access to $D_t \in \Delta(V)$, a training distribution over theorems whose proofs are known. It trains prover $f_t \in \Fprov$ on $D_t$ via \Cref{assumption:a2-rt}. Finally, it constructs i.i.d. sample access to a new distribution over theorems, $D_{t + 1} \in \Delta(V)$ using \Cref{alg:pass-rate-sampler}. Given $D_t$, \Cref{alg:pass-rate-sampler} applies the conjecturer $C_w$ to $D_t$, yielding $C_w \# D_t$. Theorems with low pass rate under $f_t$ are filtered out, yielding $D_{t + 1}$. 

\Cref{alg:main} uses \Cref{alg:pass-rate-sampler} as a subroutine to construct $D_{t + 1}$ from $D_t$.  \Cref{alg:verifier-based-sampler} is a practical approximation to \Cref{alg:pass-rate-sampler}, and it can be substituted in for \Cref{alg:pass-rate-sampler} in \Cref{alg:main}. In \Cref{alg:verifier-based-sampler}, for a sufficiently large rollout-budget $m \geq \Theta(\frac{1}{p^2}\log\frac{1}{\delta_{\textup{samp}}})$, by standard concentration inequalities \cite{hoeffding1963probability, ma}, we have that $\forall x \in V$, with probability at least $1 - \delta_{\textup{samp}}$, $|\frac{1}{m} \sum_{i \in [m]} \1[\Ver(x, \pi_i) = 1] - \Succ(x,f) | < p/2$. Thus, with high probability, if $\Succ(x,f) \geq p/2$, then $\frac{1}{m} \sum_{i \in [m]} \1[\Ver(x, \pi_i) = 1] > 0$, so that at least one empirically-sampled proof $\pi$ of theorem $x$ is correct. 



\begin{algorithm}[t]
\caption{Prover-Conjecturer System}
\label{alg:main}
\begin{algorithmic}[1]
\Require Prover class $\Fprov$, reversible conjecturer $C_w : V \to \Delta(V)$ as in \Cref{defn:conjecturer}, parameters $\eps,\delta\in(0,1)$, threshold $p>0$
\State Acquire theorem distribution $D_0\in\Delta(V)$ using \Cref{assumption:a1-js}
\For{$t=0,1,2,\ldots$}
    \State Given i.i.d. sample access to $D_t$, train prover $f_t\in\Fprov$ on $D_t$ to accuracy $\eps$ and failure probability $\delta$ using \Cref{assumption:a2-rt}
    \State Define i.i.d. sample access to $D_{t+1}$ by applying \Cref{alg:pass-rate-sampler} with input $(D_t,f_t,C_w,p/2)$
\EndFor
\end{algorithmic}
\end{algorithm}

\begin{algorithm}[t]
\caption{Pass-Rate-Oracle-Based Sampler}
\label{alg:pass-rate-sampler}
\begin{algorithmic}[1]
\Require i.i.d. sample access to $D \in \Delta(V)$, prover $f \in \Fprov$, conjecturer $C_w : V \to \Delta(V)$, threshold $q>0$, verifier $\Ver : V \times \{ 0,1\}^* \to \{0,1\}$,  and oracle access to $\Succ(\cdot,f)$
\Ensure i.i.d. sample access to $D' \in \Delta(V)$

\Sampler{$D'$}
    \Repeat
        \State Draw theorem $x\sim D$
        \State Sample $x'\sim C_w(x)$\label{algline:conj}
    \Until{$\Succ(x',f)\ge q$}
    \State Sample and record a correct proof of $x'$ from $f$, checking correctness using $\Ver$.
    \State \Return $x'$
\EndSampler
\end{algorithmic}
\end{algorithm}

\begin{algorithm}[t]
\caption{Empirical-Pass-Rate-Based Sampler}
\label{alg:verifier-based-sampler}
\begin{algorithmic}[1]
\Require i.i.d. sample access to $D \in \Delta(V)$, prover $f \in \Fprov$, conjecturer kernel $C_w$, rollout budget $m\in\N$, and verifier $\Ver : V \times \{ 0,1\}^* \to \{0,1\}$.
\Ensure i.i.d. sample access to distribution $D' \in \Delta(V)$

\Sampler{$D'$}
    \Repeat
        \State Draw theorem $x\sim D$
        \State Sample $x'\sim C_w(x)$
        \State Sample candidate proofs of $x'$ from $f$ as:  $\pi_1,\ldots,\pi_m \sim f(x')$
        \If{there exists $j\in[m]$ such that $\Ver(x',\pi_j)=1$}
            \State Record proof $\pi_j$ of theorem $x'$.
            \State \Return $x'$
        \EndIf
    \Until{a theorem $x'$ is returned}
\EndSampler
\end{algorithmic}
\end{algorithm}

Due to the push-forward construction of $D_{t + 1}$ from $D_t$ in \Cref{alg:main} for every $t \in \N$, sampling from $D_t$ requires first sampling $x_0 \sim D_0$, then applying $t$ push-forward transformations to yield $x_i \sim D_j$, for all $j \in [t]$. In particular, first draw a large number $N_0$ of i.i.d. samples $x_0^{(i)} \sim D_0$, $i \in [N_0]$. At every iteration $j \in \{ 0\} \cup [t]$, suppose inductively that there are $N_j$ remaining unused i.i.d. samples from $D_j$. Use $\Theta(\frac{\C(\Fprov)\log \frac{1}{\delta}}{\eps^2 (1 - p)^2})$ of the $N_j$ samples to train prover $f_j$ via \Cref{assumption:a2-rt}. Of the unused samples, apply the conjecturer to map $x_j^{(i)}$ to $y_j^{(i)} \sim C_w(x_j^{(i)})$. Discard $y_j^{(i)}$ if  $\Succ(y_j^{(i)}, f_j) <  p/2$, else let $x_{j + 1}^{(i)} := y_j^{(i)}$. Let $N_{j + 1}$ denote the number of remaining samples. Note that if prover $f_j$ is successfully trained according to \Cref{assumption:a2-rt} to where $\Pro_{x \sim D_j}[\Succ(x,f_j) \geq p] \geq 1 - \eps$, then by \Cref{assumption:a3-ood}, $\Pro_{x \sim D_j, x' \sim C_w(x)}[\Succ(x',f_j) \geq p/2] \geq 1 - \eps$, so that by standard concentration inequalities, $N_{j + 1} \gtrapprox (1 - \eps)N_j$. In the worst case, we have that $\forall j \in [t], N_j \approx N_0 (1 - \eps)^j$, so that the number of oracle calls to $D_0$ required to sample from $D_j$ grows as $(\frac{1}{1 - \eps})^j$.



\subsection{Expansion Guarantee for \Cref{alg:main}.}\label{sec:4-3-thm}



\Cref{thm:main} shows that on an infinite graph, with a reversible conjecturer (\Cref{defn:conjecturer}) that induces a random walk that satisfies the Strong Isoperimetric Inequality (\Cref{assumption:graph}), and for training error threshold $\eps < \frac{1}{2\kappa^2}$, \Cref{alg:main} generates a sequence of provers $\{ f_t\}_{t \in \N}$ such that $\forall t \in \N,$ the size of the $t$-th prover's knowledge set $| \{ x \in V : \Succ(x, f_t) \geq p\} |$ is exponential in $t$. 

\begin{restatable}[Analysis of \Cref{alg:main}]{theorem}{expansionthm}\label{thm:main}

Let $G = (V,E)$ be a theorem graph with $|V| = \infty$ and $(C_w,P_w, w)$ be a reversible conjecturer as in \Cref{defn:conjecturer}. Suppose $(V, P_w)$ satisfies Assumption \ref{assumption:graph} with Isoperimetric constant $\kappa > 0$. Let $\pi : V \to (0, \infty)$ be an invariant measure of $(V,P_w)$. Suppose Assumptions \ref{assumption:a1-js}-\ref{assumption:a4-nbr} hold, with universal constant $p \in (0,1)$ as in \Cref{assumption:a2-rt}, and that $\pi$  is chosen so that $\sup_{z \in V} \pi(z) < \infty$ and $0 < ||D_0||_{\ell^2(1/\pi)}^2 < \infty$. For any $\delta , \eps \in (0,1)$,  \Cref{alg:main} produces a sequence of provers $\{ f_t\}_{t \in \N}$ where $\forall t \in \N$, with probability at least $1 - (t + 1)\, \delta$,

\begin{align}
    | \{ x \in V : \Succ(x, f_t) \geq p\} | \geq \frac{(1 - \eps)^2}{(\sup_{z \in V} \pi(z)) \cdot ||D_0||_{\ell^2(1/\pi)}^2 } \cdot (\frac{1 - \eps}{1 - \frac{1}{2\kappa^2}})^{2t}
\end{align}

\end{restatable}

To prove \Cref{thm:main}, we show that for every $t \in \N$, $||D_{t + 1}||_{\ell^2(1/\pi)}^2 \leq (\frac{||P_w^\top||_{\ell^2(1/\pi) \to \ell^2(1/\pi)}}{1 - \eps})^2\cdot$ $||D_{t}||^2_{\ell^2(1/\pi)}$, where $||P_w^\top||_{\ell^2(1/\pi) \to \ell^2(1/\pi)}$ is the operator norm of $P_w^\top$ in the Hilbert space $\ell^2(1/\pi)$. By \Cref{lem:rw}, since $(V,P_w)$ satisfies \Cref{assumption:graph}, then $||P_w^\top||_{\ell^2(1/\pi) \to \ell^2(1/\pi)}\leq 1 - \frac{1}{2\kappa^2}$. Finally, \Cref{thm:main} follows from \Cref{lem:ent-to-cov-num}, which lower bounds $| \{ x \in V : \Succ(x, f_t) \geq p\} |$ using $\frac{1}{||D_{t}||^2_{\ell^2(1/\pi)}}$. The full proof is in \Cref{appen:thm-srw}. 

Thus, if the simple random walk $\Psrw$ over the theorem graph satisfies $\ISineq$ (i.e. for any finite $A \subset V, \frac{|\{ (e^-, e^+) \in E \;:\; e^- \in A, \;e^+ \notin A\}|}{\sum_{x \in A} \deg(x)} \geq \frac{1}{\kappa}$), then \Cref{alg:main} with a conjecturer, whose transition matrix is $ \Psrw$, generates a sequence of provers $\{ f_t\}_{t \in \N}$ whose knowledge set grows exponentially.


%% file: sections/arxiv/exploration-q2-arxiv.tex

\section{Diversity Measure and an Improved Conjecturing Algorithm}\label{sec:diversity-q2}

\subsection{Diversity Measure}\label{subsec:div-measure}



Unlike \Cref{sec:framework-q1}, suppose for simplicity the theorem graph has a finite but large number of nodes and $||\pi||_1 = 1$. \footnote{We believe the following analysis can be extended to infinite graphs.} The following ideas can be instantiated with any $\tau \in \N$, but practically $\tau$ should be  a small non-zero natural number. The main object of interest in this section is the $\tau$-expanded knowledge set of a prover $f$. 

\begin{definition} With universal $p \in (0,1)$ in \Cref{assumption:a2-rt},  for any $\tau \in \N$, $f \in \Fprov$, let $\Kct(f) := B(\{ x \in V : \Succ(x,f) \geq p\}, \tau)$ $=\{ x' \in V : \exists x \in V : \Succ(x,f) \geq p \textup{ and } \dist(x,x') \leq \tau\}$. We call $\Kp(f) := K_0^p(f)$ the knowledge set of $f$ and $\Kct(f)$ the $\tau$-expanded knowledge set.
\end{definition}

\textit{We define an operational notion of diversity of $f$'s knowledge as the size of $\Kct(f)$}. $\Kct(f)$ can be interpreted as the set of theorems that $f$ can prove by performing $\tau$ steps of ``reasoning" at inference time to reduce the input theorem into one it knows how to prove. \Cref{fig:diversity} depicts two provers $f_1$, $f_2$ whose knowledge sets are the same size $|\Kp(f_1)| = |\Kp(f_2)|$, but where their $\tau$-step expansions are of different sizes: $|\Kct(f_1)| \ll |\Kct(f_2)|$. Thus, $f_2$'s knowledge is more diverse than $f_1$'s. Identifying diversity with $|\Kct(f)|$ reflects the intuition that knowing a diverse set of proof ideas is useful to a prover $f$ as it increases the set of theorems that $f$ can  prove by combining known proof ideas. 




Let $(C_w, P_w, w)$ be a baseline conjecturer satisfying \Cref{defn:conjecturer} and for convenience denote $P:= P_w$. Our goal is to design an improved conjecturer that generates a sequence of training distributions $\{ D_t\}_{t \in \N}$ where at each $t \in \N$ prover $f_t$  is trained on $D_t$, and where $|\Kct(f_t)|$ increases quickly with $t$. In fact, \Cref{lem:divt-kct-connection} (proved in \Cref{appen:aux}) shows that for each $t\in \N$ it is reasonable to maximize $\Divt(D_t)$ (\Cref{defn:diversity-measure}), a diversity measure of the training distribution $D_t$.

\begin{restatable}[Diversity Measure]{definition}{divmeasure}\label{defn:diversity-measure}
Let $G = (V,E)$ be a theorem graph with  $P$ being the transition matrix of any reversible, nearest-neighbor-type random walk with stationary measure $\pi$. For any parameter $\tau \in \N$, and any $\mu \in \Delta(V)$, define:

\begin{align}
    \Divt(\mu) := \frac{1}{||\mu^\top P^\tau||^2_{\ell^2(1/\pi)}} = \frac{1}{\sum_{x \in V }(\mu^\top P^\tau e_x)^2 \frac{1}{\pi(x)} }
\end{align}
\end{restatable}


\begin{restatable}[Large $\Divt$ implies large $\Kct$]{lemma}{divtkct} \label{lem:divt-kct-connection}
Let $p \in [0,1]$ be a universal constant in \Cref{assumption:a2-rt}. Instantiate $\Divt : \Delta(V) \to \R$ with the transition matrix $P$ of any reversible, nearest-neighbor-type random walk. For any $\tau \in \N$, any prover $f \in \Fprov$, and any theorem distribution $D \in \Delta(V)$, the $\tau$-expanded knowledge set of $f$ satisfies:

\begin{align}
    |\Kct(f)| \geq \frac{(\Pro_{x \sim D}[\Succ(x,f) \geq p])^2}{\sup_{x \in V} \pi(x)} \Divt(D)
\end{align}

\end{restatable}

If prover $f$ is successfully trained on theorem distribution $D$ via \Cref{assumption:a2-rt}, then $\Pro_{x \sim D}[\Succ(x,f) \geq p] \geq 1 - \eps$. By \Cref{lem:divt-kct-connection}, $|K_\tau^p(f)| \geq \frac{(1 - \eps)^2}{\sup_{x \in V}\pi(x)} \Divt(D)$, so maximizing $\Divt(D)$ of the training distribution is a reasonable proxy for maximizing $|K_\tau^p(f)|$. 


We note that in the setting of \Cref{sec:framework-q1}, where \Cref{assumption:graph} holds for random walk $(V,P)$, then maximizing  $|B(\Kp(f), \tau)|$ reduces to maximizing $|\Kp(f)|$. When $P = \Psrw$, then \Cref{assumption:graph} implies for bounded-degree graphs that there is a reasonable constant $c$ where for any finite $A \subset V,  |B(A, \tau)| \geq \exp(c\,\tau) |A|$. That is, any finite subset of theorems is diverse. More generally, if \Cref{assumption:graph} holds for $P$ (not necessarily $\Psrw$), then for $\{ D_t\}_{t \in \N}$ generated as $D_{t + 1}^\top = D_t^\top P$,  we have $\Divt(D_t) = \frac{1}{||D_{t + \tau}||^2_{\ell^2(1/\pi)}}$, where $\pi^\top P = \pi^\top$. Because the proof of \Cref{thm:main} implies fast exponential decay of  $||D_{t}||^2_{\ell^2(1/\pi)}$ as a function of $t \in \N$ for every $t$, it follows that $\Divt(D_t)$ grows exponentially for each $t$. In short, when \Cref{assumption:graph} holds, then the two objectives $|\Kp(f)|$ and $|\Kct(f)|$ are aligned. As such, the more interesting setting for studying diversity is when \Cref{assumption:graph} does not hold for the baseline random walk $P$. When $P=\Psrw$, some intuitive graphs where \Cref{assumption:graph} does not hold are graphs consisting of multiple poorly-interconnected clusters, such as that depicted in \Cref{fig:diversity} or \Cref{fig:kcluster}. For a graph like in \Cref{fig:diversity}, the choice of which nodes to include in $\Kp(f)$ will greatly impact $\Kct(f)$, making the problem of maximizing both $\Kp(f)$ and $\Kct(f)$ non-trivial.

\subsection{Improved Conjecturing Algorithm}

\Cref{defn:conj-alg-improved} is a conjecturer that generates a sequence of training distributions $\{ D_t\}_{t \in \N}$, with the aim to maximize $\Divt(D_t)$ at each $t \in \N$, for a chosen hyper-parameter $\tau \in \N$. This maximization cannot easily be done in closed form, as the conjecturer is restricted to generate $D_{t + 1}$ from $D_t$ by one step of a nearest-neighbor-type random walk.  Instead, fixing $\tau \in \N$, \Cref{defn:conj-alg-improved}  defines a push-forward map $C : V \to \Delta(V)$ where for any $x \in V$, $C(x) \in \Delta(N(x))$ is supported on neighbors $N(x)$ of $x$. The weighting of $C(x)$ is determined by maximizing $\Divt(C(x))$.  

\begin{align}
    \forall x \in V, \, C(x) &\leftarrow \arg\max_{\mu \in \Delta(N(x))} \Divt(\mu) =  \arg\min_{\mu \in \Delta(N(x))} \mu^\top P^\tau \diag(1/\pi) (P^\tau)^\top \mu \label{eq:optim-problem-diversity}
\end{align}

Where $\forall z \in V, \diag(1/\pi)[z,z] = \frac{1}{\pi(z)}$. This optimization problem can be solved given the matrix, $M(x, \tau) \in \R^{|N(x)| \times |N(x)|}$ where for any $y,y' \in N(x)$, the $(y,y')$-th entry is: $M(x, \tau) [y,y'] := e_y^\top P^\tau \diag(1/\pi)  (P^\tau)^\top  e_{y'}$. We can characterize $M(x, \tau)$ in terms of diffusion embeddings of the neighbors of $x$, where a diffusion embedding embeds any node $z \in V$ to a vector $p^\tau(z)$ representing the probability distribution of a $\tau$-step random walk started at $z$ according to transition matrix $P$.

\begin{definition}\label{defn:diffusionembd} 
    Let $G = (V,E)$ be a connected, undirected graph, with $P$ being the transition matrix of a reversible random walk over $V$, with stationary measure $\pi$. For all $\tau \in \N, \forall x \in V$, define $p^\tau(x)$ as the scale-$\tau$ diffusion embedding of $x$ where: 

    \begin{align}
        p^\tau (x) &:= (P^\tau)^\top e_x \in [0,1]^{|V|}\\
        \forall z \in V, \, p^\tau (x,z) &:= e_z^\top(P^\tau)^\top e_x \in [0,1]
    \end{align}
    For every $x, y \in V$, define the scale-$\tau$ diffusion similarity as 

    \begin{align}
    \diffsim(x,y) := \langle p^{\tau}(x), p^{\tau}(y) \rangle_{\ell^2(1/\pi)} := \sum_{z \in V} p^{\tau}(x, z) \, p^{\tau}(y,z) \, \frac{1}{\pi(z)} 
    \end{align}
    
     and the scale-$\tau$ diffusion distance \cite{diffusionmaps, pku} as 

     \begin{align}
         \diffdist(x,y) := || p^{\tau}(x) - p^{\tau}(y) ||_{\ell^2(1/\pi)} := \sqrt{\sum_{z \in V} (p^{\tau}(x,z) - p^{\tau}(y,z) )^2 \frac{1}{\pi(z)}}
     \end{align}
\end{definition}

For any $y,y' \in N(x)$, the $(y,y')$-th entry of $M(x, \tau)$ is $\diffsim (y,y')$. Note that the diffusion similarity requires computing $\frac{1}{\pi(z)}$ for relevant $z \in V$, but the stationary measure $\pi(z)$ need only be computed up to a normalizing constant for  the optimization problem in \Cref{eq:optim-problem-diversity}. For a conjecturer $(C,P,w )$ satisfying \Cref{defn:conjecturer}, $\forall z \in V, \pi(z) \propto w(z)$. Let $M'(x, \tau) = P^\tau \diag(1/w) (P^\tau)^\top$ and  $\forall z \in V$, $\diag(1/w)[z,z] = \frac{1}{w(z)}$. With regularization strength $\eta \in [0,0.5]$, $\forall x \in V$, \Cref{defn:conj-alg-improved} computes:

\begin{align}
    C(x) &\leftarrow \arg\max_{\substack{
    \mu \in \Delta(N(x))\\
    \mu(y) \geq \eta/\deg(x),\;\forall y \in N(x)
}} \Divt(\mu) =  \arg\min_{\substack{
    \mu \in \Delta(N(x))\\
    \mu(y) \geq \eta/\deg(x),\;\forall y \in N(x)
}} \mu^\top M'(x, \tau) \mu \label{eq:diff-sim-optim}
\end{align}

This mapping $C : V \to \Delta(V)$ defined in \Cref{defn:conj-alg-improved} can be substituted in for $C_w$ in line 4 of \Cref{alg:pass-rate-sampler}, called in line 4 of the prover-conjecturer system (\Cref{alg:main}).  If $G$ is finite, connected, undirected, and non-bipartite, the regularization with $\eta > 0$ ensures  (via \Cref{lem:ensure-mixing-for-alg4}) the random walk induced by $C$ converges to a stationary measure, eventually exploring the entire graph. Note that \Cref{defn:conj-alg-improved} uses \Cref{alg:dp-pT} to compute diffusion embeddings.


\begin{algorithm}[t]
\caption{Improved Conjecturer}
\label{defn:conj-alg-improved}
\begin{algorithmic}[1]
\Require $\tau \in \N$, $\eta \in [0,0.5]$, reversible conjecturer $(C_w,P_w,w)$ satisfying \Cref{defn:conjecturer}, and i.i.d. sample access to $D \in \Delta(V)$
\Ensure i.i.d. sample access to $D'=C\#D \in \Delta(V)$

\Procedure{$C$}{$x$}
    \State Compute $p^\tau(u)$ for all $u \in N(x)$ using \Cref{alg:dp-pT} with $(C_w,P_w,w)$
    \State Let $M'(x,\tau) \in \R^{|N(x)|\times |N(x)|}$, with
     $(y,y')$-entry $\sum_{z} p^\tau(y,z)p^\tau(y',z) \frac{1}{w(z)}$,  $\forall y,y' \in N(x)$
    \State $\displaystyle
    \mu_x \gets
    \arg\min_{\substack{\mu \in \Delta([|N(x)|])\\
    \mu(i) \ge \eta/\deg(x)\; \forall i \in [|N(x)|]}}
    \mu^\top M'(x,\tau)\mu$
    \State Sample $i \sim \mu_x$ and \Return the $i$th neighbor of $x$
\EndProcedure

\Sampler{$D' = C \# D$} 
    \State Draw $x \sim D$, $x' \sim C(x)$, and \Return $x'$
\EndSampler
\end{algorithmic}
\end{algorithm}

\begin{algorithm}[t]
\caption{Dynamic Program for $p^\tau(x)$}
\label{alg:dp-pT}
\begin{algorithmic}[1]
\Require $x\in V$, $\tau\in\N$, neighbor oracle $N:V\to2^V$, and reversible conjecturer $(C_w,P_w,w)$
\Ensure $p^\tau(x)$

\State $B_{-1}\gets\emptyset$, $B_0\gets\{x\}$

\For{$s=1,\ldots,\tau$}
    \State $\partial B_{s-1}\gets B_{s-1}\setminus B_{s-2}$
    \State $B_s\gets B_{s-1}\cup \bigcup_{v\in\partial B_{s-1}}N(v)$ 
\EndFor

\State $S\gets B_\tau$; initialize $p^0(x,y)\gets\mathbf{1}\{y=x\}$ for $y\in S$
\For{$s=1,\ldots,\tau$}
    \State $p^s(x,y)\gets \sum_{z\in S}p^{s-1}(x,z)P_w[z,y]$ for all $y\in S$
\EndFor
\State \Return $p^\tau(x)$
\end{algorithmic}
\end{algorithm}

As a concrete guarantee, \Cref{prop:e2e-guarantee-alg2} shows that \Cref{defn:conj-alg-improved} provably explores the theorem graph $G$ faster than a simple random walk when $G$ is in a special class of clustered graphs. 


\cliquegraphmaintxt

In the regime where the clique size $M$ is much larger than the inter-cluster connectivity ($M \gg r^2$), \Cref{defn:conj-alg-improved} improves over the baseline simple-random-walk-based conjecturing algorithm in that: both conjecturers define a random walk over the meta-graph, but that of \Cref{defn:conj-alg-improved} has a larger spectral gap \cite{peres}. The source of the improvement is that \Cref{defn:conj-alg-improved} up-weights the probability of making inter-cluster transitions, compared to the simple random walk. This becomes useful when the cluster size $M$ is large, in which case the simple random walk can become stuck in a cluster for a large number of steps.

\cliquepropmaintxt

\Cref{prop:e2e-guarantee-alg2} shows that the diversity of theorem distributions generated under \Cref{defn:conj-alg-improved} (resp. a simple-random-walk-based conjecturer) can be guaranteed to exceed $ 1 - \eps$ in $t_{\textup{improved}}$ steps (resp. $t_{\textup{baseline}}$ steps), and where $t_{\textup{improved}} \leq O(\frac{r^2}{M}) \,t_{\textup{baseline}}$. Although this does not prove a matching lower-bound to show separation between \Cref{defn:conj-alg-improved} and the baseline, we believe the bounds used to derive these results are tight so that in the worst case over all possible meta-graphs instantiating \Cref{defn:clique-graph}, \Cref{defn:conj-alg-improved} is $\Theta(\frac{M}{r^2})$ times more efficient than a simple-random-walk-based conjecturer, which is substantial when the clusters are very dense ($M \gg r^2$).

The proof of \Cref{prop:e2e-guarantee-alg2} is in \Cref{appen:e2e-result}. We think it is plausible that \Cref{defn:conj-alg-improved} improves over a baseline simple random walk $P = \Psrw$ for a broader class of clustered graphs (graphs like those depicted in \Cref{fig:kcluster}) than \Cref{defn:clique-graph}, but we do not have a formal guarantee. Our intuition of \Cref{defn:conj-alg-improved} on clustered graphs is that it defines a random walk that, at node $x$, has a higher chance to transition to neighbors of $x$ in a different cluster than $x$. This lets the random walk escape from its starting cluster faster and should enable it to explore the entire theorem graph faster.


In general, $\tau$ should not be set too small or too big. If $\tau =0$, the diffusion embedding is not informative since every node $x \in V$ is embedded as $e_x \in \R^{|V|}$. If $\tau$ is larger than the mixing time of $P$, then $\forall x \in V,(P^\top )^\tau e_x \approx \pi$, so the diffusion embedding is also not informative. When the theorem graph is a clustered graph consisting of equally sized clusters, 
there should exist a natural setting for $\tau$: the number of steps it takes a random walk to ``forget" its exact starting node, where it still ``remembers" its starting cluster. More precisely, the ideal property of $\tau$ on a clustered graph is that for two nodes $x, y$ in the same cluster, their $\tau$-step random walk distributions are approximately the same: $p^\tau(x) \approx p^\tau(y)$, but for $x, y$ in different clusters, their $p^\tau(x)$ and $p^\tau(y)$ differ greatly. Then, $p^\tau(x)$ serves as an identifier for which cluster $x$ is in.

\Cref{defn:conj-alg-improved} accepts as input $\tau \in \N$ and a baseline reversible conjecturer $(C_w, P_w, w)$, satisfying \Cref{defn:conjecturer}. It returns an improved conjecturer. \Cref{defn:conj-alg-improved}  improves the baseline conjecturer by performing a $\tau$-step look-ahead over the graph to maximize diversity. In principle, this algorithmic scaffold can be applied to conjecturers trained empirically, if they satisfy \Cref{defn:conjecturer}. The random walk defined by \Cref{defn:conj-alg-improved} may not be reversible, so the algorithmic scaffold cannot be recursively applied to its own output.



\subsection{Computing Diffusion Similarity via Contrastive Learning}

\Cref{thm:contembd} shows the diffusion similarity (used in step 4 of \Cref{defn:conj-alg-improved}) can be computed via learned contrastive embeddings. A self-contained summary of contrastive learning is in \Cref{appen:cont-learning-intro}.

\begin{restatable}[]{theorem}{contembd}\label{thm:contembd}
    Let $G = (V,E)$ be a finite, undirected, connected theorem graph. Let reversible conjecturer $(C, P, w) := (C_w, P_w, w)$ satisfy \Cref{defn:conjecturer}, and let $\pi$ be a stationary distribution for $P$. Let $\Femb \subset \F_{V}^{(|V|)}$ be a class of functions from $V$ to $\R^{|V|}$. For any $\tau \in \N$, define the scale-$\tau$ contrastive loss of embedding model $h_\theta \in \Femb$  as: 

    \begin{align}
        \Lc^\tau(h_\theta) := \ex_{(x, x^-) \sim D_-^\tau} (h_\theta(x)^\top h_\theta(x^-))^2 - 2 \cdot \ex_{(x, x^+) \sim D_+^\tau}h_\theta(x)^\top h_\theta(x^+)
    \end{align}

    where $D^\tau_+, D^\tau_- \in \Delta(V \times V)$ are distributions of positive and negative pairs: $(a,b) \sim D^\tau_+$ where $z \sim \pi, a \sim C^\tau (z), b \sim C^\tau (z)$ ($C^\tau$ denotes applying $C$ $\tau$-times consecutively) and $(a,b) \sim D^\tau_- \textup{ where } a \sim \pi, b \sim \pi$. Fix any $\tau \in \N$. Assume $\Femb$ contains a global minimizer of $\Lc^\tau$ over $\F_{V}^{(|V|)}$. For any global minimizer $\Femb \ni h_{\theta_*} : V \to \R^{|V|}$, then $\forall u,v \in V, \langle p^\tau(u),  p^\tau(v)\rangle_{\ell^2(1/\pi)} = h_{\theta_*}(u)^\top h_{\theta_*}(v)$.
\end{restatable}

The proof of \Cref{thm:contembd} is in \Cref{appen:contembd}. \Cref{thm:contembd} shows that the diffusion similarity of any two nodes on the graph (a quantity computed in step 3  of \Cref{defn:conj-alg-improved}) can instead be computed from the Euclidean inner product of their learned embeddings. Given $n$ samples $\{ x_i\}_{i \in [n]} \subset V$ drawn i.i.d. from the stationary measure $\pi \in \Delta(V)$ of $P$, an embedding model $\Femb \ni h_{\hat{\theta}} : V \to \R^{|V|}$ can be learned by computing the empirical risk minimizer (ERM, \Cref{defn:erm-cont}) of the empirical contrastive loss $\Lchat$ in \Cref{eq:emp-cont-loss-defn}. Under additional assumptions, we show in \Cref{thm:e2e-contrastive-learning-result} an end-to-end result where ERM, in a finite-sample regime with small embedding dimension, yields embeddings that can be used to approximate the diffusion similarity. 

\etethmmaintext

\Cref{thm:contembd} and \Cref{thm:e2e-contrastive-learning-result} motivate \Cref{defn:conj-alg-improved-cont-learning}, a conjecturing algorithm that is analogous to \Cref{defn:conj-alg-improved}, but with \Cref{alg:dp-pT} replaced by a contrastive-learning procedure.

\begin{algorithm}[t]
\caption{Improved Conjecturer with Contrastive Learned Diffusion Embeddings}
\label{defn:conj-alg-improved-cont-learning}
\begin{algorithmic}[1]
\Require $\tau \in \N$, $\eta \in [0,0.5]$, $\eps, \delta \in (0,1)$, reversible conjecturer $(C_w,P_w,w)$ satisfying \Cref{defn:conjecturer}, and i.i.d. sample access to $D \in \Delta(V)$ 
\Ensure i.i.d. sample access to $D'=C\#D$

\Procedure{$C$}{$x$}
    \State Learn an embedding function $h_\theta \in \Femb$ by minimizing the empirical contrastive loss (e.g. via \Cref{defn:contrastive-learning}), with base distribution $\pi$ and augmentation function $(P_w)^\tau$, to error $\eps$ and failure probability $\delta$.
    \State Let $M'(x,\tau) \in \R^{|N(x)|\times |N(x)|}$, with
     $(y,y')$-entry $ h_\theta(y)^\top h_\theta(y')$,  $\forall y,y' \in N(x)$
    \State $\displaystyle
    \mu_x \gets
    \arg\min_{\substack{\mu \in \Delta([|N(x)|])\\
    \mu(i) \ge \eta/\deg(x)\; \forall i \in [|N(x)|]}}
    \mu^\top M'(x,\tau)\mu$
    \State Sample $i \sim \mu_x$ and \Return the $i$th neighbor of $x$
\EndProcedure

\Sampler{$D' = C\#D$}
    \State Draw $x \sim D$, $x' \sim C(x)$, and \Return $x'$
\EndSampler
\end{algorithmic}
\end{algorithm}

Finally, we discuss how to instantiate \Cref{thm:contembd} and \Cref{thm:e2e-contrastive-learning-result} for a \textit{local} subgraph of $G$ around the theorems of interest, rather than needing to take $G$ to be the entire theorem graph, which is impractical. For a theorem graph $G = (V,E)$, define the following notion of a subgraph of $G$ supported on the $T$-neighborhood of a node $x \in V$.


\localsubgmaintxt

A simple observation is that: for any $x \in V$, on a $T = \tau + 1$ local subgraph around $x$, the $\tau$-step random walk distribution starting from any neighbor $z \in N(x)$, $p^\tau_{G_{x,T}}(z)$ is the same as $p^\tau(z)$. In addition, for any reversible Markov chain $(V,P)$ on the graph $G$ with stationary measure $\pi$, for any $x \in V$ and $T \in \N$, the stationary measure $\pi_{x,T}$ is equal to $\pi$ restricted to $B(x,T) \subset V$, up to a multiplicative constant, by the self-loop construction of $G_{x,T}$. Thus, the diffusion similarity computed on the local subgraph $G_{x,T}$ differs from that computed on $G$ by a multiplicative constant, which means the locally-computed diffusion similarity is sufficient for solving the optimization problem in \Cref{eq:optim-problem-diversity}.

To learn the spectral diffusion embedding for reversible Markov chain $(V,P)$, of reversible conjecturer $(C,P,w) := (C_w,P_w,w)$ in a local $T$-neighborhood around a node $x \in V$ via \Cref{defn:contrastive-learning}, one will require access to the following oracles, defined in \Cref{defn:T-neighborhood}.

\begin{enumerate}
    \item The push-forward map $C_{x,T} : B(x,T) \to \Delta(B(x,T))$ representing the restricted Markov chain $(V_{x,T}, P_{x,T})$, where $\forall y \in B(x,T)$, $C_{x,T}(y) = P_{x,T}[y] \in \Delta(B(x,T))$,
    \item i.i.d. sample access to $\pi_{x,T} \in \Delta(B(x,T))$.
\end{enumerate}

The construction of these oracles, given $x \in V$, $T \in \N$, and neighbor oracle $N : V \to 2^V$ of $G$, is conceptually straightforward. We describe the details of their constructions in \Cref{alg:compute-pix} of \Cref{appen:local-subgraphs}.

%% file: sections/conclusion.tex
\section{Conclusion and Limitations}\label{sec:discussion}

We formalized the space of theorems as a graph and a conjecturer as a nearest-neighbor-type Markov chain over the theorem graph. \Cref{thm:main} shows that under \Cref{assumption:graph}, the knowledge of the prover trained via \Cref{alg:main} grows exponentially with the number of iterations. For the case where \Cref{assumption:graph} does not hold,  we propose a diversity measure (\Cref{defn:diversity-measure}) for a training distribution of theorems generated by a conjecturer and an improved conjecturing algorithm (\Cref{defn:conj-alg-improved}) that locally maximizes this diversity measure, by computing the diffusion similarity (\Cref{defn:diffusionembd}) between neighboring theorems in the theorem graph. Finally, we describe a method to compute the diffusion similarity using contrastive learning (\Cref{thm:contembd} and \Cref{thm:e2e-contrastive-learning-result}).

One limitation of the current framework is that the conjecturer is not adaptive toward the particular prover $f_t$ at iteration $t$, but rather adaptive toward the prover class, which determines the edge set $E$. It defines a time-homogeneous nearest-neighbor random walk, whereas a time-inhomogeneous random walk would be more realistic. Another intuition that is not captured by the current framework is that of theorem compositionality. In particular, one can imagine a conjecturer that has knowledge of two theorems $X \to Y$ and $Y \to Z$ may be able to conjecture $X \to Z$, but this is not easily compatible with the view that a conjecturer maps a single theorem to a distribution of theorems. More broadly, an interesting question is how to extend our analysis to a hypergraph of theorems or dependent types \cite{freedman}. Finally, we note that although this theoretical framework is instantiated for the setting of formal theorem proving, it could potentially serve as a framework for understanding automated synthetic data generation for other verifiable domains besides formal theorem proving.

%% file: sections/acknowledgements.tex
\section*{Acknowledgements}\label{sec:acknowledgements}

TC acknowledges funding from an NSF Graduate Research Fellowship. This work was supported in part by an OpenAI Superalignment Fast Grant and by DARPA AIQ under Agreement No. HR00112520023.

We are deeply grateful to Tengyu Ma for many insightful discussions. We also thank Kaiyue Wen for discussions.

%% file: sections/arxiv/appendix-arxiv.tex
\tableofcontents

\input{sections/appendix/q1-main.tex}

\input{sections/appendix/diffusion-and-cont-learning}

\input{sections/appendix/e2e-result}

\input{sections/appendix/q1-detailed-asmpt}

\input{sections/appendix/aux}
\input{sections/appendix/K-cluster}

%% file: sections/appendix/q1-main.tex
\section{Proof of \Cref{thm:main}}\label{appen:thm-srw}

The purpose of this section is to eventually prove \Cref{thm:main} in \Cref{appen:pf-of-thm-42}.

\subsection{Preliminaries.}\label{appen:more-detailed-prelim}

\paragraph{Isoperimetric Inequality.} Let $G = (V, E)$ be an undirected graph with a countable set of vertices, $V$, and edges $E$. $G$ is connected, and is locally finite (every vertex has finite degree). Let $P$ be a Markov chain transition matrix over $V$ (We will refer to the pair $(V,P)$ as a Markov chain). The simple random walk is a Markov chain where every node $x$ transitions uniformly to one of its neighbors in the graph, $N(x) := \{ x' \in V : (x,x') \in E\}$. In the following, we define some basic properties of Markov chains over graphs. 

For $n \in \N$, let $p^{(n)}(x,y)$ denote the probability the Markov chain $(V,P)$ initialized at $x$ is at $y$ at the $n$-th step.

\begin{definition}\label{defn:irreducible} (Irreducibility)
    Markov chain $(V, P)$ is irreducible if $\forall x,y \in V$, there is some $n \in \mathbb{N}$ such that $p^{(n)}(x,y) > 0$.
\end{definition}

For the remaining definitions, we will assume the Markov chain is irreducible. We define aperiodicity as follows.

\begin{definition}\label{defn:aperiodic} (Aperiodicity)
    A Markov chain $(V,P)$ is aperiodic if for any $x \in V$, $\textup{gcd}\{ n \in \N : p^{(n)}(x,x) > 0\} = 1$.
\end{definition}

Note that if a Markov chain is irreducible and for some $x \in V$, $\textup{gcd}\{ n \in \N : p^{(n)}(x,x) > 0\} = 1$, then for any $x' \in V$, $\textup{gcd}\{ n \in \N : p^{(n)}(x',x') > 0\} = 1$. 

A stationary measure $\pi : V \to (0,\infty)$ where $\pi^\top = \pi^\top P$. If $||\pi||_1 = 1$, then it is a stationary distribution. The following Theorem shows that irreducibility and aperiodicity imply mixing for finite Markov chains.

\begin{lemma}\label{lem:mixing-asymptotic} (Theorem 4.9 of \cite{peres})
    Suppose that Markov chain $(V,P)$ is finite ($|V| < \infty$), irreducible and aperiodic with stationary distribution $\pi$. Then there exist constants $C > 0$ and $\alpha \in (0,1)$ where

    \begin{align}
        \max_{x \in V} || p^{(n)}(x, \cdot) - \pi||_{\textup{TV}} \leq C \alpha^n
    \end{align}
\end{lemma}

The following definitions introduce the notion of reversibility and conductance.

\begin{definition} (Reversibility)
    Let $V$ be an infinite, countable set. Markov chain $(V, P)$ is reversible if there exists a (stationary) measure $\pi : V \to (0, \infty)$ such that
    \begin{align}
        \pi(x) P[x,y] = \pi(y) P[y, x]\quad \forall x, y \in V
    \end{align}
\end{definition}

\begin{definition} (Conductance)
    Let $a(x,y) = \pi(x) P[x,y] = a(y,x)$ be the conductance between $x$ and $y$. $\pi(x)$ is the total conductance at $x$. For a Markov chain transition matrix $P$, denote its edge set $E(P)$ where:
    \begin{align}
        (x,y) \in E(P) \iff a(x,y) > 0
    \end{align}
    Note, a nearest-neighbor-type random walk is a Markov chain  $(V,P)$ where $E(P) = E$ (the edge-set of the graph $G$).
\end{definition}

Now, we define the Isoperimetric Inequality for a Markov chain over a graph, which will generalize \Cref{defn:IS}. Let $(V,P)$ be a reversible, irreducible Markov chain with total conductance $\pi(\cdot)$, edge conductance $a(\cdot, \cdot)$. Define the following quantities.

\begin{align}
    \pi(A) &:= \sum_{x \in A} \pi(x), \quad \forall A \subset V\\
    a(D) &:= \sum_{e \in D} a(e^-, e^+), \quad \forall D \subset E(P)\\
    \partial A &:= \textup{set of all edges in $E(P)$ with}\\
    &\textup{one endpoint in $A$ and the other in $V - A$}
\end{align}

For the definition of $\partial A$, note that the graphs we study in this work are all undirected. We now define the Isoperimetric Inequality.

\begin{definition}[Isoperimetric Inequality]\label{defn:IS-general}
Let $\Fscr : \R_+ \to \R_+$. A Markov chain $(V,P)$ satisfies the $\Fscr$-Isoperimetric Inequality, if there exists a constant $\kappa > 0$ such that for every finite $A \subset V$, $\Fscr(\pi(A)) \leq \kappa \, a(\partial A)$. 
\end{definition}

The particular function $\Fscr : \R_+ \to \R_+$ we care about in this work is $\Fscr (t) = t$, corresponding to the Strong Isoperimetric Inequality defined in \Cref{defn:IS}. However, for the proofs of some of the following Lemmas, it is useful to define, for $1 \leq d \leq \infty$, the $d$-dimensional Isoperimetric Inequality $\ISineq_d$ as \Cref{defn:IS-general} with $\Fscr (t) = t^{1 - 1/d}$. The strong Isoperimetric Inequality $\ISineq$ can be thought of as $\ISineq_d$ with $d = \infty$.

\paragraph{Hilbert Spaces and Norms.} With $(V,P)$ being an irreducible,  reversible Markov chain with stationary measure $\pi : V \to (0,\infty)$, let $\ell_0(V)$ denote the linear space of finitely supported functions on $V$ (i.e. functions that are non-zero only on a finite number of elements of $V$). For an (undirected) edge $e = (e^-, e^+) \in E$,  define the resistance as $r(e) := \frac{1}{a(e^-, e^+)}$. The following defines Hilbert spaces of functions, relevant to \Cref{thm:main}.

\begin{definition}
    Define the Hilbert space $\ell^2(V, \pi)$ with inner product
    \begin{align}
        \langle f, g\rangle_{\ell^2(V, \pi)} = \sum_{v \in V} f(v) g(v) \pi(v)
    \end{align}
    $\ell^2(V, \pi)$ is the set of functions $f : V \to \mathbb{R}$ with $\sum_{v \in V} |f(v)|^2 \pi(v) < \infty$.

    Define the Hilbert space $\ell^2(V, 1/\pi)$ with inner product
    \begin{align}
        \langle f, g\rangle_{\ell^2(V, 1/\pi)} = \sum_{v \in V} f(v) g(v) \frac{1}{\pi(v)}
    \end{align}
    $\ell^2(V, 1/\pi)$ is the set of functions $f : V \to \mathbb{R}$ with $\sum_{v \in V} |f(v)|^2 \frac{1}{\pi(v)} < \infty$. When clear from context, we will write $\ell^2(\pi)$ and $\ell^2(1/\pi)$ for $\ell^2(V, \pi)$ and $\ell^2(V, 1/\pi)$, respectively.

    Analogously, the Hilbert space $\ell^2(E, r)$ has inner product 
    
    \begin{align}
        \langle u, v \rangle_{\ell^2(E,r)} = \sum_{e \in E} u(e) v(e) r(e)
    \end{align}

    $\ell^2(E, r)$ is the set of functions $f : E \to \mathbb{R}$ with $\sum_{e \in E} |f(e)|^2 r(e) < \infty$.
\end{definition}

We introduce the following difference operator

\begin{align}
    \nabla : \ell^2(V, \pi) \to \ell^2(E, r), \quad \nabla f(e) = \frac{f(e^+) - f(e^-)}{r(e)}
\end{align}

Define the following norms:

\begin{definition}\label{defn:norms} (Norms)
    For a function $f : V \to \R$, define its Sobolev norm 
    \begin{align}
        S_P(f) := \frac{1}{2} \sum_{x,y \in V} |f(x) - f(y)| a(x,y)
    \end{align}
    and its norm in $\ell^p(V,\pi)$
    \begin{align}
        ||f||_p = \big(\sum_{x \in V} |f(x)|^p \pi(x) \big)^{1/p}
    \end{align}
    whenever these sums converge. Let $\mathcal{D(P)}$ be the space of all functions $f$ on $V$ such that $\nabla f \in \ell^2(E, r)$. If $f$ is such a function, define its Dirichlet norm
    \begin{align}
        D_P(f) &= \langle \nabla f, \nabla f\rangle_{\ell^2(E,r)} \\
        &= \sum_{e \in E} \frac{(f(e^+) - f(e^-))^2}{r(e)}\\
        &= \frac{1}{2}\sum_{x,y \in X} (f(x) - f(y))^2 a(x,y)
    \end{align}
\end{definition}

\paragraph{Preliminary Lemmas.} The next few Lemmas will be used to prove \Cref{thm:key}, which is key to proving the final \Cref{thm:main}.

\begin{lemma}\label{lem:IS-Sobolev} (Proposition 4.3 of \cite{woess})
    Markov chain $(V,P)$ satisfies $\ISineq_d$ ($1 \leq d \leq \infty$) if and only if 
    \begin{align}
        ||f||_{\frac{d}{d - 1}} \leq \kappa \, S_P(f) \quad \textup{for every $f \in \ell_0(V)$}
    \end{align}
\end{lemma}

\begin{proof}
    See Proposition 4.3 of \cite{woess} on page 40.
\end{proof}

\begin{lemma}\label{lem:dirichlet-norm} (Lemma 2.4 of \cite{woess})
If $f \in \mathcal{D}(P)$, then $\nabla^*(\nabla f) = - (P - I)f$    
\end{lemma}

\begin{proof}
    See Lemma 2.4 of \cite{woess} on page 16.
\end{proof}

Note that $\ell_0(V) \subset \mathcal{D}(P)$.

Now, we recall some facts of linear operators. Let $Q = (q(i,j))_{i,j \in I}$ be a non-negative real matrix over a countable set. $Q$ acts on functions (vectors) $I \to \mathbb{R}$ in the usual way by matrix multiplication on the left. Define the operator norm:

\begin{definition} (Definition 2.7-1 of \cite{Kreyszig})
    Let $X, Y$ be normed spaces and $Q : D(Q) \to Y$ a linear operator, where $D(Q) \subset X$. $Q$ is bounded if there is a real number $c$ s.t.
    \begin{align}
        \forall x \in D(Q), ||Qx|| \leq c ||x||
    \end{align}
    Moreover, define $||Q|| = \sup_{x \in D(Q), x \neq 0} \frac{||Qx||}{||x||}$
\end{definition}

Define the adjoint of a linear operator:

\begin{lemma} (Theorem 3.9-2 of \cite{Kreyszig})\label{lem:op-norm-transpose}
    The Hilbert-adjoint operator $Q^*$ of a bounded linear operator $Q : H_1 \to H_2$ such that
    \begin{align}
        Q^* &: H_2 \to H_1\\
        \forall x \in H_1, y \in H_2, \langle Qx, y \rangle &= \langle x, Q^* y \rangle
    \end{align}

    is unique and is a bounded linear operator with norm:
    \begin{align}
        ||Q^*|| = ||Q||
    \end{align}
\end{lemma}

If $Q = Q^*$, it is self-adjoint, and its operator norm has the following Rayleigh characterization.

\begin{lemma}\label{lem:op-norm-rayleigh} (Theorem 9.2-2 of \cite{Kreyszig})
    For any bounded, self-adjoint $Q$ on a complex Hilbert space $H$, we have

    \begin{align}
        ||Q|| = \sup_{||x|| = 1}|\langle Qx, x\rangle |
    \end{align}
\end{lemma}

\begin{proof}
    See Theorem 9.2-2 of \cite{Kreyszig}.
\end{proof}

\paragraph{Key Theorem.} \Cref{thm:key} is the key theorem which enables us to prove \Cref{thm:main}.

\begin{theorem}\label{thm:key}(IS implies Spectral Expansion)
\end{theorem}
    Consider an irreducible, reversible Markov chain $(V, P)$ that satisfies the strong Isoperimetric inequality with Isoperimetric constant $\kappa$:

    \begin{align}
         \textup{For every finite } A \subset V, \quad \pi(A) \leq \kappa \cdot a(\partial A)
    \end{align}

    Then, the operator norm of $P$ in $\ell^2(V,\pi)$ is at most $1 - \frac{1}{2\kappa^2}$. That is, $||P||_{\ell^2(V, \pi) \to \ell^2(V, \pi)} \leq 1 - \frac{1}{2\kappa^2}$.
    
\begin{proof}
    This follows from the proof of Theorem 10.3 of \cite{woess}, but we write it out here for completeness.

    By \Cref{lem:IS-Sobolev}, IS implies the Sobolev inequality $||f||_1 \leq \kappa \, S_P(f)$ for all $f \in \ell_0(V)$.  First, we show that this implies the Dirichlet inequality, $||f||_2^2 \leq 2\kappa^2 D_P(f)$, because:

    \begin{align}
        ||f||_2^4 &= ||f^2||_1^2\\
        &\leq \kappa^2 S_P(f^2)^2\\
        &\leq \kappa^2 \big( \frac{1}{2} \sum_{x,y \in V} a(x,y)  |f(x) - f(y)| (|f(x)| + |f(y)|)\big)^2\\
        &\leq \kappa^2 D_P(f) \sum_{x,y \in V} a(x,y) (|f(x)| + |f(y)|)^2/2\\
        &\leq \kappa^2 D_P(f)\sum_{x,y \in V} a(x,y) (f(x)^2 + f(y)^2)\\
        &\leq 2\kappa^2 D_P(f) ||f||_2^2\\
        \implies ||f||_2^2 &\leq 2\kappa^2 D_P(f)
    \end{align}

By \Cref{lem:dirichlet-norm},

\begin{align}
    D_P(f) &= \langle f, (I - P)f \rangle_{\ell^2(\pi)}\\
    &= ||f||_2^2 - \langle f, Pf \rangle_{\ell^2(\pi)} \quad \textup{for $f \in \ell_0(V)$}\\
    \implies \langle f, Pf \rangle_{\ell^2(\pi)} &= ||f||_2^2 - D_P(f)\\
    &\leq (1 - \frac{1}{2\kappa^2})||f||^2_2\label{eq:rayleigh-upper-bound-for-P}
\end{align}

Where the $||\cdot||_2^2 = ||\cdot||^2_{\ell^2(\pi)}$ is the weighted 2-norm defined in \Cref{defn:norms}. By \Cref{lem:op-norm-rayleigh},

\begin{align}
    ||P||_{\ell^2(\pi) \to \ell^2(\pi)} &= \sup_{f \in \ell^2( \pi), f \neq 0} \frac{|\langle f, Pf\rangle_{\ell^2(\pi)}|}{||f||_2^2}\\
    &\leq \sup_{f \in \ell^2(\pi), f \neq 0} \frac{|\sum_{x,y \in V} f(x)f(y)\pi(x) P[x,y]|}{||f||_2^2}\\
    &\leq \sup_{f \in \ell^2( \pi), f \neq 0} \frac{\sum_{x,y \in V} \, \pi(x) P[x,y] \,|f(x)| |f(y)|}{||f||_2^2}\\ 
    &=  \sup_{f \in \ell^2(\pi), f \neq 0} \frac{\langle |f|, P|f|\rangle_{\ell^2(\pi)}}{||f||_2^2}\\
    &=  \sup_{f \in \ell^2( \pi), f \neq 0} \frac{\langle f, Pf \rangle_{\ell^2(\pi)}}{||f||_2^2}\\
    &=  \sup_{f \in \ell_0(V), f \neq 0} \frac{\langle f, Pf\rangle_{\ell^2(\pi)}}{||f||_2^2} \tag{Since $\ell_0(V)$ is dense in $\ell^2(V, \pi)$}\\
    &\leq 1 - \frac{1}{2\kappa^2} \tag{By \Cref{eq:rayleigh-upper-bound-for-P}}
\end{align}

We conclude that $||P||_{\ell^2(\pi) \to \ell^2(\pi)} \leq 1 - \frac{1}{2\kappa^2}$, where $||P||_{\ell^2(\pi) \to \ell^2(\pi)}$ denotes the operator norm of $P$ over $\ell^2(\pi)$.
    
\end{proof}

The following Corollary establishes a one-step contraction due to the random walk according to $P$ on a graph satisfying $\ISineq$.

\begin{corollary}\label{lem:rw}
    Let $G = (V,E)$ and let $(V,P)$ be a reversible random walk on $G$, with stationary measure $\pi$, satisfying Assumption \ref{assumption:graph} with Isoperimetric constant $\kappa > 0$. Then, for any $D \in \ell^2(1/\pi)$, we have

    \begin{align}
        ||P^\top D||_{\ell^2(1/\pi)} \leq (1 - \frac{1}{2\kappa^2})||D||_{\ell^2(1/\pi)} 
    \end{align}
    

\end{corollary}

\begin{proof}

     The core claim is that $||P^\top||_{\ell^2(1/\pi) \to \ell^2(1/\pi)} = ||P||_{\ell^2(\pi) \to \ell^2(\pi)}$. To see this, because $P$ is reversible with stationary measure $\pi$, we have

     \begin{align}
         \diag(\pi) P &= P^\top \diag(\pi)\\
         P^\top &= \diag(\pi) P\diag(\pi)^{-1}
     \end{align}

     Thus,

     \begin{align}
         ||P^\top||_{\ell^2(1/\pi) \to \ell^2(1/\pi)} &= \sup_{x \,:\, ||x||_{\ell^2(1/\pi)} = 1} ||P^\top x||_{\ell^2(1/\pi)}\\
         &= \sup_{x \,:\, ||x||_{\ell^2(1/\pi)} = 1} ||\diag(\pi) P\diag(\pi)^{-1} x||_{\ell^2(1/\pi)}\\
         &= \sup_{x \,:\, ||x||_{\ell^2(1/\pi)} = 1} ||P\diag(\pi)^{-1} x||_{\ell^2(\pi)}\\
         &= \sup_{x' \,:\, ||x'||_{\ell^2(\pi)} = 1} ||Px'||_{\ell^2(\pi)} \label{eq:change-of-space}\\
         &= ||P||_{\ell^2(\pi) \to \ell^2(\pi)}
     \end{align}

     \Cref{eq:change-of-space} follows from the fact that for any $x \in \ell^2(1/\pi), x' = \diag(\pi)^{-1}x \in \ell^2(\pi)$. Conversely, for any $x' \in \ell^2(\pi)$, $x = \diag(\pi) x' \in \ell^2(1/\pi)$.
     
     Finally, by \Cref{thm:key}, $||P||_{\ell^2(\pi) \to \ell^2(\pi)} \leq 1 - \frac{1}{2\kappa^2}$. Since $||P||_{\ell^2(\pi) \to \ell^2(\pi)} = ||P^\top||_{\ell^2(1/\pi) \to \ell^2(1/\pi)}$, then for any $D \in \ell^2(1/\pi)$, we have

    \begin{align}
        ||P^\top D||_{\ell^2(1/\pi)} \leq ||P^\top||_{\ell^2(1/\pi) \to \ell^2(1/\pi)}\, ||D||_{\ell^2(1/\pi)}  \leq  (1 - \frac{1}{2\kappa^2})||D||_{\ell^2(1/\pi)} 
    \end{align}
\end{proof}

\subsection{Proof of \Cref{thm:main}}\label{appen:pf-of-thm-42}

\expansionthm*

\begin{proof}
    We reason inductively, by analyzing the behavior of \Cref{alg:main} each iteration. 
    
    First, suppose at iteration $t \in \N$,  \Cref{alg:main} has i.i.d. sample access to $D_t \in \Delta(V)$, and trains prover $f_t \in \Fprov$ to where $\Pro_{x \sim D_t} [\Succ(x, f_t)  \geq p] \geq 1 - \eps$ (an event which occurs with probability at least $1 - \delta$ according to \Cref{assumption:a2-rt}). \Cref{alg:main} (using the subroutine, \Cref{alg:pass-rate-sampler}) generates $C_w\# D_t$ from $D_t$ and then generates $D_{t + 1}$ from $C_w\# D_t$ by filtering out theorems whose pass rate is below $p/2$. By \Cref{lem:rw},

    \begin{align}
        ||C_w\# D_t||_{\ell^2(1/\pi)}^2  := ||(P^\top) D_t||_{\ell^2(1/\pi)}^2 \leq (1 - \frac{1}{2\kappa^2})^2 ||D_t||^2_{\ell^2(1/\pi)}\label{eq:one-step-rw-ent}
    \end{align}

    Under the event that $\Pro_{x \sim D_t} [\Succ(x, f_t)  \geq p] \geq 1 - \eps$, let

    \begin{align}
        \alpha_t &:= \Pro_{x \sim C_w\# D_t} [\Succ(x, f_t)  \geq p/2] \\
        &\geq \Pro_{x \sim D_t} [\Succ(x, f_t)  \geq p]\\
        &\geq 1 - \eps
    \end{align}

    Let $D_{t + 1}$ be a new distribution whose sampling process involves sampling $x \sim C_w\# D_t$ until $\Succ(x, f_t) \geq p/2$ and then returning such an $x$ (i.e. the rejection sampling step in \Cref{alg:pass-rate-sampler}). This process provides i.i.d. sample access to $D_{t + 1}$ supported on the $\alpha_t$-fraction of theorems in $C_w\# D_t$ where  $\Succ(x, f_t)  \geq p/2$. The probability of any theorem under $D_{t + 1}$ is equal to its respective probability under $C_w\# D_t$ multiplied by $\frac{1}{\alpha_t} \leq \frac{1}{1 - \eps}$. Since for $f \in \ell^2(1/\pi)$, $||f||_{\ell^2(1/\pi)}^2 = \sum_{v \in V}  |f(v)|^2 \frac{1}{\pi(v)}$, then 

    \begin{align}
        ||D_{t + 1}||^2_{\ell^2(1/\pi)} \leq (\frac{1}{1 - \eps})^2||C_w\# D_t||^2_{\ell^2(1/\pi)}
    \end{align}
    
    Combining with \Cref{eq:one-step-rw-ent}, we conclude 

    \begin{align}
        ||D_{t + 1}||_{\ell^2(1/\pi)}^2  \leq (\frac{1 - \frac{1}{2\kappa^2}}{1 - \eps})^2 ||D_t||_{\ell^2(1/\pi)}^2\label{eq:per-it-guarantee}
    \end{align}

    We now describe how these per-iteration guarantees can be put together inductively. We claim the following iteration invariant: for every iteration $t \in \N$, conditional on the event that all previous iterations did not fail, then \Cref{alg:main} has access to a theorem distribution $D_t$ satisfying the premise of \Cref{assumption:a2-rt}, that $\forall x \in \supp(D_t),$ the proof of $x$ is known. Moreover, the $t$th iteration fails with probability at most $\delta$.
    
    As the base case, \Cref{assumption:a1-js} ensures \Cref{alg:main} has i.i.d. sample access to cold-start theorem distribution $D_0$, and that all proofs of theorems in $\supp(D_0)$ are known. The iteration invariant holds. 
    
    Regarding the inductive step, suppose at iteration $t$, the algorithm has i.i.d. sample access to $D_t$, where for all $x \in \supp(D_t)$, the proof of $x$ is known. Suppose also that all previous iterations $\{0 \} \cup [t - 1]$ did not fail. By \Cref{assumption:a2-rt}, \Cref{alg:main} can train a new prover $f_{t}$ using i.i.d. sample access to $D_{t}$, which guarantees that $\Pro_{x \sim D_{t}} [\Succ (x, f_{t}) \geq p] \geq 1 - \eps$ with failure probability at most $\delta$. By \Cref{assumption:a3-ood}, $\Pro_{x \sim D_{t}} [\Succ (x, f_{t}) \geq p] \geq 1 - \eps$ implies that $\Pro_{x \sim C_w \# D_{t}} [\Succ (x, f_{t}) \geq p/2] \geq 1 - \eps$. The rejection sampling procedure in \Cref{alg:pass-rate-sampler}, which generates $D_{t + 1}$ from $C_w\# D_t$, ensures that $\forall x \in \supp(D_{t + 1})$, $\Succ(x, f_t) \geq p/2$, so that it is efficient to sample the proof of any theorem in $\supp(D_{t + 1})$ from $f_t$. Thus, the proof of every theorem in $\supp(D_{t + 1})$ is known, and $D_{t + 1}$ satisfies the premise of \Cref{assumption:a2-rt}. This iteration fails with probability at most $\delta$, which occurs if step that trains $f_t$ from $D_t$ fails.
    
    Thus, by induction over the iteration number $t \in \N$, we conclude from \Cref{eq:per-it-guarantee} that for every $t \in \N$, conditional on the event that every iteration $i \in \{0 \} \cup [t - 1]$ does not fail (which occurs with probability at least $1 - t \, \delta$ since each iteration fails with probability at most $\delta$), then

    \begin{align}
        ||D_{t}||^2_{\ell^2(1/\pi)} \leq (\frac{1 - \frac{1}{2\kappa^2}}{1 - \eps})^{2t} ||D_{0}||^2_{\ell^2(1/\pi)}\label{eq:Dt-vs-D0}
    \end{align}
    
    Finally, \Cref{alg:main} trains the final prover $f_t \in \Femb$ on $D_t$ at iteration $t$. With failure probability at most $\delta$, we have $D_t(\{ x \in V : \Succ(x,f_t) \geq p\}) \geq 1 - \eps$.
    
    At iteration $t\in\N$, suppose none of the training guarantees of the $t + 1$ provers $\{ f_0, f_1, \ldots, f_t\}$ fail, which occurs with probability at least $1 - (t + 1) \, \delta$. Applying \Cref{lem:ent-to-cov-num} with $D' = D_t$, $A' = \{ x \in V : \Succ(x,f_t) \geq p\}$ and the guarantee that $D_t(\{ x \in V : \Succ(x,f_t) \geq p\}) \geq 1 - \eps$, we have

    \begin{align}
        |\{ x \in V : \Succ(x,f_t) \geq p\}| \geq \frac{(1 - \eps)^2}{\sup_{z \in V} \pi(z)} \cdot \frac{1}{ ||D_{t}||^2_{\ell^2(1/\pi)}} 
    \end{align}
    
    By \Cref{eq:Dt-vs-D0}, the conclusion follows.

    \begin{align}
        | \{ x \in V : \Succ(x, f_t) \geq p\} | \geq \frac{(1 - \eps)^2}{(\sup_{z \in V} \pi(z)) \cdot ||D_0||_{\ell^2(1/\pi)}^2 } \cdot (\frac{1 - \eps}{1 - \frac{1}{2\kappa^2}})^{2t}.
    \end{align}

    
    Note also that since  $\Pro_{x \sim C_w\# D_t} [\Succ(x, f_t)  \geq p/2] \geq 1 - \eps$ at iteration $t$, the probability of rejection each iteration is at most $\eps$, so that the rejection sampling procedure per iteration is fairly efficient. However, the probability of rejection compounds over iterations, so that the expected number of samples from $D_0$ needed to sample from $D_t$ is proportional to $(\frac{1}{1 - \eps})^t$.

\end{proof}

%% file: sections/appendix/diffusion-and-cont-learning.tex
\section{Theorems about Diffusion Similarity and Contrastive Learning}\label{appen:cont-learning-intro}

The purpose of this section is to show the connection between contrastive learning and diffusion similarity, and show how an embedding model learned by minimizing a contrastive loss can be used to compute the diffusion similarity (and diffusion distance) between points on a graph. Our main results, in \Cref{appen:contembd}, include \Cref{thm:contembd}, an infinite-sample and large embedding dimension learning result, and \Cref{thm:e2e-contrastive-learning-result}, a finite-sample and small embedding dimension learning result.

\subsection{Preliminary Lemmas}
\subsubsection{Setup}

Contrastive learning defines an algorithm for learning an embedding model that maps elements from a domain $V$ (e.g. the set of real-world images) into Euclidean space $\R^k$ for some $k \in \N$. It is most notably used for training neural networks as embedding models, and we let $\Femb$ denote the hypothesis class of models $h_\theta : V \to \R^k$ which we assume contains the ground-truth embedding model that we must learn. The core idea of contrastive learning is to define two distributions of pairs of elements in $V$: a positive pair distribution $D_+ \in \Delta(V \times V)$ , and a negative pair distribution $D_- \in \Delta(V \times V)$. Given i.i.d. samples from both distributions, the trained model is encouraged to embed two elements sampled from a positive pair $(x,y) \sim D_+$ close together in Euclidean space, while embedding two elements sampled from a negative pair $(x,x') \sim D_-$ far apart in Euclidean space. These two distributions of pairs of elements are used to define a contrastive loss, and training the model amounts to minimizing an empirical estimate of the contrastive loss from finite samples.

In fact, contrastive learning can be described mathematically through the lens of matrix factorization, as articulated in \cite{haochen2022provableguaranteesselfsuperviseddeep}. The benefit of this viewpoint of contrastive learning is that it enables us to provably characterize the learned embeddings and the sample complexity of learning them. In the following, we will describe and use the mathematical setup for contrastive learning of \cite{haochen2022provableguaranteesselfsuperviseddeep}. 

Let the domain $V$ consist of $N \in \N$ elements, where $N$ is large but finite. We will identify $V$ with $[N]$ here for simplicity. Define $\mu \in \Delta([N])$ as the \textit{base distribution} and $Q \in \R^{N \times N}$ as an augmentation scheme. Alternatively, overloading notation, we can view $Q : [N] \to \Delta(N)$ as a mapping from an element $x$ to a distribution over augmentations of $x$, $Q(x)$. For every $x \in [N]$, this distribution of augmentations, $Q(x)$, equals the $x$-th row of the matrix-representation of $Q$. 

We will first introduce the matrix factorization viewpoint of contrastive learning, and we will from this how to define the positive and negative pair distributions $D_+, D_- \in \Delta([N]\times [N])$ from $\mu$ and $Q$ and the final learning algorithm. Let $A \in \R^{N \times N}$ and its normalization $\overline{A}$  be positive semi-definite (p.s.d.) matrices of the form,

\begin{align}
    A &= Q^\top \,\diag(\mu) \,Q\\
    D &= \diag(A \1)\\
    \overline{A} &= D^{-1/2} A D^{-1/2}
\end{align}

For rank-$k$ factor $F_\theta \in \R^{N \times k}$, let $\Lmf(F_\theta) := ||\overline{A} - F_\theta F_\theta^\top||_F^2$ be the matrix factorization loss for matrix $\overline{A} \in \R^{N \times N}$. Let $\Femb$ be the class of embedding models. We will parameterize $F_\theta$ by an embedding model $h_\theta : [N] \to \R^k$. Towards this, if we let $\sigma_{xy}$ denote the $(x,y)$-th entry of $A$,

\begin{align}
    A &= [\sigma_{xy}]_{x,y \in [N]}\\
    \textup{With } \sigma_x &:= \sum_{y \in [N]} \sigma_{xy}\label{eq:wx-defn}
\end{align}

and we parameterize $F_\theta \in \R^{N \times k}$ where $\forall x \in [N], F_\theta[x] = \sqrt{\sigma_x} h_\theta(x)^\top$ for $h_\theta : [N] \to \R^k$, then $\Lmf(F_\theta) = \Lmf(h_\theta)$ (overloading notation) is expressible in the following form. 

\begin{align}
    \Lmf(h_\theta) &:= ||\overline{A} - F_\theta F_\theta^\top||_F^2\\
    &= \sum_{x,y \in [N]} \Big(\frac{\sigma_{xy}}{\sqrt{\sigma_x} \sqrt{\sigma_y}} - \sqrt{\sigma_x}\sqrt{\sigma_y}h_\theta(x)^\top h_\theta(y) \Big)^2\\
    &= \textup{const} + \sum_{x,y \in [N]}\sigma_x \sigma_y (h_\theta(x)^\top h_\theta(y))^2 - 2\sum_{x,y \in [N]}\sigma_{xy} h_\theta(x)^\top h_\theta(y)\label{eq:pos-neg-pair-prob}\\
    &= \textup{const} + \ex_{x, x^-} (h_\theta(x)^\top h_\theta(x^-))^2 - 2 \cdot \ex_{x, x^+}h_\theta(x)^\top h_\theta(x^+)\\
    &= \textup{const} + \Lc(h_\theta)\label{eq:cont-mf-equiv}\\
    \textup{Where } \Lc &: \F_V^{(k)} \to \R\\
    \Lc(h_\theta) &:= \ex_{x, x^-} (h_\theta(x)^\top h_\theta(x^-))^2 - 2 \cdot \ex_{x, x^+}h_\theta(x)^\top h_\theta(x^+)\label{eq:cont-loss-defn}
\end{align}

Where $\F_{V}^{(k)}$ denotes the class of all functions mapping $V$ to $\R^{k}$. Amazingly, $\Lc(h_\theta) := \ex_{x, x^-} (h_\theta(x)^\top h_\theta(x^-))^2 - 2 \cdot \ex_{x, x^+}h_\theta(x)^\top h_\theta(x^+)$ is one version of the (population) contrastive loss function, similar to what is used empirically.

\Cref{eq:pos-neg-pair-prob} in the derivation above naturally defines the positive pair distribution as where $(x,y) \in [N]\times[N]$ has marginal probability of $\sigma_{xy}$ of being sampled. This corresponds to the sampling process of first sampling an element from the base distribution first, $z \sim \mu$, and then sampling the pair elements as independent augmentations $x \sim Q(z)$, $y\sim Q(z)$, returning $(x,y)$. Similarly, since a negative pair $(x,y) \in [N]\times[N]$ has marginal probability of $\sigma_x \sigma_y$ in \Cref{eq:pos-neg-pair-prob}, the natural negative pair distribution corresponds to first sampling $z_x, z_y \sim \mu$ independently, and then sampling $x \sim Q(z_x), y \sim Q(z_y)$, returning $(x,y)$.

Crucially, since $\Lc(h_\theta)$ and $\Lmf(h_\theta)$ differ by only an additive constant independent of $h_\theta$, we can understand the global minimizers of $\Lc$ via the global minimizers of $\Lmf$. In particular, the Eckart-Young-Mirsky Theorem  characterizes the global minimizers of $\Lmf$ via the spectrum of $\overline{A}$ \cite{eckartyoung}. In the case $\overline{A} = U \Sigma U^\top$ is p.s.d., then the global minimizer $F_k \in \R^{N \times k}$ is of the form

\begin{align}
    F_k = U_k \Sigma_k^{1/2}R
\end{align}

Where $R \in O^{k \times k}$ is any $(k \times k)$ orthogonal matrix.

In addition to having an analytic characterization of the global minima, one other useful property of $\Lc$ is that it is written as an expectation, and so can be approximated by an empirical risk. For a dataset $\{ \overline{x}_i\}_{i \in [n]} \subset V$, define the empirical contrastive loss as:

\begin{align}
    &\Lchat(h_{\theta'}) :=\\
    &-\frac{2}{n} \sum_{i \in [n]} \ex_{\substack{x \sim Q(\cdot | \overline{x}_i),\\ x^+ \sim Q(\cdot | \overline{x}_i)}} (h_{\theta'}(x)^\top h_{\theta'}(x^+)) + \frac{1}{n(n - 1)} \sum_{i,j \in [n], i \neq j}\ex_{\substack{x \sim Q(\cdot | \overline{x}_i),\\ x^- \sim Q(\cdot | \overline{x}_j)}} (h_{\theta'}(x)^\top h_{\theta'}(x^-))^2 \label{eq:emp-cont-loss-defn}
\end{align}

Claim D.2 of \cite{haochen2022provableguaranteesselfsuperviseddeep} shows that $\ex_{\{ \overline{x}_i\}_{i \in [n]} \sim \mu} \Lchat(h_{\theta}) = \Lc(h_\theta)$, i.e. that $\Lchat$ is unbiased. Define the empirical risk minimizer (ERM) as follows.

\ermdefnappen



\Cref{defn:contrastive-learning} describes a procedure to use empirical risk minimization to learn an embedding model from finite samples from the base distribution.

\begin{algorithm}[t]
\caption{Learning Embeddings via ERM}
\label{defn:contrastive-learning}
\begin{algorithmic}[1]
\Require i.i.d. sample access to base distribution $D \in \Delta(V)$, augmentation function $Q:V\to\Delta(V)$, parameters $n,k\in\N$, and embedding model class $\Femb$ of models that embed into $\R^k$
\Ensure An embedding model $h_\theta \in \Femb$

\State Draw i.i.d. samples $\bar{x}_1,\ldots,\bar{x}_n \sim D$
\State Define the empirical contrastive loss
\[
\begin{aligned}
\Lchat(h_{\theta'})
:={}&
-\frac{2}{n}\sum_{i\in[n]}
\ex_{\substack{x\sim Q(\cdot\mid \bar{x}_i)\\
x^+\sim Q(\cdot\mid \bar{x}_i)}}
\!\left[h_{\theta'}(x)^\top h_{\theta'}(x^+)\right] \\
&+
\frac{1}{n(n-1)}
\sum_{\substack{i,j\in[n]\\ i\neq j}}
\ex_{\substack{x\sim Q(\cdot\mid \bar{x}_i)\\
x^-\sim Q(\cdot\mid \bar{x}_j)}}
\!\left[\left(h_{\theta'}(x)^\top h_{\theta'}(x^-)\right)^2\right].
\end{aligned}
\]
\State Compute $h_\theta \gets \arg\min_{h_{\theta'}\in\Femb}\Lchat(h_{\theta'})$
\State \Return $h_\theta$
\end{algorithmic}
\end{algorithm}


    


As we will later see in  \Cref{lem:haochen-thm4-1} and \Cref{lem:contrastive-learning}, under some assumptions, the empirical risk minimizer (ERM) will be also approximately minimize the population contrastive loss. The assumptions necessary for that result are expressed in terms of Rademacher complexity, a core concept of Statistical Learning Theory. For $k \in \N$, we use the definition of empirical Rademacher complexity of a function class $\F$ of models $\F \ni h: V \to \R^k$ used in \cite{haochen2022provableguaranteesselfsuperviseddeep}.


\raddefnappen

Note this definition takes a $\max$ over the $n$ datapoints $x_1, \ldots, x_n$ rather than an expectation, as is usual \cite{ma}. \Cref{lem:haochen-thm4-1} analyzes the generalization error of the ERM, using Statistical Learning Theory. 

\subsubsection{Lemmas on Contrastive Learning}

\begin{lemma}\label{lem:haochen-thm4-1} (Theorem 4.1 of \cite{haochen2022provableguaranteesselfsuperviseddeep})
    Fix a base distribution $\mu \in \Delta(V)$, augmentation function $Q : V \to \Delta(V)$, and parameters $k \in \N$, $k \leq |V|$ and $\perm \in (0,1)$. Let $\Femb \subset \F_V^{(k)}$ be a class of embedding models $h_\theta$ that map a domain $V$ into dimension $k$ Euclidean space: $\Femb \ni h_\theta : V \to \R^k$. Let $\Lc : \F_V^{(k)} \to \R$ be the population contrastive loss, instantiated with $\mu$ and $Q$, as in \Cref{eq:cont-loss-defn}, and $\Lchat : \F_V^{(k)} \to \R$ be the empirical contrastive loss over a i.i.d. dataset of size $n$ drawn from $\mu$, as in \Cref{eq:emp-cont-loss-defn}. Assume:

    \begin{enumerate}
        \item For some $\chi > 0$ that $\forall h_\theta \in \Femb, \forall x \in V$, $||h_\theta(x)||_\infty \leq \chi$. 
        \item $\Femb$ contains at least one global minimizer of $\Lc$ over functions $h : V \to \R^k$
    \end{enumerate}

    Then, with $h_{\theta_*} \in \Femb$ any global minimizer of $\Lc$ and $h_{\hat{\theta}} \in \Femb$ any global minimizer of $\Lchat$, for which a training dataset of size $n$ was drawn i.i.d. from $\mu$, then with probability at least $1 - \perm$ over the randomness of the training dataset,

    \begin{align}
        \Lc(h_{\hat{\theta}}) \leq \Lc(h_{\theta_*}) + \poly(k,\chi) \, \Rchat_{n/2}(\Femb) + \poly(k,\chi) \, (\sqrt{\frac{\log \frac{2}{\perm}}{n}} + \perm)
    \end{align}
    
\end{lemma}

\begin{proof}
    Same as Theorem 4.1 of \cite{haochen2022provableguaranteesselfsuperviseddeep}.
\end{proof}

\cite{haochen2022provableguaranteesselfsuperviseddeep} say, as a comment immediately after their Theorem 4.1, that $\Rchat_{n}(\F)$ typically looks like $\sqrt{\frac{\C(\F)}{n}}$ where $\C(\F)$ is an $\F$-dependent constant. Assuming this is the case, then as a consequence of \Cref{lem:haochen-thm4-1}, the minimizer of the empirical contrastive loss, with $n = \Theta(\frac{\C(\Femb)\, \poly(k, \chi)}{\eps^2}\log \frac{1}{p_{\textup{erm fails}}})$ i.i.d. samples, is $\eps$-minimizer embedding model with probability at least $1 - p_{\textup{erm fails}}$. We write this explicitly in the following Lemma.

\begin{restatable}[Generalization of ERM]{lemma}{ermforcont}
\label{lem:contrastive-learning}

    Fix a base distribution $\mu \in \Delta(V)$, augmentation function $Q : V \to \Delta(V)$, and parameters $k \in \N$, $k \leq |V|$,  and $\eps \in (0,1)$. Let $\Femb \subset \F_V^{(k)}$ be a class of embedding models $h_\theta$ that map a domain $V$ into dimension $k$ Euclidean space: $\Femb \ni h_\theta : V \to \R^k$. Let $\Lc : \F_V^{(k)} \to \R$ be the population contrastive loss, instantiated with $\mu$ and $Q$, as in \Cref{eq:cont-loss-defn}, and $\Lchat : \F_V^{(k)} \to \R$ be the empirical contrastive loss over a i.i.d. dataset of size $n$ drawn from $\mu$, as in \Cref{eq:emp-cont-loss-defn}. Assume:

    \begin{enumerate}
        \item For some $\chi > 0$ that $\forall h_\theta \in \Femb, \forall x \in V$, $||h_\theta(x)||_\infty \leq \chi$. 
        \item $\Femb$ contains at least one global minimizer of $\Lc$ over functions $h : V \to \R^k$
        \item $\Rchat_n(\Femb) \leq \sqrt{\frac{\C(\Femb)}{n}}$, where $\C(\Femb)$ is an $\Femb$-dependent constant
    \end{enumerate}

    Set $\perm = \frac{\eps}{\poly(k, \chi)}$. Then, with $h_{\theta_*} \in \Femb$ any global minimizer of $\Lc$ and $h_{\hat{\theta}} \in \Femb$ any global minimizer of $\Lchat$, for which a training dataset of size $n = \Theta(\frac{\C(\Femb)\, \poly(k, \chi)}{\eps^2}\log \frac{1}{p_{\textup{erm fails}}})$ was drawn i.i.d. from $\mu$, then with probability at least $1 - \perm$ over the randomness of the training dataset,

    \begin{align}
        \Lc(h_{\hat\theta}) \leq  \Lc(h_{\theta_*}) + \eps\label{eq:cont-eps-min}
    \end{align}

    Equivalently, with $F_{\hat\theta} \in \R^{|V| \times k}$ where $\forall x \in [N], F_{\hat\theta}[x] = \sqrt{\sigma_x} h_{\hat\theta}(x)^\top$ (with $\sigma_x$ defined in \Cref{eq:wx-defn}), then
    
    \begin{align}
        \Lmf (F_{\hat\theta}) := ||\overline{A} - F_{\hat\theta} F_{\hat\theta}^\top||^2_F &\leq \min_{F :\,\, \textup{rank}(F) = k} ||\overline{A} - F F^\top||^2_F + \eps \label{eq:mf-eps-min}
    \end{align}
\end{restatable}

\begin{proof}
    Apply \Cref{lem:haochen-thm4-1},  plugging in  $\Rchat_{n/2} \leq \sqrt{\frac{\C(\Femb)}{n/2}}$, $\perm =  \frac{\eps}{\poly(k, \chi)}$, and  $n = \Theta(\frac{\C(\Femb) \, \poly(k, \chi)}{\eps^2}\log \frac{1}{p_{\textup{erm fails}}})$, with the constant scaled so that the bound from \Cref{lem:haochen-thm4-1} becomes

    \begin{align}
        \Lc(h_{\hat{\theta}}) &\leq \Lc(h_{\theta_*}) + \poly(k,\chi) \, \Rchat_{n/2}(\Femb) + \poly(k,\chi) \, (\sqrt{\frac{\log \frac{2}{\perm}}{n}} + \perm)\\
        &\leq \Lc(h_{\theta_*}) + \eps
    \end{align}

    This establishes \Cref{eq:cont-eps-min}. \Cref{eq:mf-eps-min} follows from \Cref{eq:cont-eps-min} and  \Cref{eq:cont-mf-equiv}, which we showed above:

    \begin{align}
        \forall h_\theta \in \Femb, \Lmf(h_\theta) = c + \Lc(h_\theta) 
    \end{align}

    Where constant $c$ is independent of $h_\theta$. 
\end{proof}

Note that as the proof of \Cref{lem:haochen-thm4-1} bounds the excess test loss of the ERM with uniform convergence, this result can actually be strengthened: any $\frac{\eps}{2}$-minimizer of the empirical contrastive loss, with $n = \Theta(\frac{\C(\Femb) \, \poly(k, \chi)}{\eps^2}\log \frac{1}{p_{\textup{erm fails}}})$ i.i.d. samples, is $\eps$-minimizer embedding model with probability at least $1 - p_{\textup{erm fails}}$. 



\subsubsection{The Connection Between Diffusion Embeddings and Contrastive Learning}

In the following, assume the graph $G = (V,E)$ is finite ($|V| < \infty$), undirected, and connected.

\begin{definition}\label{defn:spec-diff-embd} (Spectral Diffusion Embedding)
Let $G = (V,E)$ be a connected, undirected graph. Let $(V,P)$ be a reversible Markov chain with transition matrix $P \in [0,1]^{|V| \times |V|}$ and stationary measure $\pi$. Define $\{\phi_i\}_{i \in [|V|]}$ as the right eigenvectors of $P$ such that $P \phi_i = \lambda_i \phi_i$, with $1 = \lambda_1 \geq  |\lambda_2| \geq \ldots \geq |\lambda_{|V|}| \geq 0$. The normalization of the eigenvectors is where $\forall i \in [|V|]$, $\diag(\pi)^{1/2} \phi_i$ is a unit vector.\footnote{Alternatively, note that since $P$ is reversible, $\diag(\pi)P$ is symmetric. Let symmetric matrix $\overline{A} = \diag(\pi)^{1/2}P \diag(\pi)^{-1/2} = \sum_{i}\lambda_i v_i v_i^\top$ with $\{ v_i\}_{i \in [|V|]}$ being orthonormal. Then $\forall i \in [|V|], \, \phi_i = \diag(\pi)^{-1/2} v_i$.  } Denote $\phi_i(x)$ as the $x$-th entry of $\phi_i \in \R^{|V|}$. For any $t \in \N$, define the scale-$t$ spectral diffusion embedding as

\begin{align}
    \forall x \in V, \, \Phi_t(x) = (\lambda_i^t \phi_i(x))_{i \in [|V|]} \in \R^{|V|}
\end{align}

Moreover, for any $k \in [|V|],$ define the $k$-truncated scale-$t$ spectral diffusion embedding as

\begin{align}
    \forall x \in V, \, \Phi^{\leq k}_t(x) = (\lambda_i^t \phi_i(x))_{i \in [k]} \in \R^{k}
\end{align}
\end{definition}

\Cref{lem:diff-dist} establishes a connection between diffusion similarity and Euclidean inner product of the spectral diffusion embeddings.

\begin{restatable}[Diffusion Embeddings]{lemma}{diffdistlemma} \label{lem:diff-dist}
Let $G = (V,E)$ be a connected, undirected graph. Let $(V,P)$ be a reversible Markov chain with transition matrix $P \in [0,1]^{|V| \times |V|}$ and stationary measure $\pi$. For any $t \in \N$ and $x \in V$, with $ \Phi_t(x) = (\lambda_i^t \phi_i(x))_{i \in [|V|]} \in \R^{|V|}$  denoting the spectral diffusion embedding of $x$  as in \Cref{defn:spec-diff-embd}, we have

\begin{align}
    \forall x,y \in V, \langle \Phi_t(x), \Phi_t(y) \rangle = \langle p^t(x), p^t(y) \rangle_{\ell^2(1/\pi)}
\end{align}

where $\langle \cdot, \cdot \rangle$ is the usual Euclidean inner product and $\langle \cdot, \cdot \rangle_{\ell^2(1/\pi)}$ is the inner product where $\forall i \in [|V|]$, the $i$th term is weighted by $\frac{1}{\pi(i)}$.
\end{restatable}

\begin{proof}
    Because $(V,P)$ is reversible, let $P = D^{-1} W \in \R^{|V| \times |V|}$ where $D = \textup{diag}(W \1) = \diag(\pi)$ and $W  = \diag(\pi)P \in \R^{|V| \times |V|}$ is symmetric. The symmetric matrix $\overline{A} = D^{-1/2}WD^{-1/2}$ is the normalized adjacency matrix, with eigendecomposition $\overline{A} = \sum_{i \in [|V|]} \lambda_i u_i u_i^\top$. Since $P = D^{-1/2}\overline{A}D^{1/2}$, then $P = \sum_{i \in [|V|]} \lambda_i (D^{-1/2}u_i) (D^{1/2}u_i)^\top$, so that $\forall i \in [|V|], \phi_i = D^{-1/2}u_i$. Now, for any $t \in \N$ and $x \in V$,

    \begin{align}
        p^t(x) &:= (P^\top)^t e_x\\
        &= D^{1/2} \overline{A}^t D^{-1/2} e_x\\
        &= D^{1/2} \sum_{i \in [|V|]} \lambda_i^t u_i \phi_i(x)\\
        &= D^{1/2} U\Phi_t(x)
    \end{align}

    Where $U$ is an orthogonal transformation $U = [u_1, \ldots, u_{|V|}] \in \R^{|V| \times |V|}$. It follows that for any $x,y \in V$,

    \begin{align}
        \langle p^t(x), p^t(y) \rangle_{\ell^2(1/\pi)} &:= (p^t(x))^\top D^{-1} p^t(y)\\
        &= \Phi_t(x)^\top U^\top U \Phi_t(y)\\
        &= \Phi_t(x)^\top  \Phi_t(y)
    \end{align}
\end{proof}

As a corollary, for any $u,v \in V$, the diffusion distance $||p^t(u)- p^t(v)||_{\ell^2(1/\pi)} = ||\Phi_t(u)- \Phi_t(v)||_2$. For more about diffusion distance, see \cite{diffusionmaps} and \cite{pku}.

The following Lemma describes the approximation properties of a diffusion embedding with small embedding dimension $k \in [|V| - 1]$ (rather than the full $|V|$-dimensional embedding), when $\lambda_{k + 1}^{2t}$ is small.

\begin{lemma}\label{lem:approx-equiv}(Small Embedding Dimension Approximation; Lemma 2 of \cite{pku})
    Let $G = (V,E)$ be a connected, undirected graph. Let $(V,P)$ be a reversible Markov chain with transition matrix $P \in [0,1]^{|V| \times |V|}$ and stationary measure $\pi$. Let $\pimin := \min_{v \in V} \pi(v)$ and $1 =\lambda_1 \geq |\lambda_2| \geq \ldots, \geq |\lambda_{|V|}| \geq 0$ be arranged in decreasing magnitude. Given $t \in \N$ and $k \in [|V| - 1]$,  for any $u,v \in V$,

    \begin{align}
        ||\Phi_t(u) - \Phi_t(v)||_2^2 - \frac{2\lambda_{k + 1}^{2t}}{\pimin} \cdot 1[u \neq v] &\leq ||\Phi^{\leq k}_t(u) - \Phi^{\leq k}_t(v)||_2^2 \leq ||\Phi_t(u) - \Phi_t(v)||_2^2\label{eq:dim-k-embd-approx-1}\\
        ||\Phi_t(v)||_2^2 - \frac{\lambda_{k + 1}^{2t}}{\pimin} &\leq ||\Phi_t^{\leq k}(v)||_2^2 \leq ||\Phi_t(v)||_2^2\label{eq:dim-k-embd-approx-2}\\
        \langle\Phi_t(u), \Phi_t(v) \rangle - \frac{\lambda_{k + 1}^{2t}}{\pimin} &\leq \langle\Phi_t^{\leq k}(u), \Phi_t^{\leq k}(v) \rangle \leq \langle\Phi_t(u), \Phi_t(v) \rangle + \frac{\lambda_{k + 1}^{2t}}{\pimin}
    \end{align}
\end{lemma}

\begin{proof}
As in the proof of \Cref{lem:cont-learn-min-general}, with $D = \diag(\pi)$ and matrix $D^{1/2}PD^{-1/2} $ having eigendecomposition $ U \Lambda U^\top$, then  $P = (D^{-1/2} U) \Lambda (D^{1/2} U)^\top$.

The matrix of right eigenvectors $\Phi := [\phi_1, \ldots, \phi_n] =  D^{-\frac{1}{2}}U$. Thus,

\begin{align}
    \Phi \Phi^\top &= D^{-\frac{1}{2}} UU^\top D^{-\frac{1}{2}}\\
    &= D^{-1}
\end{align}

Denoting $n  := |V|$, for any $u,v \in V$,

\begin{align}
    \sum_{i = 1}^n \phi_i(u) \phi_i(v) &= \frac{1[u = v]}{\pi(v)}\\
    \sum_{i=1}^n \bigl(\phi_i(u)-\phi_i(v)\bigr)^2 &=
\begin{cases}
0, & u=v,\\[1mm]
\frac{1}{\pi(u)}+\frac{1}{\pi(v)}, & u\neq v,
\end{cases}
\le \frac{2}{\pi_{\min}}\mathbf{1}_{\{u\neq v\}}.
\end{align}

Now
\begin{align}
||\Phi_t^{\leq k}(u)-\Phi_t^{\leq k}(v)||_2^2
&=
||\Phi_t(u)-\Phi_t(v)||_2^2
-
\sum_{i \geq k + 1}
\lambda_i^{2t}\bigl(\phi_i(u)-\phi_i(v)\bigr)^2 \\
&\ge
||\Phi_t(u)-\Phi_t(v)||_2^2
-
\lambda_{k + 1}^{2t} \sum_{i=1}^n \bigl(\phi_i(u)-\phi_i(v)\bigr)^2 \\
&\ge
||\Phi_t(u)-\Phi_t(v)||_2^2
-
\frac{2\lambda_{k + 1}^{2t}}{\pi_{\min}}\mathbf{1}_{\{u\neq v\}}.
\end{align}
The upper bound $||\Phi_t^{\leq k}(u)-\Phi_t^{\leq k}(v)||_2^2 \le ||\Phi_t(u)-\Phi_t(v)||_2^2$ is immediate, since truncation only removes nonnegative terms. This proves \Cref{eq:dim-k-embd-approx-1}.

Similarly,
\begin{align}
||\Phi_t^{\leq k}(v)||_2^2
&=
||\Phi_t(v)||_2^2
-
\sum_{i\geq k + 1} \lambda_i^{2t}\phi_i(v)^2 \\
&\ge
||\Phi_t(v)||_2^2
-
\lambda_{k + 1}^{2t} \sum_{i=1}^n \phi_i(v)^2 \\
&\ge
||\Phi_t(v)||_2^2-\frac{\lambda_{k + 1}^{2t}}{\pi_{\min}}.
\end{align}
Again, since truncation only removes nonnegative terms, $||\Phi_t^{\leq k}(v)||_2^2 \le ||\Phi_t(v)||_2^2.$ This proves \Cref{eq:dim-k-embd-approx-2}.

Finally,

\begin{align}
    \langle\Phi_t^{\leq k}(u), \Phi_t^{\leq k}(v) \rangle &= \langle\Phi_t(u), \Phi_t(v) \rangle - \sum_{i \geq k + 1} \lambda_i^{2t}\phi_i(u) \phi_i(v)\\
    &\geq \langle\Phi_t(u), \Phi_t(v) \rangle -  \sqrt{\Big(\sum_{i \geq k + 1}\lambda_{i}^{2t}  \phi_i(u)^2 \Big) \Big(\sum_{i \geq k + 1 } \lambda_{i}^{2t} \phi_i(v)^2 \Big) } \tag{Cauchy Schwartz}\\
    &\geq \langle\Phi_t(u), \Phi_t(v) \rangle - \lambda_{k + 1}^{2t} \sqrt{\Big(\sum_{i = 1}^n \phi_i(u)^2 \Big) \Big(\sum_{i = 1}^n \phi_i(v)^2 \Big) } \\
    &\geq \langle\Phi_t(u), \Phi_t(v) \rangle - \frac{\lambda_{k + 1}^{2t}}{\pimin}\\
    \textup{Similarly, } \langle\Phi_t^{\leq k}(u), \Phi_t^{\leq k}(v) \rangle &\leq \langle\Phi_t(u), \Phi_t(v) \rangle + \frac{\lambda_{k + 1}^{2t}}{\pimin}
\end{align}

\end{proof}

We will now describe how to learn the embeddings $\{ \Phi_t^{\leq k}(u)\}_{u \in V}$ for general $k \in [|V|]$ and $t \in \N$. \Cref{lem:cont-learn-min-general} shows the global minimizers of the contrastive loss recover $\{ \Phi^{\leq k}_t(u)\}_{u \in V}$ up to an orthogonal transformation.

\begin{restatable}[Global Minimizers for Contrastive Loss]{lemma}{contlossmin} \label{lem:cont-learn-min-general}
Let $G = (V,E)$ be any connected, undirected graph. Let $(V,P)$ be a reversible Markov chain with transition matrix $P \in [0,1]^{|V| \times |V|}$ and stationary measure $\pi \in \Delta(V)$. Let $\{ \lambda_i(P)\}_{i \in [|V|]}$ be the eigenvalues of $P$, ordered by decreasing absolute value. Fix any $t \in \N$ and $k \in [|V|]$. Suppose $|\lambda_k(P)| - |\lambda_{k + 1}(P)| > 0$ if $k < |V|$ and suppose the population contrastive loss (\Cref{eq:cont-loss-defn}) with base distribution $\pi$, augmentation function $P^t$, and embedding dimension $k$ has a global minimizer in $\Femb$. Then, for any such global minimizer $\Femb \ni h_{\theta_*} : V \to \R^k$, there exists an orthogonal matrix $R_* \in O^{k \times k}$ where

\begin{align}
    \forall u \in V,  h_{\theta_*}(u)  = R_* \, \Phi^{\leq k}_t(u) \in \R^{k}
\end{align}

With mapping $C : V\to \Delta(V)$ defined as the sampler-representation of $P$, where $\forall x \in V, C(x) = P[x]$, then the positive and negative pair distribution corresponding to base distribution $\pi$ and augmentation function $P^t$ are where positive pairs $(x,y)$ are generated by sampling $z \sim \pi,$ then $ x \sim C^t(z), y \sim C^{t}(z)$ independently. Negative pairs $(x,y)$ are generated via $x,y \sim \pi$ independently. Note that $\sigma_x$ (defined in \Cref{eq:wx-defn}) equals $\pi(x)$ for every $x \in V$. 
\end{restatable}

\begin{proof}
    Fix $t \in \N$ and $k \in [|V|]$. The plan is to follow the derivation process of the contrastive loss $\Lc$, but with $\mu = \pi$ and $Q = P^t$. First, we have

    \begin{align}
        A &= Q^\top \diag(\mu)Q\\
        &= (P^t)^\top \diag(\pi) P^t\\
        &= \diag(\pi) P^{2t} \tag{By reversibility of $P$}\\
        D &:= \diag(A \1)\\
        &= \diag(\pi) \tag{Row sums of $P^{2t}$ are 1}\\
        \overline{A} &:= D^{-1/2} A D^{-1/2}\\
        &= \diag(\pi)^{1/2} P^{2t} \diag(\pi)^{-1/2}
    \end{align}

    Let the eigendecomposition of $\overline{A}$ be $\overline{A} = \sum_{i \in [|V|]} \lambda_i(\overline{A})\, u_i u_i^\top = U\Lambda U^\top$. By the Eckart-Young-Mirsky Theorem \cite{eckartyoung}, since $\overline{A}$ is p.s.d. and $\lambda_k(\overline{A}) - \lambda_{k + 1}(\overline{A}) = \lambda_k(P)^{2t} - \lambda_{k + 1}(P)^{2t} > 0$ if $k < |V|$, the rank-$k$ minimizers of the matrix factorization loss are:

    \begin{align}
        \{ U\Lambda^{1/2}R : R \in O^{k \times k}\} = \arg\min_{F \in \R^{|V| \times k}} || \overline{A} - FF^\top||_F^2
    \end{align}

    Because for any $h_\theta \in \Femb$, $\Lmf (h_\theta ) = \Lc(h_\theta) + c$ for some constant $c$ independent of $h_\theta$ (\Cref{eq:cont-mf-equiv}), then if $\Femb$ contains some $h_{\theta_*} : V \to \R^k$ which is a global minimizer of $\Lc$, then the induced matrix $F_{\theta_*} \in \R^{|V| \times k}$, where the $x$-th row is:

    \begin{align}
        \forall x \in V, F_{\theta_*}[x] &= \sqrt{\pi(x)} h_{\theta_*}(x)^\top
    \end{align}

    Must be in the set $\arg\min_{F \in \R^{|V| \times k}} || \overline{A} - FF^\top||_F^2 = \{ U\Lambda^{1/2}R : R \in O^{k \times k}\}$. In particular, there exists $R \in O^{k \times k}$ where

    \begin{align}
        F_{\theta_*} = U\Lambda^{1/2} R
    \end{align}

    The $x$-th row of $U\Lambda^{1/2} R$ is $(\lambda_{i}(\overline{A})^{1/2} u_i(x))_{i \in [k]}^\top R \in \R^{1 \times k}$, where $u_i(x)$ is the $x$-th entry of $u_i$. Thus,

    \begin{align}
        \forall x \in V, \, h_{\theta_*}(x) &= (\pi(x))^{-1/2} R \, (\lambda_{i}(\overline{A})^{1/2} u_i(x))_{i \in [k]}\\
        &= R \, \Big(  \lambda_{i}(\overline{A})^{1/2} (\pi(x))^{-1/2}u_i(x)\Big)_{i \in [k]}\label{eq:spec-embd-equiv}
    \end{align}

    Finally, we note that since $ \{ (\lambda_{i}(\overline{A}), u_i) \}_{i \in [k]}$ are the top $k$ eigenvalue and eigenvectors of $\overline{A} = \diag(\pi)^{1/2}P^{2t} \diag(\pi)^{-1/2}$, then the eigendecomposition of $P^{2t}$ is:

    \begin{align}
        P^{2t} &= \sum_{i \in [|V|]}\lambda_{i}(\overline{A}) \big(\diag(\pi)^{-1/2} u_i\big) \big(\diag(\pi)^{1/2} u_i\big)^\top\\
        &= \sum_{i \in [|V|]} (\lambda_i(P))^{2t} \big(\diag(\pi)^{-1/2} u_i\big) \big(\diag(\pi)^{1/2} u_i\big)^\top
    \end{align}

    Where $\{ \lambda_i(P)\}_{i \in [|V|]}$ are the eigenvalues of $P$ arranged in decreasing order of absolute value, so that $\lambda_{i}(\overline{A})^{1/2} = |\lambda_i(P)|^t$. Finally, we see that the right eigenvectors of $P^{2t}$ are exactly $\{ \diag(\pi)^{-1/2} u_i \}_{i \in [|V|]}$. Thus, \Cref{eq:spec-embd-equiv} can be written in terms of the $k$ truncated spectral diffusion embedding in \Cref{defn:spec-diff-embd}.

    \begin{align}
        \forall x \in V, \, h_{\theta_*}(x) &=  R \, \Big(  \lambda_{i}(\overline{A})^{1/2} (\pi(x))^{-1/2}u_i(x)\Big)_{i \in [k]}\\
        &= R \, \diag(\textup{sign}(\{ \lambda_i(P)^t\}_{i \in [k]})) \, \Phi_t^{\leq k}(x)\\
        &= R_* \, \Phi_t^{\leq k}(x)
    \end{align}

    Where $\diag(\textup{sign}(\{ \lambda_i(P)^t\}_{i \in [k]})) \in O^{k \times k}$ is a diagonal matrix where  the $(i,i)$-th entry is $1$ if $\lambda_i(P)^t$ is non-negative, and $-1$ if $\lambda_i(P)^t$ is negative, for $i \in [k]$. Since $R,\, \diag(\textup{sign}(\{ \lambda_i(P)^t\}_{i \in [k]})) \in O^{k \times k}$, then $R_* := R \, \diag(\textup{sign}(\{ \lambda_i(P)^t\}_{i \in [k]})) \in O^{k \times k}$.
    
\end{proof}

\begin{corollary}\label{lem:spec-embd-k} 
    Let $G = (V,E)$ be any connected, undirected graph. Let $(V,P)$ be a reversible Markov chain with transition matrix $P \in [0,1]^{|V| \times |V|}$ and stationary measure $\pi$. Fix any $t \in \N$ and $k \in [|V|]$. Let $\{ \lambda_i(P)\}_{i \in [|V|]}$ be the eigenvalues of $P$, ordered by decreasing absolute value. Suppose $|\lambda_k(P)| - |\lambda_{k + 1}(P)| > 0$ if $k < |V|$. Denote $\Femb^{(k)}\subset \F_V^{(k)}$ as a class of embedding models from $V$ to $\R^k$ and $\Femb^{(|V|)} \subset \F_V^{(|V|)}$ a class of embedding models from $V$ to $\R^{|V|}$.  Suppose the population contrastive loss (\Cref{eq:cont-loss-defn}) with base distribution $\pi$, augmentation function $P^t$, and embedding dimension $k$ (dimension $|V|$) has a global minimizer in $\Femb^{(k)}$ ($\Femb^{(|V|)}$).

    Let $\theta_*$ be any global minimizer of the population contrastive loss with embedding dimension $|V|$, and  $\theta_*^{(k)}$ be any global minimizer of the population contrastive loss with embedding dimension with embedding dimension $k$. There exist orthogonal matrices $R \in O^{|V| \times |V|}$ and $R^{(k)} \in O^{k \times k}$ such that $\forall u \in V,$

    \begin{align}
        h_{\theta_*^{(k)}}(u)  &= R^{(k)}\, \Phi^{\leq k}_t(u) \in \R^{k} \label{eq:global-min-Gx-1}\\
        h_{\theta_*}(u)  &= R \,\Phi_t(u) \in \R^{|V|} \label{eq:global-min-Gx-2}
    \end{align}

    In particular, $\forall u, v \in V, $ with $\pimin := \min_{v \in V} \pi(v)$,

    \begin{align}
        ||h_{\theta_*}(u) - h_{\theta_*}(v)||^2 - \frac{2(\lambda_{k + 1}(P))^{2t}}{\pimin} &= ||p^t(u) - p^t(v)||^2_{\ell^2(1/\pi)}  - \frac{2(\lambda_{k + 1}(P))^{2t}}{\pimin}\\
        &\leq ||h_{\theta_*^{(k)}}(u)  - h_{\theta_*^{(k)}}(v) ||^2_2 = ||\Phi^{\leq k}_t(u) - \Phi^{\leq k}_t(v)||^2_2\\
        &\leq  ||p^t(u) - p^t(v)||^2_{\ell^2(1/\pi)}\label{eq:diff-dist-difference-eq} = ||h_{\theta_*}(u) - h_{\theta_*}(v)||^2\\
        ||h_{\theta_*}(v)||_2^2 - \frac{(\lambda_{k + 1}(P))^{2t}}{\pimin} &\leq ||h_{\theta_*^{k_*}}(v)||_2^2 \leq ||h_{\theta_*}(v)||_2^2\label{eq:diff-dist-norm-eq} 
    \end{align}

\end{corollary}

\begin{proof}
    We apply \Cref{lem:cont-learn-min-general} with  graph $G $ and with embedding dimension $k$ (resp. embedding dimension  $|V|)$ to get \Cref{eq:global-min-Gx-1} (resp. \Cref{eq:global-min-Gx-2}). \Cref{eq:diff-dist-difference-eq} is obtained by combining  \Cref{eq:global-min-Gx-1},  \Cref{eq:global-min-Gx-2}, \Cref{lem:diff-dist}, and \Cref{lem:approx-equiv}. \Cref{eq:diff-dist-norm-eq} is obtained by combining \Cref{lem:approx-equiv} with \Cref{eq:global-min-Gx-1} and \Cref{eq:global-min-Gx-2}.
\end{proof}

\subsubsection{Matrix Factorization Lemmas}

\Cref{lem:diff-dist} and \Cref{lem:cont-learn-min-general} motivate the idea of using embeddings learned via contrastive learning to compute the diffusion similarity and distance. \Cref{lem:contrastive-learning} shows that finding the ERM (or $\frac{\eps}{2}$-ERM) suffices, with high probability, to find an $\eps$-minimizer of the population contrastive loss $\Lc$. \Cref{lem:eps_close_to_opt}, a Lemma about matrix-factorization, is useful in quantifying how close an embedding corresponding to such an $\eps$-minimizer is to a ground-truth embedding, up to a rotation.

\input{sections/appendix/complex_proof}


\Cref{lem:closeness} is a simple Lemma that converts the guarantee from \Cref{lem:eps_close_to_opt} into one about the embedding functions parameterizing matrices $F_\theta$ and $F_*$.

\begin{lemma}\label{lem:closeness} Let $n \geq k$ be natural numbers, and matrices $F_\theta, F_* \in \R^{n \times k}, R \in O^{k \times k}$ to where $\mathrm{dist}(F_\theta, F_*) = ||F_\theta - F_* R||_F$. Suppose $\pi \in \Delta([n])$, $h_\theta : [n] \to \R^k, h_{\theta_*} : [n] \to \R^k$ parameterize $F_\theta, F_*$ as follows. For all $x \in [n]$,

\begin{align}
    F_\theta[x] &= \sqrt{\pi(x)} h_\theta(x)\\
    F_*[x] &= \sqrt{\pi(x)} h_{\theta_*}(x)
\end{align}
Then, $\mathrm{dist}(F_\theta, F_*)^2 = \ex_{x \sim \pi} ||h_\theta - R\, h_{\theta_*}||_2^2$.
\end{lemma}

\begin{proof}
    \begin{align}
        \mathrm{dist}(F_\theta, F_*)^2 &= ||F_\theta - F_* R||_F^2\\
        &= \sum_{x \in [n]} \pi(x) ||h_\theta - R\, h_{\theta_*}||_2^2\\
        &= \ex_{x \sim \pi} ||h_\theta - R\, h_{\theta_*}||_2^2
    \end{align}
\end{proof}

\subsubsection{Learning Contrastive Embeddings on Local Subgraphs}\label{appen:local-subgraphs}

In addition, we would like the embeddings to be learned locally, not for the entire graph simultaneously. For a theorem graph $G = (V,E)$, define the following notion of a subgraph of $G$ supported on the $T$-neighborhood of a node $x \in V$. In the following, we will denote $(C,P,w) := (C_w, P_w, w)$ as a reversible conjecturer satisfying \Cref{defn:conjecturer}, to simplify notation.

\localsubgappen

On a $T = \tau + 1$ local subgraph centered at $x \in V$, the $\tau$-step random walk distribution of $x$ and any neighbor $y \in N(x)$ is the same as that on the entire graph. 

\begin{lemma}\label{lem:local-to-global}
     Let $G = (V,E)$ be a connected, undirected graph with reversible, nearest-neighbor-type Markov chain $(V,P)$. Let $T = \tau + 1$. Let $P_{x,T}$ be the restriction of $P$ to $G_{x,T}$. For any $x \in V$ and any $y \in N(x)$, we have  $p^\tau_{G_{x, T}}(x) = p^\tau(x)$ and $p^\tau_{G_{x, T}}(y) = p^\tau(y)$, where $p^\tau(x), p^\tau(y)$ is the usual scale-$\tau$ diffusion embeddings defined in \Cref{defn:diffusionembd}.
\end{lemma}

\begin{proof}
    Since $G_{x, \tau + 1}$ is restricted to $B(x, \tau + 1)$, then $G_{x, \tau + 1}$ and $G$ look identical to a nearest-neighbor-type random walk starting at $x$ or any $y \in N(x)$ of length at most $\tau$. 
\end{proof}

\Cref{lem:local-to-global} shows the scale-$\tau$ diffusion embedding can be computed locally, in a $(\tau+1)$ neighborhood around the queried node. In addition, for any reversible Markov chain $(V,P)$ on the entire graph $G$ with stationary measure $\pi$, for any $x \in V$ and $T \in \N$, the stationary measure $\pi_{x,T}$ is equal to $\pi$ restricted to $B(x,T)$, up to a multiplicative constant, by the self-loop construction of $G_{x,T}$. Thus, the diffusion similarity and distance computed locally differs from that computed globally by a multiplicative constant, which means the locally-computed diffusion similarity is sufficient for solving the optimization problem in \Cref{eq:optim-problem-diversity}.

To learn the spectral diffusion embedding for reversible Markov chain $(V,P)$, of conjecturer $(C,P,w)$ in a local $T$-neighborhood around a node $x \in V$ via \Cref{defn:contrastive-learning}, one will require access to the following oracles, defined in \Cref{defn:T-neighborhood}.

\begin{enumerate}
    \item The push-forward map $C_{x,T} : B(x,T) \to \Delta(B(x,T))$ of the restricted Markov chain $(V_{x,T}, P_{x,T})$, where $\forall y \in B(x,T)$, $C_{x,T}(y) = P_{x,T}[y] \in \Delta(B(x,T))$,
    \item i.i.d. sample access to $\pi_{x,T} \in \Delta(B(x,T))$.
\end{enumerate}

For $(V_{x,T}, P_{x,T})$, it suffices to have access to the original conjecturer $(C,P,w) := (C_w, P_w, w)$ on $G$ and the neighbor oracle $N_{G_{x,T}} : V \to 2^V$ of the local graph $G_{x,T}$ supported on $B(x, T)$. This is because for any $y \in B(x,T)$, to sample $z \sim C_{x,T}(y)$, one can first sample $z' \sim C(y)$. Then, if $z' \in N_{G_{x,T}}(y)$, set $z = z'$. Else, set $z = y$. $z$ will be a sample from $C_{x,T}(y)$ with the correct distribution.

\Cref{alg:compute-pix} describes how to compute $N_{G_{x,T}} : V \to 2^V$ and $\pi_{x,T}$. It involves first performing a $T$-step BFS starting from node $x$. Once the set of all theorems in the $T$-neighborhood of $x$ is collected, it is straightforward to create an oracle for $N_{G_{x,T}}$. Regarding the sampler for $\pi_{x,T}$, because $\forall y \in B(x,T), \pi_{x,T}(y) \propto \pi(y) \propto w(y)$, then $\pi_{x,T}(y) = \frac{w(y)}{\sum_{z \in B(x,T)} w(z)}$. It suffices to compute the vector $\pi_{x,T}$ for all $y \in B(x,T)$. 

\begin{algorithm}[t]
\caption{Compute Local Neighbor Oracle and Stationary Measure}
\label{alg:compute-pix}
\begin{algorithmic}[1]
\Require $T\in\N$, source theorem $x\in V$, neighbor oracle $N:V\to 2^V$, and conjecturer $(C,P,w) := (C_w, P_w, w)$
\Ensure Local neighbor oracle $N_{G_{x,T}} : B(x,T) \to 2^{B(x,T)}$ and local stationary measure $\pi_{x,T} \in \Delta(B(x,T))$

\State Set $B_{-1}\gets \emptyset$ and $B_0\gets \{x\}$

\For{$s=1,\ldots,T$}
    \State Set $\partial B_{s-1}\gets B_{s-1}\setminus B_{s-2}$

    \State Set $B_s \gets B_{s-1} \cup \bigcup_{v\in \partial B_{s-1}} N(v).$

\EndFor

\State Set $S\gets B_T$

\State \State Define $N_{G_{x,T}}$ by the following subroutine:
\Indent\Oracle{$N_{G_{x,T}}$}
    \State \textbf{Input:} $y \in S$
    \State Query $N(y)$
    \State Set $m(y) \gets |N(y)\cap S|$
    \State \Return $N(y)\cap S$ together with $\deg(y)-m(y)$ copies of $y$
\EndOracle
\EndIndent

\State Set
\[
    \pi_{x,T}(y)
    \gets
    \frac{w(y)}{\sum_{z\in S} w(z)}
    \qquad
    \text{for all } y\in S.
\]

\State \Return $N_{G_{x,T}}$ and $\pi_{x,T}$
\end{algorithmic}
\end{algorithm}

\subsection{Main Theorems}\label{appen:contembd}

\contembd*

\begin{proof}
    By \Cref{lem:cont-learn-min-general},  there exists an orthogonal matrix $R \in O^{|V| \times |V|}$ where

\begin{align}
    \forall u \in V,  h_{\theta_*}(u)  = R \, \Phi_\tau(u) \in \R^{|V|}
\end{align}
Thus,
\begin{align}
    \forall u, v \in V,  \langle h_{\theta_*}(u) , h_{\theta_*}(v)  \rangle  &= \langle R \, \Phi_\tau(u) , R \, \Phi_\tau(v) \rangle\\
    &= \langle  \Phi_\tau(u) ,\Phi_\tau(v) \rangle\\
    &= \langle p^\tau(u), p^\tau(v) \rangle_{\ell^2(1/\pi)} \tag{By \Cref{lem:diff-dist}}
\end{align}

\end{proof}

Below is our end-to-end learning result, which can be viewed as a finite-sample, small-embedding dimension version of \Cref{thm:contembd}.

\etethmappen

\begin{proof}
    By \Cref{lem:contrastive-learning}, with probability at least $1 - \perm$,

    \begin{align}
        \Lc(h_{\hat\theta}) &\leq  \Lc(h_{\theta_*}) + \eps\\
        ||\overline{A} - F_{\hat\theta} F_{\hat\theta}^\top||^2_F &\leq ||\overline{A} - F_{\theta_*} F_{\theta_*}^\top||^2_F + \eps
    \end{align}

    Under that event, since $\rank(F_{\hat\theta}) = k$ by assumption, by \Cref{lem:eps_close_to_opt} and \Cref{lem:closeness}, 

    \begin{align}
        \ex_{x \sim \pi } ||h_{\hat\theta}(x) - R \, h_{\theta_*}(x)||_2^2 \leq \eps \cdot \frac{\Big(1 + 4\sqrt{\frac{2}{\lambda_k^2 - \lambda_{k + 1}^2}}\Big)^2}{2 (\sqrt{2} - 1)\lambda_k}
    \end{align}

    Where we've used that $||\overline{A}||_2 = \lambda_1(\overline{A}) = \lambda_1(P^{2\tau}) =  1$. By Markov's inequality,

    \begin{align}
        \Pro_{x \sim \pi } \Big[ ||h_{\hat\theta}(x) - R \, h_{\theta_*}(x)||_2^2 \geq \Delta \Big] \leq \frac{\eps}{\Delta} \cdot \frac{\Big(1 + 4\sqrt{\frac{2}{\lambda_k^2 - \lambda_{k + 1}^2}}\Big)^2}{2 (\sqrt{2} - 1)\lambda_k}
    \end{align}

    In fact, using the premise that $\exists \beta > 0 \, : \forall y \in V, \, \sum_{x \in N(y)} \pi(x) \leq \beta \, \pi(y)$, we can extend this closeness guarantee to a neighborhood around $x \sim \pi$.

    \begin{align}
        &\Pro_{x \sim \pi } \Big[ \exists y \in N(x) : ||h_{\hat\theta}(y) - R \,h_{\theta_*}(y)||_2^2 \geq \Delta \Big]\\
        &\leq \ex_{x \sim \pi} \sum_{y \in N(x)} 1\Big[||h_{\hat\theta}(y) - R \,h_{\theta_*}(y)||_2^2 \geq \Delta\Big]\\
        &= \sum_{x \in V} \pi(x) \sum_{y \in N(x)} 1\Big[||h_{\hat\theta}(y) - R \,h_{\theta_*}(y)||_2^2 \geq \Delta\Big]\\
        &\leq \sum_{y \in V} \beta \, \pi(y)\,  1\Big[||h_{\hat\theta}(y) - R \,h_{\theta_*}(y)||_2^2 \geq \Delta\Big]\\
        &=   \beta \cdot \Pro_{y \sim \pi} \Big[||h_{\hat\theta}(y) - R \,h_{\theta_*}(y)||_2^2 \geq \Delta\Big]\\
        &\leq  \beta \cdot \frac{\eps}{\Delta} \cdot \frac{\Big(1 + 4\sqrt{\frac{2}{\lambda_k^2 - \lambda_{k + 1}^2}}\Big)^2}{2 (\sqrt{2} - 1)\lambda_k}
    \end{align}

    Thus,

    \begin{align}
        &\Pro_{x \sim \pi } \Big[ \exists y \in N(x) : ||h_{\hat\theta}(y) - R \,h_{\theta_*}(y)||_2^2 \geq \Delta \textup{ or } ||h_{\hat\theta}(x) - R \,h_{\theta_*}(x)||_2^2 \geq \Delta \Big]\\
        &\leq \frac{\eps }{\Delta} \cdot \frac{(\beta + 1)\cdot \Big(1 + 4\sqrt{\frac{2}{\lambda_k^2 - \lambda_{k + 1}^2}}\Big)^2}{2 (\sqrt{2} - 1)\lambda_k}
    \end{align}
    
    Under the event that $x\sim \pi$ satisfies 

    \begin{align}
        \forall y \in N(x) : ||h_{\hat\theta}(y) - R \,h_{\theta_*}(y)||_2^2 \leq \Delta \textup{ and } ||h_{\hat\theta}(x) - R \,h_{\theta_*}(x)||_2^2 \leq \Delta\label{eq:good-case-cont-learning}
    \end{align}

    Then, $\forall y, z \in N(x) \cup \{ x\}$, 

    \begin{align}
        \langle h_{\hat\theta}(y), h_{\hat\theta}(z)\rangle  &\leq \langle R \,h_{\theta_*}(y), R \,h_{\theta_*}(z) \rangle  + 2 \max_{h_\theta \in \Femb}\max_{\hat{y} \in V} ||h_{\theta}(\hat{y})||_2 \cdot \sqrt{\Delta} + \Delta\\
        &\leq \langle  h_{\theta_*}(y),  h_{\theta_*}(z) \rangle  + 2 \chi \sqrt{k} \cdot \sqrt{\Delta} + \Delta\\
        \textup{Similarly, } \langle h_{\hat\theta}(y), h_{\hat\theta}(z)\rangle &\geq \langle  h_{\theta_*}(y),  h_{\theta_*}(z) \rangle  - 2 \chi \sqrt{k} \cdot \sqrt{\Delta} - \Delta
    \end{align}

    By \Cref{lem:cont-learn-min-general} (which says $\langle h_{\theta_*}(x), h_{\theta_*}(y) \rangle = \langle\Phi_t^{\leq k}(x), \Phi_t^{\leq k}(y) \rangle$) and \Cref{lem:approx-equiv}, then $\forall x, y \in V$,

    \begin{align}
        \langle\Phi_t(x), \Phi_t(y) \rangle - \frac{\lambda_{k + 1}}{\pimin} &\leq \langle h_{\theta_*}(x), h_{\theta_*}(y) \rangle \leq \langle\Phi_t(x), \Phi_t(y) \rangle + \frac{\lambda_{k + 1}}{\pimin}
    \end{align}

     It follows that when $x \sim \pi$ satisfies \Cref{eq:good-case-cont-learning}, then with $\xi := \frac{\lambda_{k + 1}}{\pimin} + 2 \chi \sqrt{k} \cdot \sqrt{\Delta} + \Delta$, we have

     \begin{align}
        \langle\Phi_t(y), \Phi_t(z) \rangle - \xi  &\leq \langle h_{\hat\theta}(y), h_{\hat\theta}(z) \rangle \leq \langle\Phi_t(y), \Phi_t(z) \rangle + \xi
    \end{align}

    Since by \Cref{lem:diff-dist} we have $\diffsim(y,z) = \langle\Phi_t(y), \Phi_t(z) \rangle$, then with $\xi := \frac{\lambda_{k + 1}}{\pimin} + 2 \chi \sqrt{k} \cdot \sqrt{\Delta} + \Delta$, we have

    \begin{align}
        &\Pro_{x \sim \pi}\Big [\exists y,z \in N(x) \cup \{x\} \, : \, |\diffsim(y,z) -  \langle h_{\hat\theta}(y), h_{\hat\theta}(z) \rangle| > \xi \Big]\\
        &\leq \frac{\eps }{\Delta} \cdot \frac{(\beta + 1)\cdot \Big(1 + 4\sqrt{\frac{2}{\lambda_k^2 - \lambda_{k + 1}^2}}\Big)^2}{2 (\sqrt{2} - 1)\lambda_k}
    \end{align}

In short, for sufficiently small $\eps$, the learned embedding model $h_{\hat{\theta}} : V \to \R^k$ can be used to approximate $\diffsim(y,z)$ of neighboring points of \textit{most} $x \in V$ by the inner product, $\langle h_{\hat\theta}(y), h_{\hat\theta}(z) \rangle$.

\end{proof}

%% file: sections/appendix/complex_proof.tex
To start, \Cref{lem:projector-diff-ineq}, \Cref{lem:kyfan-max-principle}, and \Cref{lem:procrustes-bound} are introduced here to prove \Cref{lem:eps_close_to_opt}. Note that for two matrices $A, B \in \R^{n \times n}, \langle A,B \rangle := \tr(A^\top B)$ denotes the Hilbert-Schmidt inner product of the two matrices, with $||A||_F^2 = \langle A, A\rangle$.

\begin{lemma}\label{lem:procrustes-bound} (Lemma 5.4 of \cite{tu2016lowranksolutionslinearmatrix})
For any natural numbers $n \geq k$ and matrices $U, X\in\mathbb{R}^{n\times k}$,
\begin{align}
\dist(U,X)^2
:=\min_{R\in O^{k \times k}}||U-XR||_F^2
\le
\frac{1}{2(\sqrt2-1)\sigma_k(X)^2}||UU^\top-XX^\top||_F^2.
\end{align}
\end{lemma}

\begin{lemma}\label{lem:kyfan-max-principle} (Ky Fan Maximum Principle \cite{kyfan}; Theorem 1 of \cite{overtonkyfan})
    Let $A$ be a real, symmetric $n \times n$ matrix, with eigenvalues $\lambda_1 \geq \ldots \geq \lambda_n$. For any $k \in [n]$,

    \begin{align}
        \max_{X^\top X = I_k} \tr(X^\top A X) = \sum_{i = 1}^k \lambda_i
    \end{align}
\end{lemma}

\begin{lemma}\label{lem:projector-diff-ineq} (Lemma 3.1 of \cite{curvature})
    Let $A \in R^{n \times n}$ be a symmetric matrix, with eigenvalues ordered as $|\lambda_1| \geq |\lambda_2|\geq \ldots \geq |\lambda_n|$. For $k \in [n - 1]$, let $P_k$ be the projector onto the subspace spanned by the top-$k$ eigenvectors of $A$, denoted $\lambda_1, \ldots, \lambda_k$. Assume that $\lambda_k^2 - \lambda_{k + 1}^2 > 0$. For any symmetric rank-$k$ projector $P \in \R^{n \times n}$, $P^\top = P$,  $P^2 = P$, $\tr(P) = k$, we have:
    
    \begin{align}
        \langle A^2,P_k-P\rangle \ge \frac{\lambda_k^2-\lambda_{k+1}^2}{2}||P_k-P||_F^2. \label{eq:curvature_proj}
    \end{align}
\end{lemma}

\begin{proof}

Diagonalize $A^2 = U\diag(\{\lambda_i^2\}_{i \in [n]})U^\top$
and write $\widetilde P:=U^\top P U$.
Partition $\widetilde P$ as
\begin{align}
\widetilde P =
\begin{pmatrix}
P_{11} & P_{12}\\
P_{21} & P_{22}
\end{pmatrix}
\quad\text{with } P_{11}\in\mathbb{R}^{k\times k}.
\end{align}
Then $\widetilde P_k = U^\top P_k U =\diag(I_k,0)$ in the eigenbasis, and
\begin{align}
\langle A^2,P_k-P\rangle &= \tr\big(\diag(\{\lambda_i^2\}_{i \in [n]})(\widetilde P_k-\widetilde P)\big)\\
&= \tr\big(\Lambda_k^2(I-P_{11})\big) - \tr\big(\Lambda_\perp^2 P_{22}\big)
\end{align}
where $\Lambda_\perp:=\diag(\{\lambda_{k+1},\dots,\lambda_n\})$.
Using $\Lambda_k^2\succeq \lambda_k^2 I_k$ and $\Lambda_\perp^2\preceq \lambda_{k+1}^2 I$ gives
\begin{align}
\langle A^2,P_k-P\rangle
\ge
\lambda_k^2\tr(I-P_{11}) - \lambda_{k+1}^2\tr(P_{22}).
\end{align}
Since $\widetilde P$ is a rank-$k$ projector, $\tr(\widetilde P)=k$, so
$\tr(P_{11})+\tr(P_{22})=k$ and thus $\tr(I-P_{11})=\tr(P_{22})$.
Also, for rank-$k$ projectors,
\begin{align}
||P_k-P||_F^2 &=\tr\big((P_k-P)^2\big)\\
&=\tr(P_k)+\tr(P)-2\tr(P_kP)\\
&=2k-2\tr(P_{11})\\
&=2\tr(I-P_{11}).
\end{align}
Therefore $\tr(I-P_{11})=\frac12||P_k-P||_F^2$, and we obtain \Cref{eq:curvature_proj}.
\end{proof}

\begin{restatable}[Approximate minimizers are close to a global minimizer]{lemma}{matfactconj}\label{lem:eps_close_to_opt}
Let $A\in\mathbb{R}^{n\times n}$ be symmetric positive semidefinite with eigendecomposition
$A = U\diag(\{\lambda_i\}_{i \in [n]})U^\top$ where
$\lambda_1\ge \cdots \ge \lambda_n\ge 0$.
Fix $1\le k < n$ and assume $\lambda_k>0$ and the squared eigengap $\lambda_k^2-\lambda_{k+1}^2 > 0.$ Define, for matrix $F\in\mathbb{R}^{n\times k}$ of rank $k$,

\begin{align}
    \Lo(F) &:= ||A-FF^\top||_F^2
\end{align}

Suppose $F\in\mathbb{R}^{n\times k}$ satisfies $\Lo(F) \leq  \min_{\rank(Z)=k} \Lo(Z) + \eps$ for some $0<\eps\leq 1.$ Define the constant $\C(A,k):= \frac{\Big(1 + 4||A||_2\sqrt{\frac{2}{\lambda_k^2 - \lambda_{k + 1}^2}}\Big)^2}{2(\sqrt2-1)\lambda_k}$. Then, there exists a global minimizer $F^*\in\arg\min_{\rank(Z)=k} \Lo(Z)$ such that:
\begin{align}
    ||F-F^*||_F^2 &\leq \C(A,k) \,\eps
\end{align}
\end{restatable}

\begin{proof} First, we recall the Eckart-Young-Mirsky Theorem \cite{eckartyoung}, which characterizes the global minima of the loss function $\Lo(F)$. Let $U_k\in\mathbb{R}^{n\times k}$ be the first $k$ eigenvectors and let
$\Lambda_k:=\diag(\lambda_1,\dots,\lambda_k)$. Since $A\succeq 0$ and $\lambda_k^2 - \lambda_{k + 1}^2>0$, the best rank-$k$ approximation in Frobenius norm is $A_k := U_k\Lambda_k U_k^\top$, which is unique.
\begin{align}
    \min_{\rank(Z)=k} \Lo(Z) = ||A-A_k||_F^2 = \sum_{i>k}\lambda_i^2 = ||A||_F^2-||\Lambda_k||_F^2.
\end{align}

Because of this, the set of global minimizers of $\Lo(\cdot)$ over rank-$k$, $(n \times k)$ matrices are $S :=\{U_k\Lambda_k^{1/2}R : R\in O^{k \times k}\}$. With this, the proof plan is to upper bound $||FF^\top  - A_k||_F$, and then apply \Cref{lem:procrustes-bound} to get the final result. To bound, $||FF^\top  - A_k||_F$, define $P$ as the rank-$k$ projector onto $\textup{col}(F)$.

\begin{align}
    P &:= F(F^\top F)^{-1} F^\top \in \R^{n \times n}\\
    Q &:= F(F^\top F)^{-1/2}\in \R^{n \times k}
\end{align}

Where $P = QQ^\top.$ We will use the following decomposition, as a result of the triangle inequality.

\begin{align}
    ||FF^\top  - A_k||_F \leq ||FF^\top  - PAP||_F + ||PAP  - A_k||_F\label{eq:decomposition-of-gram-diff}
\end{align}

The remaining proof will be dedicated to bounding the terms $||FF^\top  - PAP||_F $ and $||PAP  - A_k||_F$ individually.
\medskip

\textbf{Bounding $||FF^\top  - PAP||_F$.} 

First, we will show that $||FF^\top  - PAP||_F \leq \sqrt{\eps}$. Towards this, we derive a useful identity for $\Lo(F)$, culminating in \Cref{eq:gap_decomp}. A direct expansion of $\Lo(F)$ gives the identity


\begin{align}
    \Lo(F) &= ||(FF^\top - PAP) + (PAP - A)||_F^2\\
    &= ||FF^\top - PAP||_F^2 + ||PAP - A||_F^2 + 2 \langle FF^\top - PAP, PAP - A\rangle\\
    &= ||FF^\top - PAP||_F^2 + ||PAP - A||_F^2
\end{align}

Where the last equality is because $\langle FF^\top - PAP, PAP - A\rangle = \tr( FF^\top PAP - FF^\top A) + \tr(-(PAP)^2 + PAPA) = 0$,  as $\tr( FF^\top PAP)= \tr (FF^\top A)$ and $\tr((PAP)^2) = \tr( PAPA)$. Subtracting $\min_{\rank(Z)=k} \Lo(Z)=||A||_F^2-||\Lambda_k||_F^2$ yields

\begin{align}
    \Lo(F) - \min_{\rank(Z)=k} \Lo(Z) &= ||FF^\top - PAP||_F^2 + ||PAP - A||_F^2 - (||A||_F^2-||\Lambda_k||_F^2)\\
    &= ||FF^\top - PAP||_F^2  + (||\Lambda_k||_F^2 + ||PAP||_F^2 - 2\tr(PAPA))\\
    &= ||FF^\top - PAP||_F^2  + (||\Lambda_k||_F^2 - ||Q^\top AQ||_F^2) \label{eq:gap_decomp}
\end{align}

Where in \Cref{eq:gap_decomp}, we used that $||Q^\top AQ||_F^2 = ||PAP||_F^2 = \tr(PAPA)$.


Because $\Lo(F) - \min_{\rank(Z)=k} \Lo(Z) \leq \eps$ by assumption, then by \Cref{eq:gap_decomp},  to show that $||FF^\top - PAP||_F^2 \leq \eps$, it suffices to show that  $||\Lambda_k||_F^2-||Q^\top A Q||_F^2\ge 0$. Towards this, let $Q_\perp$ complete $Q$ to an orthonormal basis, so that $I = QQ^\top + Q_\perp Q_\perp^\top$.
Then
\begin{align}
Q^\top A^2 Q
&= Q^\top A(I)A Q\\
&=Q^\top A QQ^\top A Q + Q^\top A Q_\perp Q_\perp^\top A Q\\
\tr(Q^\top A^2 Q) &=||Q^\top A Q||_F^2 + ||Q_\perp^\top A Q||_F^2 \ge ||Q^\top A Q||_F^2
\label{eq:trA2P_ge_B2}
\end{align}
On the other hand, since $P=QQ^\top$ has rank $k$ with $Q^\top Q = I_k$, by \Cref{lem:kyfan-max-principle}, 
\begin{align}
\tr(A^2P)
=\tr(Q^\top A^2 Q)
\le \sum_{i=1}^k \lambda_i(A^2)
=\sum_{i=1}^k \lambda_i^2
=||\Lambda_k||_F^2\label{eq:trA2P-lb}
\end{align}

Combining \Cref{eq:trA2P-lb} with \Cref{eq:trA2P_ge_B2} gives $||Q^\top A Q||_F^2\le ||\Lambda_k||_F^2$. By \Cref{eq:gap_decomp}, and the fact that  $\Lo(F) - \min_{\rank(Z)=k} \Lo(Z) \leq \eps$, we conclude that 

\begin{align}
    ||FF^\top - PAP||_F^2 \leq \eps \label{eq:bdd-term1-mf}
\end{align}

which upper bounds the first term in our decomposition, \Cref{eq:decomposition-of-gram-diff}.

\medskip
\textbf{Bounding $||PAP  - A_k||_F$.} 





With $P_k:=U_kU_k^\top$ denoting the projector onto the top-$k$ eigenspace of $A$,  we will first argue that projectors $P$ and $P_k$ are close in Frobenius norm. We have

\begin{align}
\tr(A^2P_k)-\tr(A^2P)
&=||\Lambda_k||_F^2-\tr(A^2P)\\
&\le ||\Lambda_k||_F^2-||Q^\top AQ||_F^2 \tag{By \Cref{eq:trA2P_ge_B2} and $P = QQ^\top$}\\
&\le \eps
\end{align}

Where the last inequality is by \Cref{eq:gap_decomp}, the non-negativity of $||FF^\top - PAP||_F^2$, and the fact that $\Lo(F) - \min_{\rank(Z)=k} \Lo(Z) \leq \eps$.

The squared eigengap $\lambda_k^2-\lambda_{k+1}^2 > 0$ is positive by assumption. Since
$\langle A^2,P_k-P\rangle=\tr(A^2P_k)-\tr(A^2P)\le\eps$,  \Cref{lem:projector-diff-ineq} (\Cref{eq:curvature_proj}) implies that

\begin{equation}
||P_k-P||_F^2
\le
\frac{2\eps}{\lambda_k^2-\lambda_{k+1}^2}
\label{eq:PkP_bound}
\end{equation}

Write $\Delta:=P-P_k$. Note that $A_k = P_k AP_k$. We have
\begin{align}
PAP-P_kAP_k
&= P_kA\Delta + \Delta A P_k + \Delta A \Delta\\
||PAP-P_kAP_k||_F
&\le 2||A||_2\, ||P_k||_2\,||\Delta||_F + ||A||_2\, ||\Delta||_2\, ||\Delta||_F\\
&\le 4||A||_2||\Delta||_F  \tag{Since $||P_k||_2 = 1$ and $||\Delta||_2 \leq ||P||_2 + ||P_k||_2 = 2$ }
\end{align}

With $||\Delta||_F=||P-P_k||_F\le \sqrt{\frac{2\eps}{\lambda_k^2-\lambda_{k+1}^2 }}$ from \eqref{eq:PkP_bound}, we have
\begin{align}
||PAP-A_k||_F = ||PAP-P_kAP_k||_F
\le
4||A||_2\sqrt{\frac{2\eps}{\lambda_k^2-\lambda_{k+1}^2 }}\label{eq:bdd-term2-mf}
\end{align}

\textbf{Combining the Bounds.}

Plugging in \Cref{eq:bdd-term1-mf} and \Cref{eq:bdd-term2-mf} into \Cref{eq:decomposition-of-gram-diff},

\begin{align}
    ||FF^\top  - A_k||_F &\leq ||FF^\top  - PAP||_F + ||PAP  - A_k||_F\\
    &\leq \sqrt{\eps} + 4||A||_2\sqrt{\frac{2\eps}{\lambda_k^2-\lambda_{k+1}^2 }}\\
    &=\Big(1 + 4||A||_2\sqrt{\frac{2}{\lambda_k^2-\lambda_{k+1}^2 }}\Big)\sqrt{\eps}\label{eq:gram-bdd-mf}
\end{align}



Finally, noting that any global minimizer $F_0 \in \arg \min_{\rank(Z)=k} \Lo(Z)$ satisfies that $F_0 F_0^\top = A_k$, we can view $||FF^\top  - A_k||_F$ as the difference between the gram matrices of $F$ and a global minimizer $F_0$. \Cref{lem:procrustes-bound} enables us to convert an upper bound on the Gram error $||FF^\top  - A_k||_F$ to factor error, between $F$ and the closest global minimizer to it in Frobenius norm (specified by a $(k \times k)$ orthogonal matrix). Let $F_0:=U_k\Lambda_k^{1/2}\in S$ so that $F_0F_0^\top=A_k$ and $(\sigma_k(F_0))^2=\lambda_k$. Apply \Cref{lem:procrustes-bound} with $U=F$ and $X = F_0$, and then combine the result with $(\sigma_k(F_0))^2=\lambda_k$ and \Cref{eq:gram-bdd-mf} to obtain:

\begin{align}
\dist(F,F_0)^2 &:= \min_{R\in O^{k \times k}}||F-F_0R||_F^2\\
&\le \frac{1}{2(\sqrt2-1)\lambda_k}||FF^\top-A_k||_F^2\\
&\le
\frac{\Big(1 + 4||A||_2\sqrt{\frac{2}{\lambda_k^2-\lambda_{k+1}^2 }}\Big)^2}{2(\sqrt2-1)\lambda_k}\eps.
\end{align}
Finally, choose $R^*\in O^{k \times k}$ that attains the minimum in $\dist(F,F_0)$ and set
$F^*:=F_0R^*\in S$. Then $||F-F^*||_F^2=\dist(F,F_0)^2 \leq \frac{\Big(1 + 4||A||_2\sqrt{\frac{2}{\lambda_k^2 - \lambda_{k + 1}^2}}\Big)^2}{2(\sqrt2-1)\lambda_k}\eps$, completing the proof.
\end{proof}

%% file: sections/appendix/e2e-result.tex
\section{Instantiating \Cref{defn:conj-alg-improved} on Clustered Graphs}\label{appen:e2e-result}

In this section, we first describe a concrete end-to-end result for how \Cref{defn:conj-alg-improved} improves over a simple random walk, in terms of diversity (\Cref{defn:diversity-measure}), on a special class of clustered graphs. 

\subsection{End-to-End Result of \Cref{defn:conj-alg-improved} on Clique Graphs}

\Cref{prop:e2e-guarantee-alg2} is an end-to-end result for \Cref{defn:conj-alg-improved}, showing that on a special class of theorem graphs (\Cref{defn:clique-graph}), that \Cref{defn:conj-alg-improved} explores the graph faster than a simple random walk. The final guarantee is in terms of the spectral gaps of the random walks, and the diversity of the distributions they generate, according to \Cref{defn:diversity-measure}.

\cliquegraphappen

\cliquepropappen

\begin{proof}
Fix a node \(x=(i,a)\in C_i\). Its neighbors split into two classes:
\begin{align}
I_x&:=\{(i,b): b\in[M],\ b\neq a\}\\
E_x&:=\{(j,a): j\in N_Q(i)\},
\end{align}
with \(|I_x|=M-1\), \(|E_x|=r\), and total degree $d=M-1+r.$ For \(\tau=1\), since \(G\) is \(d\)-regular, $\forall x, p^1(x) = \frac{1}{d}\1_{N(x)}$. The matrix $M_x'$ in \Cref{defn:conj-alg-improved} equals, up to a constant multiplicative factor of \(d^{-3}\), $M_x'[y,z]\propto |N(y)\cap N(z)|, y,z\in N(x).$ Therefore \Cref{defn:conj-alg-improved} minimizes
\begin{align}
F_x(\mu)
:=
\sum_{y,z\in N(x)} \mu(y)\mu(z)\,|N(y)\cap N(z)|
\qquad
\text{over }\mu\in\Delta(N(x)).
\end{align}

Write $u_b:=(i,b)\in I_x$ and $ v_j:=(j,a)\in E_x.$ We compute the relevant neighborhood intersections.

\paragraph{(i) Internal/internal.}
For any \(b \in [M], b\neq a\), $|N(u_b)\cap N(u_b)|=d.$ For any distinct \(b,b'\neq a\), $|N(u_b)\cap N(u_{b'})|=M-2.$ Indeed, \(u_b\) and \(u_{b'}\) share exactly the \(M-2\) vertices in \(C_i\) other than
\(u_b\) and \(u_{b'}\).

\paragraph{(ii) Internal/external.}
For any \(b\neq a\) and any \(j\in N_Q(i)\), $|N(u_b)\cap N(v_j)|=2.$  The two common neighbors are \(x=(i,a)\) and \((j,b)\).

\paragraph{(iii) External/external.}
For any \(j\in N_Q(i)\), $|N(v_j)\cap N(v_j)|=d.$ For any distinct \(j,k\in N_Q(i)\), $|N(v_j)\cap N(v_k)|=c.$ Indeed, the common neighbors are exactly the vertices $\{(\ell,a): \ell\in N_Q(j)\cap N_Q(k)\}$ whose number is \(c\) by assumption.

Now let $\mu_b:=\mu(u_b), \,\,\nu_j:=\mu(v_j),\,\,
s:=\sum_{b\neq a}\mu_b.$ Then \(\sum_{j\in N_Q(i)} \nu_j = 1-s\), and the objective becomes
\begin{align}
F_x(\mu)
&=
d\sum_{b\neq a}\mu_b^2
+(M-2)\sum_{\substack{b,b'\neq a\\ b\neq b'}}\mu_b\mu_{b'}
+d\sum_{j\in N_Q(i)}\nu_j^2
+c\sum_{\substack{j,k\in N_Q(i)\\ j\neq k}}\nu_j\nu_k
+2 \cdot \sum_{j \in N_Q(i), \, b \neq a} \nu_j \mu_b \\
&=
\bigl(d-(M-2)\bigr)\sum_{b\neq a}\mu_b^2
+(M-2)s^2
+(d-c)\sum_{j\in N_Q(i)}\nu_j^2
+c(1-s)^2
+ 2 \cdot 2s(1-s)
\end{align}
Since $d-(M-2)=r+1>0, \,d-c>0,$ for fixed \(s\) the objective is minimized by taking \(\mu_b\) uniform on \(I_x\) and \(\nu_j\)
uniform on \(E_x\), i.e.
\begin{align}
\mu_b&=\frac{s}{M-1}\\
\nu_j&=\frac{1-s}{r}
\end{align}
Thus it suffices to minimize the one-variable quadratic
\begin{align}
f_M(s)
=
A_M s^2 + 4s(1-s) + B_M(1-s)^2
\end{align}
where $A_M=M-2+\frac{r+1}{M-1}$ and $B_M=c+\frac{d-c}{r}.$ Expanding,
\begin{align}
f_M(s)
=
(A_M+B_M-4)s^2 + (4-2B_M)s + B_M.
\end{align}
Because \(M\ge 3\), we have \(A_M+B_M-4>0\), so \(f_M\) has the unique minimizer
\begin{align}
s=\eta_M:=\frac{B_M-2}{A_M+B_M-4}.
\end{align}
Hence the unique optimizer of \Cref{defn:conj-alg-improved} at \(x=(i,a)\) is
\begin{align}
\mu_x^*
=
\frac{\eta_M}{M-1}\sum_{b\neq a}e_{(i,b)}
+
\frac{1-\eta_M}{r}\sum_{j\in N_Q(i)}e_{(j,a)}
\end{align}
In particular, \Cref{defn:conj-alg-improved} sends total mass \(\eta_M\) to the current cluster \(C_i\), and
total mass \((1-\eta_M)/r\) to each neighboring cluster \(C_j\), \(j\in N_Q(i)\). These transition dynamics depend
only on the cluster index \(i\), not on the label \(a\). Thus, we can describe the dynamics of the Markov chain in terms of cluster-level transition probabilities, i.e. as a random walk over the meta-graph $Q$, where each cluster  \(\{C_i\}_{i=1}^K\) corresponds to one state. As a random walk over the meta-graph $Q$, \Cref{defn:conj-alg-improved} is described by a transition matrix $\Pimp^{(Q)} \in [0,1]^{K \times K}.$

\begin{align}
\Pimp^{(Q)}=\eta_M I_K + (1-\eta_M)P_Q  \in [0,1]^{K \times K}
\end{align}

Similarly, the simple random walk on \(G\) can also be described as a random walk over the meta-graph $Q$, whose transition matrix we denote $\Pbase^{(Q)}  \in [0,1]^{K \times K}$. From any \(x=(i,a)\in C_i\) there are \(M-1\) neighbors inside \(C_i\) and exactly one neighbor in \(C_j\) for each \(j\in N_Q(i)\). Therefore,
\begin{align}
\Pbase^{(Q)}
=
\frac{M-1}{M-1+r}I_K
+
\frac{r}{M-1+r}P_Q  \in [0,1]^{K \times K}
\end{align}

Regarding notation, define $\Pbase \in [0,1]^{KM \times KM}$ and $\Pimp  \in [0,1]^{KM \times KM}$ as the simple random walk over $G$ and the random walk induced by \Cref{defn:conj-alg-improved} over $G$ (not over the meta-graph $Q$), respectively. $\Pbase^{(Q)}$ and $\Pimp^{(Q)}$ are the meta-graph analogs of $\Pbase$ and $\Pimp.$

Additionally, for transition matrix $P$ of a Markov chain, define the spectral gap as $\gamma(P) := 1 - \max(\lambda_2(P), |\lambda_K(P)|)$. Then,

\begin{align}
    \gamma(\Pbase^{(Q)}) &= 1 - \lambda_2(\Pbase^{(Q)}) \tag{Since when $M - 1 \geq r,$ then  $\Pbase^{(Q)} \succeq 0$}\\
    &= 1 - [\frac{M-1}{M-1+r} + \frac{r}{M-1+r}\lambda_2(P_Q)] \\
    &=  \frac{r}{M-1+r}(1 - \lambda_2(P_Q))\\
    \gamma(\Pimp^{(Q)}) &= \min \big( 1 - \lambda_2 (\eta_M I + (1 - \eta_M) P_Q), 1 - |\lambda_K (\eta_M I + (1 - \eta_M) P_Q)| \big)\\ 
    &= \min \big( (1 - \eta_M)(1 - \lambda_2(P_Q)), 1 - |\lambda_K (\eta_M I + (1 - \eta_M) P_Q)| \big)\\ 
    &\geq \min \big( (1 - \eta_M)(1 - \lambda_2(P_Q)), 1 - |2\eta_M - 1| \big) \tag{Since $\lambda_K(P_Q) \geq -1$}\\ 
    \implies \frac{\gamma(\Pimp^{(Q)})}{\gamma(\Pbase^{(Q)})} &\geq \frac{M - 1 + r}{r} \min(1 - \eta_M, \frac{1 - |2\eta_M - 1|}{1 - \lambda_2(P_Q)})\\
    &\geq \frac{M - 1 + r}{r} \min(1 - \eta_M, \frac{1 - |2\eta_M - 1|}{2}) \tag{$1 - \lambda_2(P_Q) \leq 2$}\\
    &\geq \frac{M - 1 + r}{r} \min(1 - \eta_M, \eta_M)\\
    &\geq \frac{M - 1 + r}{r} \cdot \frac{ c + \frac{d - c}{r} - 2 }{M - 2 + \frac{r + 1}{M - 1} + c + \frac{d - c}{r} - 4}\\
    &\geq \Omega(\frac{M}{r^2})
\end{align}

This proves the first claim.

Towards the second claim, by assumption, $D_0$ is uniform over its starting cluster: 

\begin{align}
    \forall a, a' \in [M], \forall i \in [K], D_0((i,a)) = D_0((i, a'))
\end{align}

By an inductive argument we have that for all $t \in \N, \mutalg $ and $ \mutsrw$ are uniform over clusters:

\begin{align}
    \forall a, a' \in [M], \forall i \in [K], \mutalg((i,a)) &= \mutalg((i, a')) \textup{ and }\\
    \mutsrw((i,a)) &= \mutsrw((i, a'))
\end{align}

Thus, for all $t \in \N$,  $\mutalg$ (resp. $\mutsrw$) is characterized by a probability distribution $\qtalg \in \Delta([K])$  (resp. $\qtsrw$) over clusters, where each node in cluster $i$ has probability $\frac{\qtalg (i)}{M}$ (resp. $\frac{\qtsrw(i)}{M}$). The stationary measure $\pi$ is uniform over the $KM$ nodes of $G$ since $G$ is regular. Thus,

\begin{align}
    \forall t \in \N, ||\mutalg||_{\ell^2(1/\pi)}^2 &= KM \cdot \frac{1}{M} \cdot ||\qtalg||_2^2\\
    \forall t \in \N, ||\mutsrw||_{\ell^2(1/\pi)}^2 &= KM \cdot \frac{1}{M} \cdot ||\qtsrw||_2^2
\end{align}

Let $u = \frac{1}{K} \1$ be the stationary measure of both $\Pbase^{(Q)}$ and $\Pimp^{(Q)}$, random walks over the meta-graph $Q$. With $\tau = 1$ and setting $P  := \Psrw = \Pbase$, for any $t \in \N$ we have:

\begin{align}
    \forall t \in \N, ||\mutalg||_{\ell^2(1/\pi)}^2 &= 1  + K \cdot ||\qtalg - u||_2^2\\
    \implies \forall t \in \N, \Divt (\mutalg) &= \frac{1}{||(\mutalg)^\top \Pbase||_{\ell^2(1/\pi)}^2}\\
    &= \frac{1}{1 + K \cdot ||(\qtalg - u)^\top\Pbase^{(Q)}||_2^2}\\
    &= \frac{1}{1 + K \cdot ||(q_0 - u)^\top(\Pimp^{(Q)})^{t}\Pbase^{(Q)}||_2^2}
\end{align}


For $t = \frac{1}{2\gamma(\Pimp^{(Q)})}\log \frac{K}{\eps}$, then 

\begin{align}
    &K \, ||(q_0 - u)^\top(\Pimp^{(Q)})^{t}\Pbase^{(Q)}||_2^2\\
    &\leq K \, [(1 - \gamma(\Pimp^{(Q)}))^t (1 - \gamma(\Pbase^{(Q)}))]^2 ||q_0 - u||_2^2\\
    &\leq K \, (1 - \gamma(\Pimp^{(Q)}))^{2t}\\
    &\leq \eps
\end{align}

Since $\eps < 1$, when $t = \frac{1}{2\gamma(\Pimp^{(Q)})}\log \frac{K}{\eps}$, we have that $\Divt (\mutalg) \geq \frac{1}{1 + \eps} \geq 1 - \eps$.

By an analogous argument for $\mutsrw$, for any $t \in \N$ we have

\begin{align}
    \forall t \in \N, ||\mutsrw||_{\ell^2(1/\pi)}^2 &= 1  + K \cdot ||\qtsrw - u||_2^2\\
    \implies \forall t \in \N, \Divt (\mutsrw) &= \frac{1}{||(\mutsrw)^\top \Pbase||_{\ell^2(1/\pi)}^2}\\
    &= \frac{1}{1 + K \cdot ||(\qtsrw - u)^\top\Pbase^{(Q)}||_2^2}\\
    &= \frac{1}{1 + K \cdot ||(q_0 - u)^\top(\Pbase^{(Q)})^{t}\Pbase^{(Q)}||_2^2}
\end{align}

For  $t = \frac{1}{2\gamma(\Pbase^{(Q)})}\log \frac{K}{\eps}$, then $\Divt (\mutsrw) \geq \frac{1}{1 + \eps} \geq 1 - \eps$.


\end{proof}

%% file: sections/appendix/q1-detailed-asmpt.tex
\section{Justification of \Cref{assumption:a2-rt}}\label{appen:more-detailed-asmpts}

The core result of this section is \Cref{lem:567-imply-1-4}, which justifies  \Cref{assumption:a2-rt} by showing that finer-grained assumptions inspired from statistical learning theory implies \Cref{assumption:a2-rt}.

\begin{restatable}[Finer-Grained Assumptions]{lemma}{detailedassumptions}
\label{lem:567-imply-1-4}
Assumptions \ref{assumption:expressivity}, \ref{assumption:rad-com}, \ref{assumption:opt-erm} imply Assumption \ref{assumption:a2-rt}.
\end{restatable}

\subsection{Statistical Learning Theory Inspired Assumptions}\label{sec:4-4-justification}

First, we  introduce Assumptions \ref{assumption:expressivity} - \ref{assumption:opt-erm}. Let $\Ver : V \times \{ 0,1\}^* \to \{ 0,1\}$ be a formal verifier where for  theorem $x \in V$, then $\Ver(x,y) = 1$ iff $y \in \{0,1\}^*$ is a correct proof of $x$. The first assumption assumes the function class $\Fprov$ of provers $\Fprov \ni f : V \to \Delta(\{ 0,1\}^*)$ can express the ground-truth function.

\begin{assumption}(Expressivity)\label{assumption:expressivity}
    There exists $f^* \in \Fprov$ which is a ground-truth prover, so that for every theorem $x \in V,$ $f^*(x) \in \Delta(\{ 0,1\}^*)$ is distribution over correct proofs of $x$. 
\end{assumption}

Each $f \in \Fprov$ maps a theorem to a distribution of proofs ($f : V \to \Delta(\{ 0,1\}^*)$, where $V$ is the set of theorems). Thus we can define a notion of the success rate of a prover on a theorem as the probability that the prover outputs a correct proof. Formally, let $f(x,r)$ be the output of prover $f$ on theorem $x$ and random seed $r$, where $r \sim D_r$. Define the following loss function over $\Fprov$ (which abstracts away the randomness $r \sim D_r$ into a success rate):

\begin{definition}\label{defn:loss-fn-prover} (Loss Functions) Define the following loss functions on provers in $\Fprov$.
    \begin{align}
        \ell(f, x) &:= \ex_{r \sim D_r}[ 1 - \Ver(x, f(x, r)) ]\\
        &= 1 - \Succ(x,f)\\
        \Lo(f, D) &:= \ex_{x \sim D} \, \ell(f,x)\\
        \Lo_{n} (f, \{ x_i\}_{i \in [n]}) &:= \frac{1}{n} \sum_{i \in [n]}  \ell(f, x_i)
    \end{align}
    With the understanding that $S_n = \{ x_i\}_{i \in [n]}$ are drawn i.i.d. from $D$. Let $\ell(\Fprov) := \{ x \to \ell(f,x) : f \in \Fprov\}$ be the family of loss functions induced by $\Fprov$ and $\ell$. 
\end{definition}

We define the empirical risk minimizer (ERM) as follows.

\begin{definition} (Minimizers of Population and Empirical Risk)\label{defn:erm} Let $D \in \Delta(V)$ be any distribution over theorems. For any $n \in \N_{+}$, suppose $\{ x_i\}_{i \in [n]}$ are drawn i.i.d. from $D$. First, note that $f^*$ is the population minimizer,

\begin{align}
    \Lo(f^*, D) = 0
\end{align}

Now, define the ERM $\hat{f}^{(n)}$.

\begin{align}
    \hat{f}^{(n)} \leftarrow \arg\min_{f \in \Fprov} \Lo_{n} (f, \{ x_i\}_{i \in [n]})\label{eq:erm}
\end{align}
\end{definition}

The Rademacher complexity\footnote{This definition is different than \Cref{defn:rad-complexity}, previously defined and used in the context of \cite{haochen2022provableguaranteesselfsuperviseddeep} in \Cref{appen:cont-learning-intro}} is defined as follows.

\begin{definition}\label{defn:rad-complexity-average} For a function class $\F$, and data distribution $D$, the average Rademacher complexity is defined as
    \begin{align}
        \Rc_n(\F) := \ex_{\{x_i\}_{i \in [n]} \overset{\mathrm{iid}}{\sim} D } \ex_{\{\sigma_i \}_{i \in [n]} \overset{\mathrm{iid}}{\sim} \{ \pm 1\} }\Big[ \sup_{f \in \F} \frac{1}{n}  \sum_{i = 1}^n \sigma_i f(x_i) \Big]
    \end{align}
\end{definition}

Upper bounding the Rademacher complexity of $\ell(\Fprov)$ will upper bound the generalization error of the ERM. In particular,

\begin{lemma} (Corollary 4.19 of \cite{ma})\label{lem:rad} For any $\delta > 0, n \in \mathbb{N}$, let $S_n := \{x_i\}_{i \in [n]}$ be a training set of size $n$  drawn i.i.d. from $D$, and let $\hat{f}^{(n)}$ minimize the empirical loss over $S_n$ (via \Cref{defn:erm}). Then, with probability at least $1 - \delta$ over the draw of $S_n$,
    \begin{align}
        \Lo(\hat{f}^{(n)}, D) - \Lo(f^*, D) &\leq 2\Rc_n(\ell(\Fprov)) +  \sqrt{\frac{\log \frac{2}{\delta}}{2n}}
    \end{align}
\end{lemma}

\begin{proof}
    \begin{align}
        \Lo(\hat{f}^{(n)}, D) - \Lo(f^*, D) &\leq \Lo(\hat{f}^{(n)}, D) - \Lo_n(\hat{f}^{(n)}, S_n) + \Lo_n(f^*, S_n) - \Lo(f^*, D)\\
        &\leq 2 \cdot \sup_{f \in \Fprov} |\Lo(f, D) - \Lo_{n}(f, \{ x_i\})|
    \end{align}
    Then, apply Corollary 4.19 of \cite{ma}.
\end{proof}

In this context, we assume an upper bound on the Rademacher Complexity of the function class $\ell(\Fprov)$.

\begin{assumption}\label{assumption:rad-com} (Generalization)
     Given prover class $\Fprov$ assume there exists a $\Fprov$-dependent quantity $\C(\Fprov)$ where for any theorem distribution $D \in \Delta(V)$, with function class $\ell(\Fprov)$ defined as in \Cref{defn:loss-fn-prover}, we have:
     \begin{align}
         \Rc_{n} (\ell(\Fprov)) \leq  \sqrt{\frac{\C(\Fprov)}{n}}
     \end{align}
\end{assumption}

The next assumption describes under what conditions it is efficient to find the ERM from the model class $\Fprov$, on a size $n$ dataset of i.i.d. theorems from a distribution of theorems. In particular, a necessary condition for such a procedure to find the ERM is that each theorem in the training distribution has a proof that can be efficiently found. 

\begin{assumption}\label{assumption:opt-erm} (Optimization)
    Suppose we have i.i.d. sample access to theorem distribution $D \in \Delta(V)$, and for any $x \in \supp(D)$, the proof of $x$ is known. Then, for any $n$, it is computationally tractable to find the ERM $\hat{f}^{(n)}$ (\Cref{eq:erm}) of a size $n$ dataset drawn i.i.d. from $D$.
\end{assumption}

Note that in practice, we may only be able to efficiently find an $\eps_{\textup{opt}}$-ERM, for any reasonable $\eps_{\textup{opt}} > 0$, where an $\eps_{\textup{opt}}$-ERM is a prover $\tilde{f}^{(n)}$ where 

\begin{align}
    \Lo_n(\tilde{f}^{(n)}, \{ x_i\}_{i \in [n]}) \leq \min_{f \in \Fprov} \Lo_n(f, \{ x_i\}_{i \in [n]}) + \eps_{\textup{opt}}
\end{align}

However, the remaining analysis only requires $\Lo(\tilde{f}^{(n)},D) - \Lo(f^*, D) \leq \eps_{\textup{opt}} + 2 \cdot \sup_{f \in \Fprov} |\Lo(f, D) - \Lo_{n}(f, \{ x_i\})|$. With $n = \Theta(\frac{\C(\Fprov)}{\eps_{\textup{gen}}^2})$ samples, we get $L(\tilde{f}^{(n)},D) - L(f^*, D) \leq \eps_{\textup{opt}} + \eps_{\textup{gen}}$ with high probability. Thus, for any reasonable desired accuracy level $\eps$, we can set $\eps_{\textup{opt}}, \eps_{\textup{gen}} = \eps/2$, finding the $\eps_{\textup{opt}}$-ERM with generalization gap of $\eps_{\textup{gen}}$. The necessary sample complexity would only increase by a factor of 4, a constant factor which we neglect. With this understanding, we keep the current phrasing of \Cref{assumption:opt-erm} of assuming the ERM is tractable.





\subsection{Proof of \Cref{lem:567-imply-1-4}}
We now prove the main result of this section.

\detailedassumptions*

\begin{proof}

Since $\forall x \in \supp(D)$, the proof of $x$ is known, then by \Cref{assumption:opt-erm}, given $\delta > 0$ and $n \in \N_+$ i.i.d. samples from $D$, we can compute $\hat{f}^{(n)}$ via \Cref{eq:erm}. By \Cref{lem:rad}, \Cref{assumption:expressivity}, and \Cref{assumption:rad-com},  we have that with probability at least $1 - \delta$,

\begin{align}
    \ex_{x \sim D} [1 - \Succ(x, \hat{f}^{(n)})] = L(\hat{f}^{(n)}, D) &\leq \frac{4\cdot \sqrt{\C(\Fprov)}}{\sqrt{n}} + 2 \sqrt{\frac{\log \frac{2}{\delta}}{2n}}
\end{align}

By Markov's inequality, for any $p \in (0,1)$,

\begin{align}
    \Pro_{x \sim D} [\Succ(x, \hat{f}^{(n)}) \geq p] \geq 1 - \frac{\frac{4\sqrt{\C(\Fprov)} + 2\sqrt{\frac{1}{2}\log \frac{2}{\delta}}}{\sqrt{n}}}{1 - p}
\end{align}

Thus, it suffices to set $n$ to be $(\frac{4\sqrt{\C(\Fprov)} + 2\sqrt{\frac{1}{2}\log \frac{2}{\delta}}}{\eps (1 - p)})^2 \leq \Theta(\frac{\C(\Fprov)\log \frac{1}{\delta}}{\eps^2 (1 - p)^2})$ in order for $\Pro_{x \sim D} [\Succ(x, \hat{f}^{(n)}) \geq p] \geq 1 - \eps$. 
\end{proof}

%% file: sections/appendix/aux.tex
\section{Auxiliary Lemmas}\label{appen:aux}

\divtkct*

\begin{proof}
    For a distribution $D \in \Delta(V)$ and subset $A \subset V$, denote $D(A)$ as the measure of $A$ under $D$. Since $P$ is the transition matrix of a nearest-neighbor-type random walk, then for any set $A \subset V$, and any distribution $D \in \Delta(V)$, with $(P^\top)^\tau D$ representing the distribution of a $\tau$-step random walk according to $P$ initialized with $D$, we have

    \begin{align}
        ((P^\top)^\tau D) \, (B(A, \tau)) \geq D(A)
    \end{align}

    In particular, for the set $A = \{ x \in V : \Succ(x,f) \geq p\}$, we have  

    \begin{align}
        ((P^\top)^\tau D) \, (B(\{ x \in V : \Succ(x,f) \geq p\}, \tau)) \geq \Pro_{x \sim D}[\Succ(x,f) \geq p]\label{eq:min-measure-of-Kct}
    \end{align}

    Next, for any set $A' \subset V$ and any distribution $D' \in \Delta(V)$, the cardinality of $A'$ is lower bounded in terms of its measure under $D'$ and the (weighted) two-norm of $D'$. In particular, we apply \Cref{lem:ent-to-cov-num} with $A' = B(\{ x \in V : \Succ(x,f) \geq p\}, \tau)$ and  $D' = (P^\top)^\tau D$.

    \begin{align}
        |K_\tau^p(f)| &:= |B(\{ x \in V : \Succ(x,f) \geq p\}, \tau)|\\
        &\geq \frac{((P^\top)^\tau D) \, (B(\{ x \in V : \Succ(x,f) \geq p\}, \tau))^2}{\sup_{x \in V} \pi(x)} \frac{1}{||(P^\top)^\tau D||^2_{\ell^2(1/\pi)}}\\
        &\geq \frac{(\Pro_{x \sim D}[\Succ(x,f) \geq p])^2}{\sup_{x \in V} \pi(x)} \Divt(D) \tag{By \Cref{eq:min-measure-of-Kct}}
    \end{align}
    
\end{proof}

\begin{lemma}\label{lem:ent-to-cov-num}
For graph $G = (V,E)$, let $\pi : V \to (0,\infty)$. For any  distribution $D' \in \Delta(V)$ and any $A' \subset V$, we have

    \begin{align}
        |A'| \geq \frac{(D'(A'))^2 }{\sup_{x \in V} \pi(x)} \frac{1}{||D'||^2_{\ell^2(1/\pi))}}\label{eq:min-A-ent}
    \end{align}
\end{lemma}

\begin{proof}
    For any  distribution $D' \in \Delta(V)$, any $\alpha \in [0,1]$, and any $A' \subset V$, we have:
    
    \begin{align}
        (D'(A'))^2 &\leq (\sum_{x \in V} \1_{A'}(x) D'(x))^2\\
        &\leq [\sum_{x \in V} (\1_{A'}(x) \sqrt{\pi(x)})^2] \cdot  [\sum_{x \in V} (D'(x) \frac{1}{\sqrt{\pi(x)}})^2]
    \end{align}

    Thus, 

    \begin{align}
        |A'| &=\sum_{x \in V} \1_{A'}(x)\\
        &\geq \frac{\sum_{x \in V} (\1_{A'}(x) \sqrt{\pi(x)})^2}{\sup_{x \in V} \pi(x)}\\
        &\geq \frac{(D'(A'))^2 }{\sup_{x \in V} \pi(x)} \frac{1}{||D'||^2_{\ell^2(1/\pi))}}
    \end{align}

\end{proof}





\begin{lemma}\label{lem:ensure-mixing-for-alg4}
    Let $G = (V,E)$ be a finite, undirected, connected, and non-bipartite graph. Let $(V,P)$ be a nearest-neighbor-type Markov chain on $G$. For every $x \in V,$ let $N(x)$ denote its neighbors in $G$. If $\, \forall x \in V, \forall y \in N(x), P[x,y] > 0$, then $P$ is irreducible and aperiodic, and there exist constants $C > 0$ and $\alpha \in (0,1)$ where

    \begin{align}
        \max_{x \in V} || p^{(n)}(x, \cdot) - \pi||_{\textup{TV}} \leq C \alpha^n
    \end{align}
\end{lemma}

\begin{proof}
    Because $G$ is connected, then $\forall x \in V, \forall y \in N(x), P[x,y] > 0$, so $P$ is irreducible.

    Because $G$ is undirected, $\textup{gcd}\{ n \in \N : p^{(n)}(x,x) > 0\} \leq 2$. Because $G$ is non-bipartite, $\textup{gcd}\{ n \in \N : p^{(n)}(x,x) > 0\} \neq 2$. Thus, $P$ is aperiodic.

    Because $V$ is finite and the rows of $P \in [0,1]^{|V| \times |V|}$ sum to $1$, then $P \1 = \1$, so that $P$ has an eigenvalue of $1$. The left eigenvector corresponding to eigenvalue $1$ is the stationary distribution $\pi$. Since $P$ is irreducible, $\pi$ is unique.

    By \Cref{lem:mixing-asymptotic}, there exist constants $C > 0$ and $\alpha \in (0,1)$ where

    \begin{align}
        \max_{x \in V} || p^{(n)}(x, \cdot) - \pi||_{\textup{TV}} \leq C \alpha^n
    \end{align}
\end{proof}

%% file: sections/appendix/K-cluster.tex
\section{Figures}\label{sec:figures}

\begin{figure}[H]
\centering
\begin{tikzpicture}[font=\small]
  \def\K{6}             
  \def\R{4.1cm}         
  \def\clusterR{1.25cm} 
  \def\innerR{0.72cm}   
  \def\npc{7}           

  \tikzset{
    clust/.style={draw, thick},
    inode/.style={circle, fill=black, inner sep=1.2pt},
    bridge/.style={draw, thick, dashed},
  }

  \foreach \i in {1,...,\K}{
    \pgfmathsetmacro\ang{360/\K*(\i-1)}
    \coordinate (C\i) at (\ang:\R);

    \draw[clust] (C\i) circle (\clusterR);

    \node at (C\i) {$M$ nodes};
    \node[font=\scriptsize] at ($(C\i)+(0,0.55cm)$) {$\gamma=\Omega(1)$};

    \foreach \j in {1,...,\npc}{
      \pgfmathsetmacro\angj{\ang + 360/\npc*(\j-1)}
      \node[inode] (V\i-\j) at ($(C\i)+(\angj:\innerR)$) {};
    }
  }

  \foreach \i [evaluate=\i as \next using {int(mod(\i,\K)+1)}] in {1,...,\K}{
    \coordinate (A\i) at ($(C\i)!\clusterR!(C\next)$);
    \coordinate (B\i) at ($(C\next)!\clusterR!(C\i)$);
    \draw[bridge] (A\i) -- (B\i);
  }
  \draw[bridge] ($(C1)!\clusterR!(C3)$) -- ($(C3)!\clusterR!(C1)$);

  \node[align=center, font=\scriptsize] at (0,0)
    {sparse inter-cluster edges\\$\Rightarrow$ global spectral gap $\ll \gamma$};
\end{tikzpicture}
\caption{Schematic of a $K$-cluster graph: each cluster is internally well-connected (spectral gap $\gamma$), but only a sparse set of inter-cluster edges connects clusters, yielding a much smaller global spectral gap.}
\label{fig:kcluster}
\end{figure}
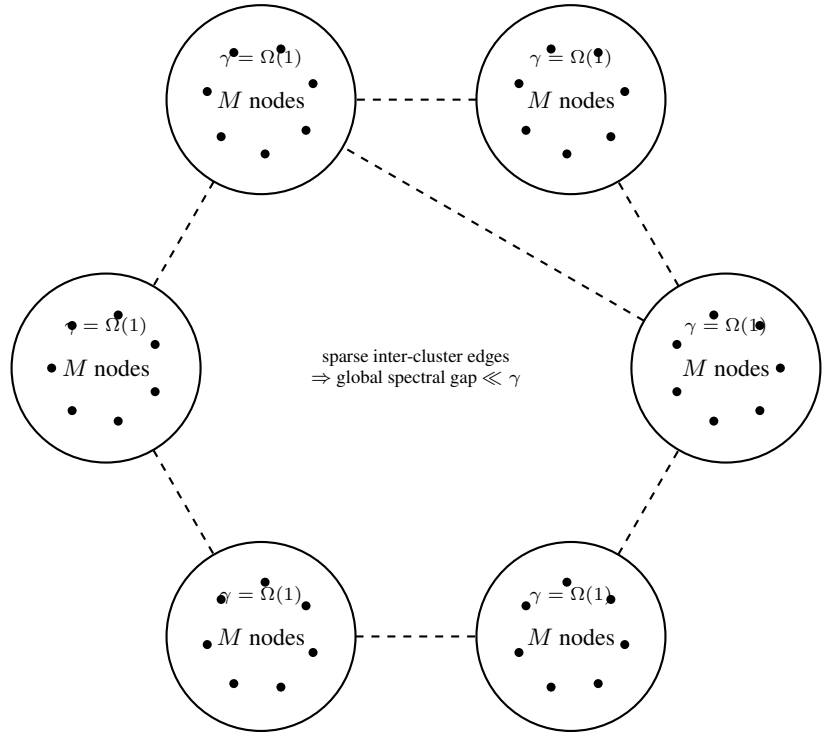